\pdfoutput=1
\documentclass[english, 10pt]{article}
\usepackage[margin=1in]{geometry}
\usepackage{csquotes}
\usepackage{multirow}
\usepackage{amsthm}
\usepackage{amsmath}
\usepackage{amsfonts}
\usepackage{amssymb}
\usepackage{bm}
\usepackage{natbib}
\usepackage{authblk}
\usepackage{color}
\usepackage[dvipsnames]{xcolor}
\usepackage{caption}
\usepackage{subcaption}
\usepackage{wrapfig}
\usepackage{graphicx}
\usepackage{array}
\usepackage{setspace}
\usepackage[ruled,vlined]{algorithm2e}
\SetKwRepeat{Do}{do}{while}
\usepackage[bookmarks=false, colorlinks=true, citecolor=MidnightBlue, linkcolor=MidnightBlue, urlcolor=MidnightBlue]{hyperref}
\usepackage{setspace}
\usepackage{bbm}
\usepackage{thm-restate}

\usepackage{xr}

\makeatletter

\usepackage{verbatim}   

\usepackage{microtype}
\usepackage{mathpazo} 
\usepackage{courier} 
\normalfont
\usepackage[T1]{fontenc}
\renewcommand{\leq}{\leqslant}

\renewcommand{\geq}{\geqslant}

  \theoremstyle{plain}
  \newtheorem{Theorem}{\protect\theoremname}
  \theoremstyle{plain}
  \newtheorem*{Theorem*}{\protect\theoremname}
  \theoremstyle{plain}
  \newtheorem{proposition}{\protect\propositionname}
  \theoremstyle{plain}
  \newtheorem*{prop*}{\protect\propositionname}
  \theoremstyle{plain}
  \newtheorem{lemma}{\protect\lemmaname}
   \theoremstyle{plain}
  \newtheorem*{lemma*}{\protect\lemmaname}  
  \theoremstyle{plain}
  \newtheorem{definition}{\protect\definitionname}
  \theoremstyle{plain}
  \newtheorem{corollary}{\protect\corollaryname}
  \theoremstyle{plain}
  \newtheorem{example}{\protect\examplename}
 \theoremstyle{plain}
\newtheorem{remark}{Remark}
\theoremstyle{plain}
\newtheorem{fact}{\protect\factname}
\theoremstyle{plain}
\newtheorem{assumption}{\protect\assumptionname}
\newtheorem*{assumption*}{\protect\assumptionname}
\theoremstyle{plain}

 \theoremstyle{plain}
\newtheorem*{model*}{Model}

\makeatother

\newcommand{\ind}{\mathbb{1}}
\newcommand{\E}{\mathbb{E}}
\newcommand{\I}{\mathbb{I}}
\newcommand{\ent}{\mathbb{H}}

\newcommand{\Prob}{\mathbb{P}}

\newcommand{\Xc}{\mathcal{X}}
\newcommand{\Dc}{\mathcal{D}}

\usepackage{babel}
\providecommand{\assumptionname}{Assumption}
\providecommand{\definitionname}{Definition}
\providecommand{\lemmaname}{Lemma}
\providecommand{\propositionname}{Proposition}
\providecommand{\corollaryname}{Corollary}
\providecommand{\examplename}{Example}
\providecommand{\factname}{Fact}
\providecommand{\conditionname}{Condition}
\providecommand{\theoremname}{Theorem}

\DeclareMathOperator*{\argmax}{arg\,max}  

\usepackage{enumitem}   

\begin{document}

\title{Adaptive Experimentation in the Presence of \\Exogenous Nonstationary Variation}
\author{Chao Qin and Daniel Russo}
\affil{Columbia University}
\maketitle
\begin{abstract}
	We investigate experiments that are designed to select a treatment arm for population deployment. Multi-armed bandit algorithms can enhance efficiency by dynamically allocating measurement effort towards higher performing arms based on observed feedback. However, such dynamics can result in brittle behavior in the face of nonstationary exogenous factors influencing arms’ performance during the experiment. To counter this, we propose deconfounded Thompson sampling (DTS), a more robust variant of the prominent Thompson sampling algorithm. As observations accumulate, DTS projects the population-level performance of an arm while controlling for the context within which observed treatment decisions were made. Contexts here might capture a comprehensible source of variation, such as the country of a treated individual, or simply record the time of treatment. We provide bounds on both within-experiment and post-experiment regret of DTS, illustrating its resilience to exogenous variation and the delicate balance it strikes between exploration and exploitation. Our proofs leverage inverse propensity weights to analyze the evolution of the posterior distribution, a departure from established methods in the literature.
	Hinting that new understanding is indeed necessary, we show that a deconfounded variant of the popular upper confidence bound algorithm can fail completely. 
\end{abstract}

\section{Introduction}\label{sec:intro}
Multi-armed bandit (MAB) algorithms are crafted to enhance efficiency beyond what classical randomized controlled trials (RCTs) offer. In contrast to RCTs which maintain a fixed probability for assigning treatment arms throughout an experiment, MAB algorithms dynamically redistribute measurement effort towards higher performing arms based on observed feedback. Such strategies not only reduce experimental cost --– since inferior arms are played less frequently – but variants of MAB algorithms can increase statistical power in identifying the most effective arm \citep{bubeck2012best, kaufmann2016complexity,russo2020simple}. These efficiency advantages have motivated widespread adoption of MAB algorithms for selecting and personalizing digital content.  

Traditional application of MAB algorithms, as commonly advised in academic texts \citep{lattimore2020bandit} and industry-centric blogs \citep{amadio2020bandits}, presupposes that rewards linked with a particular arm selection are independently and identically distributed (i.i.d.) over time. However, this assumption often fails in real-world scenarios. 

To elucidate, consider a hypothetical\footnote{Shortcuts are a real product feature that enables users to conveniently access their favorite or recently played content. The discussion, however, does not necessarily mirror the specifics of the product or the available data. This example is presented solely for illustrative purposes.} scenario. In this scenario, the audio streaming platform Spotify aims to optimize the shortcuts displayed in Figure \ref{fig:spotify}. Imagine an experiment lasting one week. Each time period in the MAB model might correspond to a particular user who just opened the app, treatment arms might be slight alterations in the user interface of the shortcuts, and a positive 'reward' could signify a user locating an item to listen to without navigating away from the home page. The i.i.d.~assumption implies that a random sample of users who open the app on Monday morning will exhibit behavior similar to another random sample of users who open the app later in the week, such as on Friday evening. Yet, substantial variation in user behavior can occur over time.

\begin{wrapfigure}[16]{r}{0.3\textwidth}
	\centering
	\includegraphics[width=.75\linewidth]{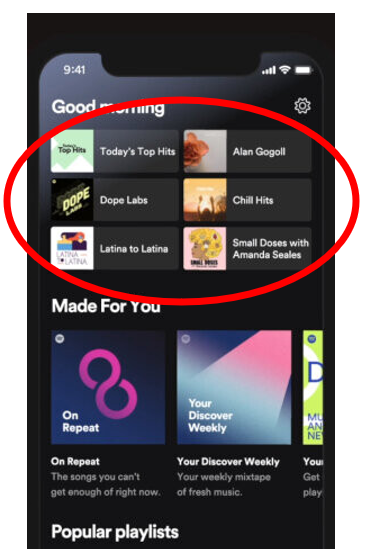}
	\caption{Shortcuts}
	\label{fig:spotify}
\end{wrapfigure}
RCTs are designed to be robust to exogenous variation like this. Since the probability of an arm being selected remains constant throughout the experiment, averaging the reward produced by an arm provides an unbiased estimate of the performance it would have yielded if it were deployed consistently to all users throughout the time period of the experiment. By varying arm selection probabilities over time, MAB algorithms lose this inherent resilience to nonstationary patterns. 
Of course, by abandoning adaptive arm selection, RCTs lose the efficiency advantages offered by MAB algorithms.

\subsection{An overview of the paper} 
We propose a new model in which nonstationary exogenous factors influence treatment arms’ performance during an experiment. Adapting the prominent Thompson sampling  algorithm \citep{thompson1933likelihood} to this model yields a new, more robust, variant which we call deconfounded Thompson sampling. 
We illustrate the algorithm's performance through simulations and conduct substantive theoretical analysis. 
To help the reader digest the full paper, this section provides an abbreviated overview of our model, proposed algorithm, and results.

\subsubsection{A new model of bandit experiments}\label{subsec:model_intro}

Among a set of $k$ predefined treatment arms, indexed as $[k]\triangleq \{1,\ldots,k\}$, a decision-maker (DM) aims to select an arm $I_{\rm post} \in [k]$ to deploy to the population at the end of the experiment. The experiment proceeds sequentially across $T$ rounds, 
thought of as representing interactions with distinct individuals or `users.' 
In each round $t\in [T]$, the DM selects a treatment arm $I_t \in [k]$ and observes a noisy reward   $R_{t,I_t}\in \mathbb{R}$ signaling the quality of the outcome. 
The DM also observes a vector of exogenous factors $X_t \in \mathbb{R}^d$ which influence rewards; these might encode things like features of the user, the weather, or timing of the interaction. 
Temporal patterns in factors $X_1, \ldots, X_T$ drive temporal patterns in rewards. 

In keeping with the tradition of the literature, we call these exogenous factors ``contexts''. 
However, unlike contextual bandit models \citep{li2010contextual} in which contexts are used to segment or personalize decision-rules, here they are used to control for exogenous sources of variation in experiments that seek to deploy a single treatment arm to the population. 
This reflects common experimental practice. Consider the representative example displayed in Figure \ref{fig:spotify},  where the goal is to establish a consistent user interface rather than one that undergoes erratic changes as a user's context (e.g. the time of day, their recent interactions) changes. Refer to Appendix \ref{sec:policy_learning} for a more substantive discussion and a generalization of our formulation that accommodates personalization. 

A Bayesian linear model allows the DM to draw inferences about the population-level reward an arm generates as observations are gathered. 
The mean reward signal of arm $i\in[k]$ in context $x\in \mathbb{R}^d$ is  $r_{\theta}(i, x) = \langle \theta^{(i)}, x \rangle$, where 
the parameter vector $\theta=(\theta^{(1)},\ldots, \theta^{(k)}) \in \mathbb{R}^{k\cdot d}$ is drawn from a multi-variate Gaussian prior, denoted  $\theta \sim N(\mu_1, \Sigma_1)$ where $\mu_1\in\mathbb{R}^{k\cdot d}$ and $\Sigma_1\in\mathbb{R}^{kd\times kd}$. The reward realized at time $t$ is
\begin{equation}
	R_{t,I_t} = r_{\theta}(I_t, X_t) + W_{t,I_t},
\end{equation}
where $W_{t,i} \sim N(0,\sigma^2)$ is independent Gaussian noise. In modeling reward noise as independent, we are implicitly assuming that any exogenous nonstationarity is ``explained'' by the contexts. The quality of the deployment arm $I_{\rm post}$ is assessed through its population-level reward $r_{\theta}(I_{\rm post})$. We model the population-level reward of an arm,
\[
r_{\theta}(i) = \E_{x\sim \Dc_{\rm pop}} \left[ r_{\theta}(i, x) \right] = \langle \theta^{(i)} \, , \, x_{\rm pop}\rangle \quad \text{where} \quad x_{\rm pop} = \E_{x\sim \Dc_{\rm pop}}[x], 
\]
as the average reward over contexts drawn from a pre-defined population distribution $\Dc_{\rm pop}$.
Because the decision-maker knows $x_{\rm pop}$ and the contexts are observable, standard calculations allow one to compute the (multi-variate Gaussian) posterior distribution of the population-level rewards $(r_{\theta}(1),\ldots, r_{\theta}(k))$ as observations accumulate. 

Why model $x_{\rm pop}$ as known to the DM? Continuing the example in Figure \ref{fig:spotify}, we imagine the company might form an empirical population distribution by subsampling from the features of users who visit the home page over e.g. the month prior to the experiment. Controlled experiments are usually conducted to evaluate differences in how treatment arms perform; using them to estimate passively observable quantities is wasteful. 
 
\subsubsection{Models of contextual variation subsume other models of nonstationarity}
This turns out to be a surprisingly rich modeling framework. 
The term `context' evokes a comprehensible source of exogenous variation. However, as illustrated in the next example, one can also model bandit experiments with nonstationary rewards whose pattern, seemingly, cannot be explained by any observable factor. 
We treat this as a special case of our formulation by taking the time period at which an arm was selected to be an observable context.
\begin{example}[Modeling latent exogenous variation with contexts]\label{ex:latent_confounders}
	Take $d=T$ and assume $X_{1:T}$ is deterministic with the $t^{\rm th}$ context equal to the $t^{\rm th}$ standard basis vector: $X_t = e_t \in \mathbb{R}^T$. Let $\Dc_{\rm pop}$ be the uniform distribution over $\{e_1, \ldots, e_T\}$. In this setting, the reward at time $t$,  $R_{t,I_t} = \theta_{t}^{(I_t)}+W_{t,I_t}$  is a noisy sample of $\theta^{(I_t)}_{t}$ and the experimenter's goal is to select the arm
	\begin{equation}\label{eq:best_arm_in_hindsight} 
		I^* \in \argmax_{i \in [k]} r_{\theta}(i)  \qquad  \text{where} \qquad r_{\theta}(i) = \frac{1}{T} \sum_{t=1}^{T} \theta^{(i)}_{t}, 
	\end{equation}
	which has highest average reward throughout the experiment.\footnote{This objective is implicit in the way that average treatment effects are estimated in A/B tests, and we choose to mimic this standard practice in Example \ref{ex:latent_confounders}. 
		The rationale behind this practice is subtle, however. What does it mean to optimize a backward-looking objective when nonstationarity is a concern?  Our partial answer is that the objective in~\eqref{eq:best_arm_in_hindsight} reflects a belief that an arm that outperformed others over a substantial time-span, like a couple of weeks, is likely to continue its strong performance. This belief is consistent with concerns about nonstationarity in other forms, like exogenous time trends that shift all arm's mean rewards (see Figure~\ref{fig:nonstationary}) or more cyclic patterns, where the performance of an arm depends on the time of day.} 
	
	The prior  $\theta \sim N(\mu_1, \Sigma_1)$ allows the decision-maker to draw inferences based on observations so far. 
	Note that a vanilla bandit experiment is an extreme, degenerate, special case, where the rank of $\Sigma_1$ is $k$ and $\theta_{1}^{(i)}=\cdots = \theta_{T}^{(i)}$ almost surely.  
	Figure \ref{fig:nonstationary} represents a structured prior on $\theta$ that allows the decision-maker to guard against certain nonstationarity patterns while still allowing
	them to use past observations to forecast arms' relative performance. 
\end{example}

\begin{figure}
	\centering
	\begin{subfigure}{.45\textwidth}
		\centering
		\includegraphics[width=1\linewidth]{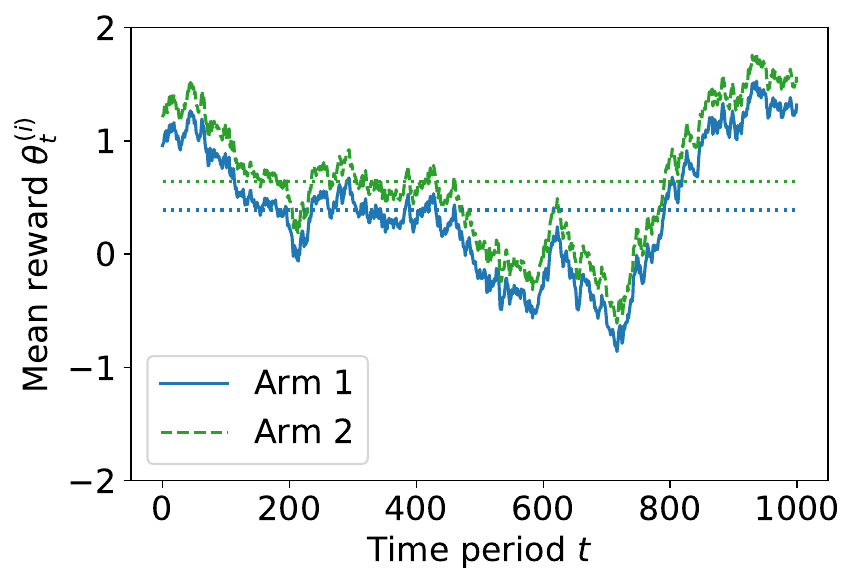}
		\caption{bandwidth $\kappa=1,000$}
		\label{fig:nonstationary1}
	\end{subfigure}
	\begin{subfigure}{.45\textwidth}
		\centering
		\includegraphics[width=1\linewidth]{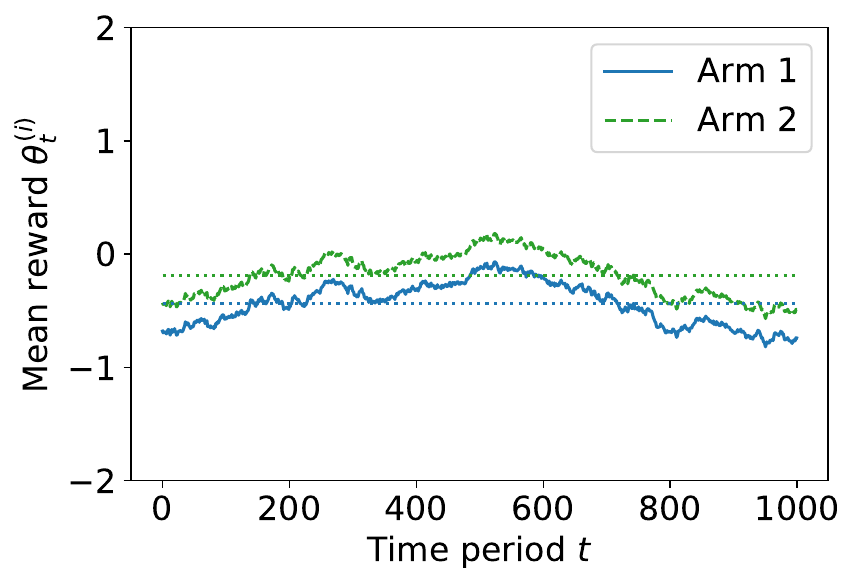}
		\caption{bandwidth $\kappa=10,000$}
		\label{fig:nonstationary2}
	\end{subfigure}
	\caption{Two draws of $\{\theta_{t}^{(i)}\}_{t\in [1000], i\in[2]}$in a special case of Example \ref{ex:latent_confounders} in which parameters follow the latent variable model $\theta_{t}^{(i)} = \theta^{(i)}_{0} + \epsilon_t$. The exogenous process $\{\epsilon_t\}$ has correlation ${\rm corr}(\epsilon_t, \epsilon_{\tilde{t}}) = \exp\{ - |t-\tilde{t}| / \kappa \}$. The horizontal lines denote time averages and the vertical distance between them is $|r_{\theta}(1) - r_{\theta}(2)|$.}
	\label{fig:nonstationary}
\end{figure}

Example \ref{ex:time_zones}, presented in the appendix, illustrates a setting in which contexts represent more comprehensible sources of variation.  
There, a context  indicates a user's country of some app.
We assume this is an observable user feature, 
so the platform can calculate population average weights $x_{\rm pop}$ by looking up the mix of countries among users who opened the app over a long period prior to the start of the experiment. 
Notice that, due to timezone differences, the mix of countries among users arriving during a particular hour within the experiment may not reflect the population proportions.
In our model, the DM can `control for' this source of exogenous variation, which might otherwise confound their inferences.

\subsubsection{A new algorithm: deconfounded Thompson sampling}
 We propose deconfounded Thompson sampling (DTS). It is a more robust variant of the Thompson sampling (TS) algorithm, which is popular  in both academic and industrial contexts \citep{chapelle2011empirical,scott2010modern,russo2018tutorial}.
As observations are gathered, it projects the population-level performance of an arm while controlling for the contexts in which past decisions are made. 
The probability it selects an arm in a given period during the experiment corresponds to the posterior probability of that arm being optimal for population deployment.

To define DTS precisely,  observe that the optimal deployment arm $I^* = \argmax_{i \in [k]} r_{\theta}(i)$ is a random variable, due to its dependence on the uncertain parameter $\theta$.
At any time period $t$ within the experiment, DTS randomly samples an arm $I_t$ to measure with sampling probabilities 
\[ 
\Prob(I_t = i \mid H_t) = \Prob(I^* = i \mid H_t),
\]
where $H_t$ is the full history of  rewards and contexts observable so far. 
Sampling probabilities do not depend on the current context. 
As with standard TS, there is a very simple way to implement this sampling step;  
Algorithm \ref{algo:DTS}, presented in Section \ref{sec:algo}, calculates the posterior mean $\mu_t$ and covariance~$\Sigma_t$ of $\theta$, samples $\tilde{\theta} \sim N(\mu_t, \Sigma_t)$ and picks $I_t \in \argmax_{i \in [k]} r_{\tilde{\theta}}(i)$.  At the end of the experiment, DTS selects the arm  
$ I_{\rm post} \in \argmax_{i \in [k]} \E[ r_{\theta}(i) \mid H_{\rm post}], $  
where $H_{\rm post}$ consists of all observations available at the end of the experiment.

\subsubsection{Numerical illustration: a teaser}

Figure \ref{fig:teaser} is a teaser of a numerical illustration of DTS we provide in Section \ref{sec:numerical}. 
It simulates bandit algorithms applied to a hypothetical week-long experiment, conducted to select an arm to deploy across future weeks. 
Day-of-week effects influence reward observations during the experiment.  
The experiment involves 700 time periods (representing distinct users); the first 100 time periods occur during the context `Monday', the next 100 occur during `Tuesday' and so on. 
Focusing solely on DTS,  Figure \ref{fig:teaser} captures a delicate balance it strikes between exploration and exploitation. By the end of the week, it has explored enough to deploy a near-optimal arm, reflected in its low post-experiment regret $\E[r_{\theta}(I^*) - r_{\theta}(I_{\rm post})]$. 
But it reduces the cost of experimentation by redistributing measurement effort to higher performing arms during the experiment, reflected in its low cumulative within-experiment regret $\E\left[\sum_{\ell=1}^{t} (r_{\theta}(I^*) - r_{\theta}(I_\ell))\right]$.   Section \ref{sec:numerical} also plots two related performance metrics.
\begin{figure}
		\centering
		\includegraphics[width=1\linewidth]{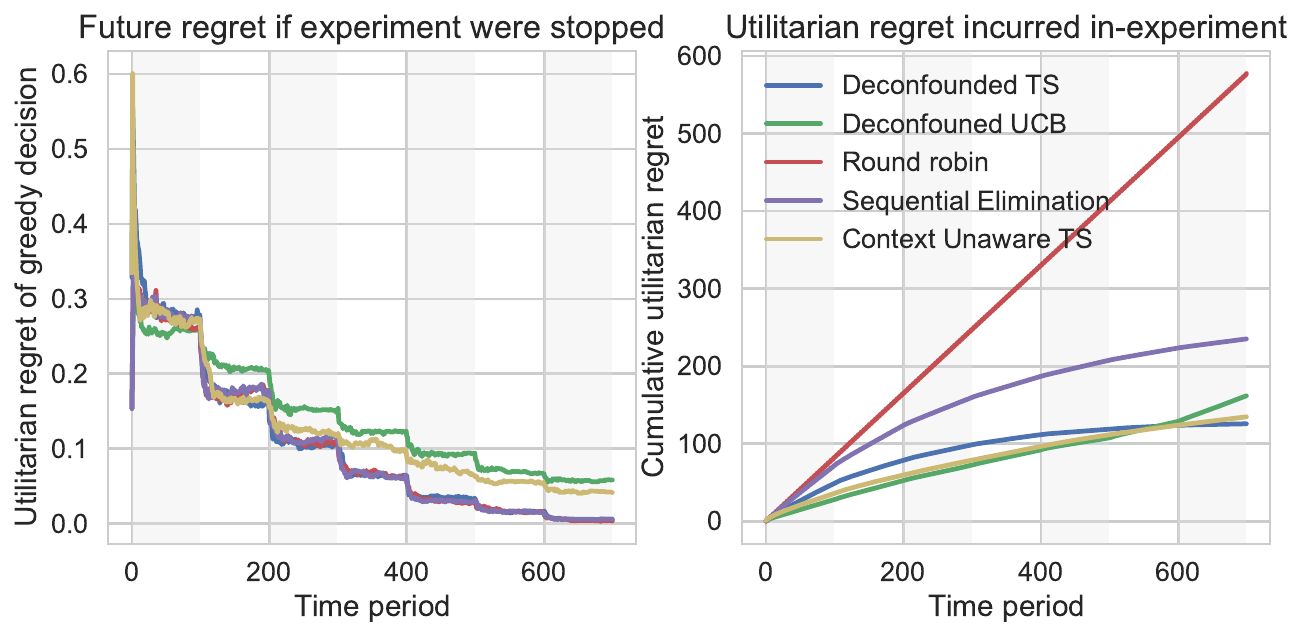} 
	\caption{A teaser illustrating DTS, labeled here as `TS', in a simulated experiment in which day-of-week effects impact arm-rewards. }
	\label{fig:teaser}
\end{figure}

The algorithm labeled `round-robin', is an algorithm that operates like an RCT, sampling arms uniformly throughout the experiment. The algorithm labeled `sequential elimination' brings round-robin closer to DTS. It removes arms from consideration if their posterior probability of being optimal drops below a small threshold.  Round robin, sequential elimination, and DTS all appear to be effective at deploying a (nearly) optimal treatment arm at the end of the experiment. The primary advantage of DTS is its ability to reduce regret incurred during the experiment. 

Other natural bandit algorithms fare poorly in the experiment, failing to gather the information required to select a good arm by the end of the experiment.
One of these is `context unaware TS' --- a standard implementation of TS which (incorrectly) assumes rewards are i.i.d.
Another is deconfounded UCB,  a variant of the upper confidence bound algorithms which dominate much of the literature on exploration in multi-armed bandit problems and reinforcement learning \citep{auer2002finite, auer2008near, rusmevichientong2010linearly}. 
Whereas DTS chooses an arm by maximizing a posterior sample from arms's population-level rewards, deconfounded UCB maximizes and upper confidence bound on the same quantity. 
Appendix \ref{sec:failure} provides formal counterexamples for these two algorithms and also for a third --- a variant of Thompson sampling used for contextual linear bandit problems.

\subsubsection{Theoretical analysis}
DTS is a relatively straightforward adaptation of Thompson sampling to our model. 
Perhaps surprisingly, rigorously understanding its performance required us to develop a completely original approach to analyzing bandit algorithms.
Our proofs, distinct from others in the literature, use inverse propensity weights to analyze the evolution of the posterior distribution.
The theorem statement below is also distinctive --- depending on a what we call ``attainable precision'' rather the length of time horizon. 
At a high-level, the challenge is that learning dynamics in our model markedly deviate from those in i.i.d.~bandit models. Unlike i.i.d.~environments, where the DM can choose to rapidly resolve uncertainty through exploration, our model introduces an unavoidable delay in this process as the DM awaits the  occurrence of relevant contexts.

When specialized to models with i.i.d.~rewards (i.e. no contextual variation), DTS is just standard TS and the theorem below provides per-period regret bounds on the order of $\sigma \sqrt{k/t}$, recovering standard results in the literature. 
More generally, the bound depends on what we call attainable precision --- the inverse posterior variance of an arms population level variance \emph{assuming the DM chose to exclusively measure that arm in all contexts that have occurred so far}. Precision  measures how much uncertainty is resolvable if the DM explored as aggressively as possible. 
A subtle element of this result is that DTS does not explore as aggressively as possible: instead the regret bound in \eqref{eq:intro_within_exp_regret} and its performance in Figure \ref{fig:teaser} suggest it aggressively `exploits' past observations to select good arms.  
\begin{Theorem*}[Informal version of our main result when there is no observation delay]
	Define 
	\[
		{\rm Precision}(X_{1 : t}) \triangleq  \min_{i\in [k]} \,  \frac{1}{{\rm Var}\left(  r_{\theta}( i) \mid   (X_1, R_{1,i}, \ldots, X_t, R_{t, i}) \right)}. 
	\]
	Fix any context sequence $x_{1:T}=(x_1, \ldots, x_T)$ with $\|x_t\|_2 \leq 1$.  Under DTS, within-experiment regret at any time $t\in [T]$ is bounded as 
	\begin{equation}\label{eq:intro_within_exp_regret}
	\E[r_{\theta}(I^*) - r_{\theta}(I_t) \mid X_{1:t} = x_{1:T}]  \leq  \tilde{O}\left( \sqrt{  \frac{k}{  {\rm Precision}(x_{1 : (t-1)})} }\right)
	\end{equation}
	and post-experiment regret is bounded as 
	\[
	\E[r_{\theta}(I^*) - r_{\theta}(I_{\rm post}) \mid X_{1:T}=x_{1:T}] \leq  \tilde{O}\left( \sqrt{  \frac{k}{  {\rm Precision}(x_{1 : T})} }\right).
	\]
\end{Theorem*} 
Our full theoretical results extend the above theorem in substantial ways. First, they accommodate settings in which the DM only observes rewards within the experiment after some delay. 
Second, the appendix provides additional results that are more similar to past literature:  Proposition \ref{prop:contextual} bounds what we term the total ``within-experiment contextual regret'' of DTS and Proposition~\ref{prop:contextual_generalized} extends the bound to a extension of DTS that aims to learn personalized policies.

\subsection{Connections to the literature}

\paragraph{Learning with resilience to exogenous nonstationarity.} 
 Two approaches, Thompson sampling and upper confidence bound algorithms, dominate much of the literature on multi-armed bandit algorithms.
 However, we are not aware of any previous papers examining their ability to identify an effective treatment arm despite exogenous nonstationary variation. 
 
 A large literature on nonstochastic bandit problems considers a related goal: they aim to design procedures that  earn rewards within the experiment which are competitive with that of the best stable decision, even when reward sequences are not i.i.d.  
 This literature was launched by \cite{auer2002nonstochastic} and is reviewed in \cite{lattimore2020bandit}. Our work is a substantial departure, making precise comparisons difficult. Our Bayesian model emphasizes the role of contextual variation as a driver of nonstationarity and prioritizes the quality of post-experiment decision-making.
 The nonstochastic MAB literature instead assumes rewards are picked by an intelligent adaptive adversary and aims to attain low within-experiment regret despite this fact. 
 A few papers \citep{abbasi2018best, jamieson2016non} which study the problem of nonstochastic best-arm selection are more similar in (implicitly) considering post-experiment performance, but still differ in how nonstationarity is modeled.  
  Algorithm design in the adversarial bandit literature is usually tightly coupled to worst-case theoretical bounds, typically resulting in algorithms which are much more conservative than Thompson sampling.  
  Appendix \ref{sec:adversarial} provides a more precise discussion of nonstochastic bandit models.

 A related paper by \cite{farias2022synthetically} was posted online concurrently with our paper. Their model is most similar to Example \ref{ex:latent_confounders} and the reward-model in Figure \ref{fig:nonstationary}, in that an exogenous time trend additively shifts all arm's rewards. They assume access to observations of non-experimental units and use synthetic control techniques to estimate and control for the exogenous trend. 
 One technical difference is that our results
 \ref{thm:main_result} has (essentially) no dependence on the dimension of the context space, suggesting that our algorithms use contexts to deconfound with minimal cost. It is an open question whether such guarantees are possible in the setting of \cite{farias2022synthetically}.

 \paragraph{Adapting decisions to respond to exogenous variation.} 
 Our focus on reaching a stable decision despite non i.i.d.~exogenous variation distinguishes this work from most of the literature on decision-making in nonstationary environments. Works like \cite{mellor13, besbes15, cheung19, trovo20, yadkori22} and \cite{suk22} focus on adapting decision-making rules as the environment evolves. These papers may provide a natural model for a recommendation system where items, which represent the arms, may lose relevance over time, requiring an adaptable system. 
 Our model is particularly well-suited to scenarios such as the A/B testing problem described previously.
 

Similarly, the focus on reaching a stable decision distinguishes our work from a large literature which emphasizes how decision-making can respond to evolving context. For instance, in standard linear contextual bandit models \citep{li2010contextual}, the DM aims to converge on a decision-rule mapping contexts to actions that maximizes expected reward accrued in each specific context. In our model, we use contextual observations to draw reliable inferences from past reward observations, rather than as input for a context-reactive decision rule.
Although our model reflects a common experimental practice, it is a departure from much of the bandit literature. As a result, we provide a thorough discussion in Appendix \ref{sec:policy_learning}. That section includes a generalization of DTS for learning personalized decision-rules. Appendix \ref{sec:failure} establishes that, without modification, contextual bandit algorithms can fail for our objective. 

\paragraph{Within-experiment and post-experiment decision quality.}
Our paper is somewhat atypical in considering both within-experiment and post-experiment decision quality. 
In one of the the most classical formulations of a multi-armed bandit, due to \cite{lai1985asymptotically} one aims to minimize exploration costs while, in the long-run, almost always choosing an optimal action. 
There is no notion of post-experiment decisions, and the sole performance measure is what we call within-experiment regret. 
Another segment of the literature, focuses solely on post-experiment performance. Papers in this literature go by a variety of names, including pure-exploration in MABs \citep{bubeck2009pure}, best-arm identification \citep{kaufmann2016complexity}, or ranking and selection \citep{kim2006selecting}.

In models with i.i.d.~reward observations, what we call post-experiment regret is widely studied. It is often called ``simple regret'' \citep{bubeck2009pure} or ``expected opportunity cost''  \citep{frazier2008knowledge}.  These differ from another common performance metric, which considers only the probability a suboptimal arm is selected, because it more severely penalizes selection of very low quality arms. 
Studying combined objectives is quite natural. See the rich decision-theoretic model of clinical trials in \cite{chick2021bayesian}, for example. 
Rather than combine within-experiment and post-experiment regret into a single coherent objective function, we treat DTS as a heuristic that does not perfectly optimize any goal. We study its performance according to both regret measures.
Other papers that study both within- and post- experiment decision quality include \cite{degenne19a, caria2020adaptive, athey2022contextual, krishnamurthy2023proportional} and \cite{zhong2023achieving}.

\paragraph{Learning with resilience to delayed reward observations.}
In cases with no contextual variation, DTS corresponds to standard Thompson sampling. Even then, Theorem \ref{thm:main_result} is notable in providing guarantees when reward observations are subject to delay. Our bound on post-experiment regret in the second part of Theorem \ref{thm:main_result} permits delay in observing rewards as long as the experiment itself, paralleling a situation where all arm pulls must be pre-determined at the experiment's outset. 
Previous work by \cite{kandasamy2018parallelised} provided guarantees for a Thompson sampling which allocates a batch of arm selections at once; however, their performance guarantees degrade with increasing batch size.
A related preprint by \cite{wu2022thompson} was posted online concurrently with our paper, showing that vanilla Thompson sampling outperforms many algorithms designed specifically to address problems with delayed rewards. The first part of Theorem \ref{thm:main_result}, which bounds within-experiment regret, is different from and complementary to their theoretical bounds.
Beyond results on Thompson sampling, a number of MAB papers establish theoretical bounds on regret when observations are subject to delay. See for instance \cite{dudik2011efficient, joulani2013online, zhou2019learning} and references therein.

\section{Formal problem formulation allowing for observation delay}\label{sec:formulation}

We provide a complete problem formulation that is more formal than the one contained in Subsection \ref{subsec:model_intro}.
One substantive generalization is that we allow for a reward observation delay of $L\geq 1$ periods. 
This means that the arm selection at time $t$ must be based on rewards associated with arms played more than $L$ periods earlier, which constrains adaptivity within an experiment. Nevertheless, the post-experiment arm deployment decision can still incorporate the full experiment results $(I_1, R_{1,I_1},\ldots, I_T, R_{T, I_T})$. That is, we imagine that the DM waits for reward realizations before population deployment. 

\paragraph{Discussion of modeling choices.}
 The presentation here is mathematically precise but omits discussion of subtle modeling choices. Some of these modeling choices were already discussed briefly in Subsection \ref{subsec:model_intro}. The first sections of the appendix
 provide more detailed comparisons to the literature; see Appendix \ref{sec:policy_learning} for a discussion of connections to contextual bandit models and Appendix \ref{sec:adversarial} for a discussion of adversarial nonstationary bandit models . The reader may choose to skip to those section after reading the formulation, or may proceed directly to the main results.  To understand the flexibility of this abstract modeling framework, one might look to various examples we present; see Example \ref{ex:latent_confounders} in Section \ref{sec:intro}, Example \ref{ex:day_of_week_intro} in Section \ref{sec:numerical} and Examples \ref{ex:time_zones} and \ref{ex:combined} in Appendix~\ref{sec:more_examples}. 

\paragraph{Mathematical notation.} For an integer $k$, we write $[k] = \{1,\ldots, k\}$. For a sequence $x_1, x_2, \ldots$, we use the ``Matlab style'' indexing notation $x_{m:n} = (x_m, \ldots, x_n)$ to refer to sub-sequences. All vectors in this paper are viewed as column vectors.  We use $\langle x, y\rangle =x^\top y$ to denote the standard inner product between two vectors. For three random variables $X,Y$ and $Z$, the notation $X\perp Y$ means that $X$ and $Y$ are independent and $X\perp Y \mid Z$ means they are independent conditioned on $Z$.

\paragraph{Our model.}
The DM would like to deploy the utilitarian optimal arm $I^* \in \argmax_{i \in [k]} r_{\theta}(i)$ in the population, where $r_{\theta}(i)$ denotes the population average reward of arm $i$. 
 We model the population average reward as the average over heterogeneous conditional average rewards among contexts drawn from some population distribution:
\begin{equation}\label{eq:pop_avg_reward}
	r_{\theta}(i) = \E_{x\sim \Dc_{\rm pop}} \left[  r_{\theta}(i,x) \right] = \langle \theta^{(i)} \, ,\, x_{\rm pop}\rangle,
\end{equation}
where $\Dc_{\rm pop}$ is a distribution over $d$ dimensional context vectors, $x_{\rm pop} = \E_{x\sim \Dc_{\rm pop}}\left[ x\right]\in \mathbb{R}^d$ is the mean context vector, $r_{\theta}(i, x) = \langle \theta^{(i)}, x \rangle$ is a linear model governing how mean-rewards vary across contexts.
As discussed in the introduction, we assume that the DM knows $x_{\rm pop}$. 
But the DM is uncertain about the parameter $\theta = (\theta^{(1)}, \ldots, \theta^{(k)}) \in \mathbb{R}^{k\cdot d}$, and knows only that it is drawn from a multivariate Gaussian prior, denoted $\theta \sim N(\mu_1, \Sigma_1)$ where $\mu_1\in\mathbb{R}^{k\cdot d}$ and $\Sigma_1\in\mathbb{R}^{kd\times kd}$.

To resolve uncertainty, the DM conducts a $T$ period experiment.  In any period $t \in [T]$ during the experiment, the DM observes a context $X_t \in \mathbb{R}^d$ and chooses an arm $I_t \in [k]$. 
After a delay of $L\geq 1$ periods, they observe a reward $R_{t,I_t}$
associated with the selected arm. 
Formally, the \emph{potential reward} of arm $i$ at time $t$ is 
\begin{equation}
	R_{t,i} = r_{\theta}(i,X_t) + W_{t,i},
\end{equation}
where  $W_{t,i} \sim N(0, \sigma^2)$ is i.i.d.~Gaussian noise that is assumed to be jointly independent of $\theta$, the contexts $X_{1:t}$ and the decisions $I_{1:t}$.  

The sequence of contexts $X_{1:T}$ within the experiment is drawn from a distribution $\Dc_{\rm exp}$ over $\Xc^T$, where $\Xc  \subset \{x  \in \mathbb{R}^d : \|x\|_2 \leq 1     \}$ is a subset of context vectors with bounded norm;
an important special case is where $\Dc_{\rm exp}$ is a point mass on a particular sequence $x_{1:T}$. 
The algorithms that we study do not require prior knowledge of $\Dc_{\rm exp}$. 
We assume the draw of $X_{1:T}$ is independent of $\theta$, so that the DM cannot resolve their uncertainty by passively observing contexts, and assume that $X_{(t+1):T}\perp I_{1:t} \mid X_{1:t}$, so that the DM cannot purposefully influence future contexts through their arm selection.

The DM employs a \emph{policy} $\pi= (\pi_1, \ldots, \pi_T, \pi_{\rm post})$. 
To treat randomized policies, we will allow the policy to take as input random seeds $(\xi_1, \ldots, \xi_T, \xi_{\rm post})$, which are drawn i.i.d., and are independent from the context sequence, potential rewards, and $\theta$.
For a period $t\in [T]$ within the experiment, $\pi_t$ determines an arm to sample as
\[ 
I_t = \pi_t(\underbrace{H_t}_{\text{history}} \,,\, \underbrace{X_t}_{\text{context}}\,,\, \underbrace{\xi_t}_{\text{seed}}) \quad \text{where} \quad H_t \triangleq \left(X_{1:(t-1)}, I_{1:(t-1)},  R_{1:(t-L)} \right);
\] 
for notational convenience, we write $R_\ell\triangleq R_{\ell,I_\ell}$ for $\ell\in [T]$ and define $R_{1:(t-L)}=\emptyset$ for any $t\in [L]$.
At the end of the experiment, the post experiment decision-rule $\pi_{\rm post}$ selects an to deploy in the population as
\[ 
I_{\rm post} = \pi_{\rm post}(\underbrace{H_{\rm post}}_{\text{final history}} \,,\,  \underbrace{\xi_{\rm post}}_{ \text{seed}}) \quad \text{where} \quad H_{\rm post} \triangleq \left(X_{1:T}, I_{1:T},  R_{1:T} \right).
\]

We consider two measures of the performance of an algorithm $\pi$:
\begin{description}
	\item[Expected post-experiment (utilitarian) regret:] $\E[\Delta_{\rm post}]$ where  $\Delta_{\rm post} \triangleq r_{\theta}(I^*) - r_{\theta}(I_{\rm post})$.
	\item[Expected within-experiment (utilitarian) regret:] $\E[\Delta_t]$ where $\Delta_t \triangleq r_{\theta}(I^*) - r_{\theta}(I_t)$.
\end{description}

The term `utilitarian' is taken from \cite{athey2021policy} and reflects that an arm's performance is measured in terms of its average reward or `utility' it generates within a population. Appendix \ref{sec:contextual_regret}  studies a measure of regret on the contexts encountered within the experiment.

Post-experiment regret  measures whether the policy is able to choose an arm with near-optimal population average reward at the end of the experiment. 
Within-experiment (utilitarian) regret captures whether a decision made within the experiment has near-optimal population average reward. 
Attaining low within-experiment regret indicates that the DM was able to select arms of similar quality within the experiment to the arm they hoped to employ post-experiment (with `similarity' assessed through $r_{\theta}(\cdot)$); This can be thought of as an indicating a reduced cost of experimentation.

We make a couple of extra assumptions to simplify the presentation. First, we assume $x_{\rm pop} \neq 0$. Otherwise, it is known at the beginning of the experiment that each arm's population average reward is zero.  
Next, we assume that $\Sigma_1$ has full rank (although some eigenvalues could be arbitrarily small). This allows us to write some expressions in terms of matrix inverses.
Together, the two assumptions imply that, with probability 1, there is a unique solution to the maximization problem defining $I^*$ and we do not need to discuss tie-breaking rules.  
For similar reasons, assume $\sigma>0$. 

\subsection{Remarks on interpretation}

\begin{remark}[A unified objective function]\label{rem:overall-objective}
	Roughly speaking, we interpret attaining low post-experiment regret as a constraint on the policies a decision-maker could employ. In practice, an experimenter is unlikely to knowingly choose a policy that is incapable of deploying an effective arm to the population. Subject to this (loosely defined) constraint, we seek an policy that reduces experimentation costs by minimizing within-experiment regret. This kind of evaluation is clearest in the the interpretation of the experiment results in Section \ref{sec:numerical}. There, some algorithms are effectively disqualified due to suffering high post-experiment regret. Several algorithms attain comparable post-experiment regret, but nevertheless differ substantially in the regret they incur within the experiment. 
\end{remark}

\begin{remark}[Connection to i.i.d.~bandits]
	Classical bandit models with i.i.d.~reward observations are a special case of the model in which there is no variation in contextual observations. Specifically, take contexts to be one dimensional ($d=1$), and assume that $X_{1} = X_{2} = \cdots = X_{T} = x_{\rm pop}$. Then potential reward observations $(r_{1,i}, \ldots, r_{T,i})$ are i.i.d.~samples with mean $r_{\theta}(i)$. In this special case, post-experiment regret is often called ``simple regret'' \citep{bubeck2009pure}.  We specialize our results to this case in Corollary \ref{cor:main}.   
\end{remark}

\begin{remark}[Interpretation of post-experiment regret]
	Our discussion implicitly imagines that treatment decisions continue after the end of the experiment. 
	Here, we make that explicit. Extend the time horizon by $N$ periods; The DM chooses arm $I_{t} = I_{\rm post}$ for all $t \in \{T+1, \ldots, T+N\}$. Then, the reward earned post experiment is:
	\begin{align*}
	  \sum_{t=T+1}^{T+N}  r_{\theta}( I_t, X_t) = 	 \sum_{t=T+1}^{T+N}  r_{\theta}(I_{\rm post}, X_t)  = \sum_{t=T+1}^{T+N}   \langle \theta^{(I_{\rm post})} \, ,\, X_t  \rangle  = N \cdot  r_{\theta}( I_{\rm post} , \hat{X}_{\rm pop} )  \approx N \cdot r_{\theta}(I_{\rm post})
	\end{align*}
	where $\hat{X}_{\rm pop} = \frac{1}{N}\sum_{t=T+1}^{T+N} X_t$ is the average post-experiment context and the final approximate equality holds when $\hat{X}_{\rm pop}  \approx x_{\rm pop}$. The approximate equality is exact if $\hat{X}_{\rm pop}  = x_{\rm pop}$.  What we call post-experiment utilitarian regret is the per-period expected regret of the post-experiment decision under a post-experiment context distribution whose mean matches the DM's target context weights $x_{\rm pop}$. 
\end{remark}

\begin{remark}[Comparison to the probability of incorrect selection]
		One might also be interested in comparing post-experiment regret to the probability of incorrect selection, $\Prob(I^*\neq I_{\rm post})$, a metric that is widely studied in the literature. We can write 
	\[
	\Prob(I_{\rm post} \neq I^*) = \E\left[ \sum_{i\neq I^*} \ind(I_{\rm post} = i) \right]  \quad \& \quad \E[\Delta_{\rm post}] = \E\left[\sum_{i\neq I^*} \ind(I_{\rm post} = i) \left( r_{\theta}(I^*) - r_{\theta}(i)  \right) \right],
	\]
	revealing that post-experiment regret is similar to the probability of incorrect selection, except it is more forgiving of instances where ``incorrect'' but very nearly optimal arms are deployed post-experiment. 
\end{remark}

\section{Deconfounded Thompson sampling}\label{sec:algo}
We propose deconfounded Thompson sampling (DTS). It is the natural way of applying Thompson sampling to our problem. At each time period $t\in [T]$, it selects an arm to measure randomly by sampling from the posterior distribution of the optimal arm:
\begin{equation}\label{eq:probability_matching_def}
\Prob(I_t = i \mid H_t, X_t) = \Prob(I^* = i \mid H_t, X_t), \quad \forall i\in[k]. 
\end{equation}
At the end of the experiment, DTS chooses the arm with highest expected reward in the population under posterior beliefs:
\begin{equation}\label{eq:bayes_selection}
	I_{\rm post}  \in \argmax_{i \in [k]} \E\left[  r_{\theta}(i)  \mid   H_{\rm post} \right].
\end{equation}
These definitions make no explicit reference to contextual observations. 
But implicitly, through proper Bayesian inference, DTS is using contextual observations to 'deconfound' its reward observations.  
Full pseudocode is given below.

\begin{algorithm}[H]
	\SetKwComment{Comment}{/* }{ */}
	\SetAlgoLined
	Input prior parameters $\left(\mu_{1}, \Sigma_{1}\right)$, population weights $x_{\rm pop}$ and noise variance $\sigma^2$ \; 
	Define the feature map $\phi: [k]\times \mathbb{R}^d \to \mathbb{R}^{kd}$ by 
	$
	\phi(x,i) = (0, \ldots, 0,\underbrace{ x_1, \ldots, x_d}_{ i\text{-th subvector} }, 0,\ldots, 0)$ \;
	Define linear reward $r_{\theta}(i) = \langle \theta \, ,\, \phi(x_{\rm pop}, i) \rangle$ \;
	Start with empty history $\tilde{H}_1=\{\}$ \;	
	\For{$t=1,2,\ldots, T$}{
		\If{$t\geq L+1$}{
				Gather (potentially delayed) observation $O_{t-L} \gets (X_{t-L}, I_{t-L}, R_{t-L})$ \;
				Update history: $\tilde{H}_t \gets \tilde{H}_{t-1} \cup \{O_{t-L} \}$ \;
			    \Comment*[l]{Update posterior mean and covariance}  
                $\phi_{t-L} \gets \phi(X_{t-L}, I_{t-L}) $\;
				 $\Sigma_t \gets \left( \Sigma_{t-1}^{-1} + \sigma^{-2} \phi_{t-L} \phi_{t-L}^\top \right)^{-1}$,  \quad i.e. $\Sigma_t = \left( \Sigma_1^{-1} + \sigma^{-2} \sum_{\ell=1}^{t-L} \phi_\ell \phi_\ell^\top \right)^{-1}$\; 
				 $\mu_t \gets \Sigma_t\left(\Sigma_{t-1}^{-1} \mu_{t-1} + \sigma^{-2}\phi_{t-L}  R_{t-L}\right)$, \quad i.e. $\mu_t = \Sigma_t\left( \Sigma_{1}^{-1} \mu_1 + \sigma^{-2}\sum_{\ell=1}^{t-L} \phi_\ell  R_{\ell}\right)$\;
		}
	\Else{
        $\tilde{H}_t \gets \tilde{H}_{1} $ \;
        $(\mu_t, \Sigma_t) \gets (\mu_1, \Sigma_1)$\;
		}
		\Comment*[l]{Sample from the posterior distribution of the optimal arm}  
		Sample $\tilde{\theta} \sim N(\mu_t, \Sigma_t)$ and choose treatment arm $I_t \in \argmax_{i \in [k]} r_{\tilde{\theta}}(i)$ \;
	}
	Wait to observe $O_{T-L+1},\ldots, O_{T}$\; 
	Form post-experiment history: $H_{\rm post} \gets \{O_{1}, \ldots, O_{T}\}$\; 
	Form posterior mean: $\mu_{\rm post} \gets \left( \Sigma_1^{-1} + \sigma^{-2} \sum_{\ell=1}^{T} \phi_\ell \phi_\ell^\top \right)^{-1} \left( \Sigma_{1}^{-1} \mu_1 + \sigma^{-2}\sum_{\ell=1}^{T} \phi_\ell  R_{\ell}\right) $\;
	Choose arm to deploy in population: $I_{\rm post}\in\argmax_{i \in [k]}  r_{\mu_{\rm post}}(i)$\;
 \caption{DTS in Gaussian contextually confounded experiments}
	\label{algo:DTS}
\end{algorithm}

A striking feature of DTS is that the decision at time $t$ does not depend on the context at time $t$ --- or even contexts in the past $L-1$ periods. That is, in \eqref{eq:probability_matching_def}, 
\begin{equation}\label{eq:context_indep_probs}
\Prob(I^* = i \mid H_t, X_t)= \Prob\left(I^* = i \mid X_{1:(t-L)},  I_{1:(t-L)}, R_{1:(t-L)} \right).
\end{equation}
This equation uses that, conditioned on the observations $X_{1:(t-L)}, I_{1:(t-L)}, R_{1:(t-L)}$,  the latent variable $\theta$ is independent of the additional arm selections and observed contexts. 
That decisions are \emph{context independent} in this way could offer substantial practical benefits. 
Even if contexts are logged, enormous engineering resources might be required to develop a system that observes contexts and responds in real time. 
For instance, assessing $X_t$ could easily require querying several different datasets containing the current user's interaction history and then applying a trained machine learning algorithm that generates a compact feature vector from this history. 
With a context independent algorithm, this could be done without substantial latency requirements.

\section{Main result }\label{sec:main_result}

\subsection{Warmup: bound in vanilla bandit environments}
To build intuition, we first consider a special case of the our result that applies to vanilla bandit problems  In this case, DTS is just standard TS and the results we provide here are (essentially) known. By presenting them in a style that mirrors our main theorem, we hope to make it easier to digest the main theorem itself. 

Under the next assumption, potential arm rewards $(R_{t,i})_{t\in [T]}$ are i.i.d.~with mean $r_{\theta}(i)$ and rewards are observed immediately after an arm is played. 
\begin{assumption}[Vanilla bandit problem]\label{assumption:iid}
	Suppose that $L=1$ (no delay),  the context dimension is $d=1$, and with probability 1,  $X_{1} = X_{2} = \cdots = X_{T} = x_{\rm pop}$.
\end{assumption}

Under this assumption, DTS is just standard TS followed by selecting the arm \eqref{eq:bayes_selection} at the end of the experiment. 
Summing the bound in \eqref{eq:within_exp_regret_bound_iid} over\footnote{Technically \eqref{eq:within_exp_regret_bound_iid} can only be summed over $t\geq 2$. It is easy to provide separate bounds when $t=1$.} $t\in [T]$, yields familiar $\tilde{O}(\sqrt{kT})$ cumulative regret bounds for Thompson sampling \citep[See e.g.][]{russo2016information}. 
The form in \eqref{eq:within_exp_regret_bound_iid} is stronger, since it bounds performance loss in every period, rather than on average. 
The bound in \eqref{eq:post_exp_regret_bound_iid} ensures that TS gathers the information required to select an effective arm at the end of the experiment. 
This result is not commonly stated in the literature, but it is implied by the algorithm's cumulative regret bounds; See \cite[Proposition 8]{russo2018learning}. 

Recall that $\Delta_t= r_{\theta}(I^*) - r_{\theta}(I_t)$ is the regret of the exploratory actions picked by DTS within the experiment. The post-experiment regret $\Delta_{\rm post }= r_{\theta}(I^*) - r_{\theta}(I_{\rm post})$ is the regret of the arm $I_{\rm post} \in \argmax_{i \in [k]} \E[r_{\theta}(i) \mid H_{\rm post}]$ which maximizes posterior expected reward given all the information acquired throughout the experiment. It is possible to show that, in general, $\E[\Delta_{\rm post }] \leq \E[\Delta_t]$ for any $t$, since $I_{\rm post}$ is selected based on more information and does not involve exploration.  In this sense, \eqref{eq:within_exp_regret_bound_iid} is the stronger and more surprising property. 
\begin{corollary}\label{cor:main}
	 If Assumption \ref{assumption:iid} holds, then under DTS, within-experiment regret is bounded as
	\begin{equation}\label{eq:within_exp_regret_bound_iid}
		\E\left[   \Delta_{t} \right] \leq   \sigma \sqrt{ \frac{ 2\cdot\iota \cdot k \cdot \log(k) } {   t-1  } }, \quad \forall t \in \{2, \ldots, T\},
	\end{equation}
		where  $\iota$ is defined in \eqref{eq:iota}. Post-experiment regret is bounded as
	\begin{equation}\label{eq:post_exp_regret_bound_iid}
		\E\left[   \Delta_{\rm post} \right] \leq   \sigma\sqrt{ \frac{ 2\cdot\iota \cdot k \cdot \log(k) }{ T}} \,.
	\end{equation}
\end{corollary}
Our use the term $\iota$ to capture a messy factor which comes from the application of a concentration inequality.  In our main regime of interest, $\iota$ is a numerical constant, so we defer discussion until after our main theorem.

\subsection{General result}

We seek a generalization of Corollary \ref{cor:main} that holds throughout the scope of our problem formulation, removing the need for Assumption \ref{assumption:iid} and establishing the robustness of DTS to exogenous nonstationary variation. 
The main intellectual challenge is that learning dynamics in our model markedly deviate from those in i.i.d.~bandit models. In i.i.d.~environments, the DM can choose to quickly resolve uncertainty about an arm's population-level performance through exploration. 
By contrast, our model introduces an unavoidable delay in this process as the DM awaits the occurrence of relevant contexts. 
A bound like \eqref{eq:within_exp_regret_bound_iid}, which says that DTS makes near optimal decisions as soon $t$ is large, may not be possible under some context sequences. 

Instead of depending explicitly on the number of arm pulls $t$, our bound depends on a what we call \emph{attainable precision}, defined as 
\begin{align} \label{eq:precision_def}
	{\rm Precision}(X_{1 : t}) &\triangleq  \min_{i\in [k]} \,  \frac{1}{{\rm Var}\left(  r_{\theta}( i) \mid   (X_1, R_{1,i}, \ldots, X_t, R_{t, i}) \right)}  \\
	&=  \min_{i\in [k]} \,  \left( x_{\rm pop}^\top \left[ {\rm Cov}\left(\theta^{(i)}\right)^{-1} + \sigma^{-2} \sum_{\ell=1}^{t} X_\ell X_\ell^\top      \right]^{-1} x_{\rm pop} \right)^{-1} .  \label{eq:precision_formula}
\end{align}
To treat cases with no reward observations, define ${\rm Precision}(X_{1:(t-L)}) \triangleq  \min_{i\in [k]}   \frac{1}{{\rm Var}\left(  r_{\theta}( i)\right)}$ when $t\leq L$. 
Precision is the inverse posterior variance of the arm's population average reward if the potential reward outcomes $(R_{\ell, i})_{\ell =1,\ldots, t}$ from measuring the arm in contexts $X_{1}, \ldots, X_{t}$ were observable. 
The formula in \eqref{eq:precision_formula} uses standard rules for computing Gaussian posterior distributions. Under Assumption \ref{assumption:iid}, ${\rm Precision}(X_{1 : t}) \geq \sigma^{-2} \cdot t$ and Theorem \ref{thm:main_result} implies the corollary stated above.
More generally, if the contexts so far are reflective of the population distribution (e.g. they are drawn i.i.d.), then precision scales as $\sigma^{-2} \cdot t$, with no or minimal dependence on context dimension; See Lemma \ref{lem:bounds_on_precision}. 
But precision can behave quite differently if contexts have a strong non-stationary pattern; see Figure \ref{fig:precision} in Section \ref{sec:numerical}.

Attainable precision measures whether the decision-maker \emph{could have} precisely estimated an arm's population average reward by playing it in each context observed so far in the experiment. The next theorem formalizes a striking result about DTS: once high precision is attainable, the expected regret of each subsequent decision made by DTS is low. The result generalizes Corollary~\ref{cor:main} to problems with exogenous nonstationary variation and delayed reward observations. Full discussion is deferred until Subsection \ref{subsec:main_discuss}.

\begin{restatable}[Bound on within- and post-experiment utilitarian regret]{Theorem}{utilitarianRegret}\label{thm:main_result}	
	Fix any sequence $x_{1: T} \in \Xc^T$. Under DTS, within-experiment regret is bounded as 
	\begin{equation}\label{eq:within_exp_regret_bound}
		\E\left[   \Delta_{t}  \mid X_{1:T} = x_{1:T}\right] \leq   \sqrt{ \frac{ 2\cdot\iota \cdot k \cdot \log(k) } { {\rm Precision}\left(x_{1 : (t-L)}\right) } }\,, \quad \forall t \in [T],
	\end{equation}
where $\iota$ is defined in Equation \eqref{eq:iota}. Post-experiment regret is bounded as
	\begin{equation}\label{eq:post_exp_regret_bound}
		\E\left[   \Delta_{\rm post} \mid X_{1:T} = x_{1:T}\right] \leq   \sqrt{ \frac{ 2\cdot\iota \cdot k \cdot \log(k)   }{ {\rm Precision}(x_{1 : T}) }}\,.
	\end{equation}
\end{restatable}
We define
\begin{align}\label{eq:iota}
	\iota &\triangleq \max\left\{8 \left(\sigma^{-2} /\lambda_{\min}\left(\Sigma_1^{-1}\right) \right)\cdot \log\left( d k \lambda_{\max}\left( \Sigma_1\right) \left[\lambda_{ \max }\left( \Sigma_{1}^{-1} \right) + \sigma^{-2}T\right]\right)  + 1 \, , \, 9  \right\} \\ \nonumber 
	&= \tilde{O}\left(  \max\{ \underbrace{\lambda_{ \max }\left( \Sigma_1 \right) / \sigma^2}_{ \text{signal-to-noise ratio} } \, , 1 \}  \right),
\end{align}
where $\tilde{O}$ hides logarithmic factors.
This term comes from applying a concentration inequality to control for the impact of randomness in action selection; see inequality $(a)$ in the proof sketch in Section \ref{subsec:analysis_overview}. We are interested in problems with a low signal-to-noise ratio --- where a single user interaction does not resolve much uncertainty --- in which case $\iota$ is a constant.  

\begin{remark}[Treating $\iota$ as a constant]\label{rem:iota}
	In choosing to downplay the importance of $\iota$, we are implicitly assuming that the signal-to-noise ratio is $\tilde{O}(1)$, i.e. we are in a regime where observing a single reward realization does not resolve most prior uncertainty. Indeed, many A/B tests  involve just a few treatment arms, but still require (many) millions of users to attain statistical power.  A line of the literature formally studies such a regime by taking a diffusion limit of bandit problems \citep{kuang2023weak,fan2021diffusion,araman2022diffusion,adusumilli2023risk} which is similar to letting $\lambda_{ \max }\left( \Sigma_1 \right) / \sigma^2 \to 0$  but taking the time horizon $T\to \infty$ at a comparable rate.  In such a limit, $\iota = 9$, a numerical constant that comes from crude application of concentration inequalities. 
	
	Notice that our overall regret bounds do not degrade as the noise variance $\sigma^2$ tends to zero. In that case, two factors of $\sigma$, one in $\iota$ and the other in the right-hand-side of \eqref{eq:within_exp_regret_bound_iid} or \eqref{eq:post_exp_regret_bound_iid}, will cancel.  However, our analysis is not well suited to tightly bounding  the regret incurred when $\sigma \approx 0$. 
\end{remark}

\subsection{Growth rate of attainable precision}\label{subsec:precision_growth}
In benign settings, where observed contexts  are generally reflective of the population distribution, precision in period $t$ scales with $\sigma^{-2} \cdot t$ and does not depend on the context dimension. In such cases, the bounds in Theorem~\ref{thm:main_result} are roughly on the order of $\sigma \sqrt{k/(t-L)}$ or $\sigma \sqrt{k/T}$. 

The next lemma provides four results in such settings. The first result is a generic bound from which other bounds follow.
The second considers standard $k$-armed bandit problem, viewed as a special case of our formulation. 
The third generalizes the second, allowing for arbitrary context order while requiring that the empirical mean of the contexts matches the population mean. 
The fourth result integrates the first result with concentration inequalities applied to sample covariance matrices.

\begin{restatable}[Bound on attainable precision]{lemma}{precisionBounds}\label{lem:bounds_on_precision} 
	Fix any sequence $x_{1: T} \in \Xc^T$ and $t\in[T]$.
 \begin{enumerate}
		\item(Generic bound) Let $S_x \triangleq \frac{1}{t}\sum_{\ell=1}^{t} x_{\ell}x_\ell^\top$ denote the empirical second moment matrix and $\tilde{S}_x \triangleq S_x+\frac{\sigma^2\cdot\lambda_{\min}\left(\Sigma_1^{-1}\right)}{t} I\,\,$ (where $I\in\mathbb{R}^{d\times d}$ is an identify matrix). 
  Then 
		\[
		{\rm Precision}(x_{1 : t}) \geq  \sigma^{-2} t  \cdot  \left(x_{\rm pop}^\top  \tilde{S}_x^{-1}   x_{\rm pop}\right)^{-1}.
		\]
		\item (Vanilla bandit) Suppose $d=1$ and $x_\ell=1=x_{\rm pop}$ for each $\ell\in [t]$.  Then
		\[
		{\rm Precision}(x_{1 : t}) = \min_{i\in [k]}\, \Sigma_{1,ii}^{-1}+  \sigma^{-2} t \geq \lambda_{\min}\left(\Sigma_1^{-1}\right) + \sigma^{-2}t,
		\]
  where $\Sigma_{1,ii}$ is the $(i,i)$-th element of the prior covariance matrix $\Sigma_1$.
		\item (No empirical distribution shift) Suppose $\frac{1}{t} \sum_{\ell=1}^{t} x_\ell = x_{\rm pop}$. Then 
		\[
		{\rm Precision}(x_{1 : t}) \geq \lambda_{\min}\left( \Sigma_1^{-1} \right)\|x_{\rm pop}\|_2^{-2} + \sigma^{-2} t.
		\]
		\item (I.i.d.~contexts) 
  Suppose $X_1, \ldots, X_t$ are drawn i.i.d.~from a distribution satisfying that $\mathbb{E}[ X_1 X_1^\top] \succeq  c\cdot x_{\rm pop} x_{\rm pop}^{\top}$ for some $c\geq 0$. Then for any $\delta>0$,
  with probability greater than $1-\delta$, 
		\[
		{\rm Precision}(X_{1 : t}) \geq \lambda_{\min}\left( \Sigma_1^{-1} \right)\|x_{\rm pop}\|_2^{-2} + c\cdot \sigma^{-2}t - 4\sigma^{-2} \|x_{\rm pop}\|_2^{-2} \sqrt{2t\log\left(\frac{d}{\delta}\right)}.
		\]
  and
  \[
  t\geq \frac{128\|x_{\rm pop}\|_2^{-4}\log\left(\frac{d}{\delta}\right)}{c^2}
  \quad\implies\quad
  {\rm Precision}(X_{1 : t}) \geq \lambda_{\min}\left( \Sigma_1^{-1} \right)\|x_{\rm pop}\|_2^{-2} + \frac{c}{2}\cdot \sigma^{-2}t.
  \]
  
	\end{enumerate}
\end{restatable} 
Beyond the settings considered in this lemma, attainable precision can depend in an interesting way on the context sequence and the prior distribution. Figure \ref{fig:precision} in Section \ref{subsec:numerical_attainable_precision} provides an illustration.

\subsection{Discussion of the main result}\label{subsec:main_discuss}

Theorem \ref{thm:main_result} has several striking implications about the performance of DTS. 
\begin{description}
	\item[A delicate balance between exploration and exploitation.]  The attainable precision in estimating an arm's performance, defined in \eqref{eq:precision_def}, imagines that the potential reward outcomes of that arm were observed \emph{in every period}. An adaptive algorithm can try to emulate this by selecting arms uniformly at random, roughly leading to the same bound on post-experiment performance as in \eqref{eq:post_exp_regret_bound}. But doing so would forego the possibility of having low regret within the experiment as shown for DTS in \eqref{eq:within_exp_regret_bound}. Attaining these two guarantees requires striking a delicate balance between exploring arms to gather all attainable information that is useful, and also aggressively exploiting this information by shifting measurement effort away from bad arms. 
	\item[Robustness to context order.] Theorem \ref{thm:main_result} highlights DTS's robustness when faced with a challenging context order.  For instance, Example \ref{ex:day_of_week_intro}, presented in the next section, studies a weeklong experiment in which the first $T/7$ contexts are Monday, then next $T/7$ are Tuesday, and so on, until the last $T/7$ contexts are Sunday. The bound in \eqref{eq:post_exp_regret_bound} implies that the DTS still gathers adequate information by the end of the experiment. Sections \ref{sec:numerical} and \ref{sec:failure_preview} explain why this context order can create challenges in the design and analysis of bandit algorithms. 
	\item[Robustness to delayed observations.]  Recall that rewards are observed only after some delay of $L\geq 1$ periods.  When $L$ is very large, DTS is not able to get feedback on the decisions during the experimentation phase.  For that reason, the bound in \eqref{eq:within_exp_regret_bound} measures precision offered by the context sequence upto $L$ periods ago. This mild dependence on $L$ provides assurances of robustness. According to our formulation, the decision-maker can wait for all rewards observations to realize before implementing a decision in the population, which is why it is possible for \eqref{eq:post_exp_regret_bound} to have no dependence on $L$. 
	That bound suggests that, even in the face of extreme delay, DTS's arm selections will provide adequate information if one waits for the rewards to realize.  
	\item[Low price of using contexts to deconfound.]  The result highlights the low price of using rich contextual information to deconfound. 
	Unlike contextual bandit results, under which regret generally scales polynomially in the context dimension $d$ \citep{agrawal2013thompson}, our bound has at most a logarithmic dependence on $d$ when the contexts satisfy the conditions of Lemma \ref{lem:bounds_on_precision}.
	This also mimics the bound in Proposition \ref{prop:contextual} in Appendix \ref{sec:contextual_regret}, which is completely independent of the dimension of context vectors. 
	Of course, bounds that are nearly independent of the context dimension offer a stronger guarantee. 
	More importantly, they offer a different conceptual guidance to a practitioner: when using contexts to 'deconfound' inferences, but not to personalize decisions, it is better to use very rich features.
\end{description}

\subsection{Analysis}\label{subsec:analysis_overview}

The analysis leading to Theorem \ref{thm:main_result} may be of independent interest.
We outline ideas underlying the proof of \eqref{eq:within_exp_regret_bound}, which is the more delicate part, 
with \eqref{eq:post_exp_regret_bound} following as a corollary of the analysis. 
A key quantity in the analysis is the posterior standard deviation of population average reward: for any $t\in[T]$ and $i\in[k]$,
\[
s_{t,i} \triangleq \sqrt{{\rm Var}\left( r_{\theta}(i)  \mid H_t\right)}.
\] 
We also define the propensity (also called ``propensity score'') assigned to arm $i$ at time $t$ by
\[
p_{t,i} \triangleq \Prob(I_t=i \mid H_t, X_t).
\]

Learning about the population reward of an arm has limited value if that arm is believed to be very unlikely to be optimal. The term
\begin{equation}\label{eq:regret_to_estimation} 
	\E\left[ s_{t,I^*}^2 \mid H_t, X_{1:T}\right] = \sum_{i=1}^{k} \Prob( I^* = i \mid H_t) \, s_{t,i}^2 = \sum_{i=1}^{k} p_{t,i} \, s_{t,i}^2,
\end{equation}
assesses remaining uncertainty about the performance of arms while giving low weight  to arms that are unlikely to be optimal under the posterior. 
(That $X_{1:T}$  does not appear on the right-hand-side of \eqref{eq:regret_to_estimation} follows from logic similar to Equation \eqref{eq:context_indep_probs}).

The proof highlights two key properties of DTS.

\paragraph{DTS exploits what is known.}
The next result shows DTS has small expected regret in any period if the posterior uncertainty in \eqref{eq:regret_to_estimation} is small. A relatively short proof is given in Appendix \ref{subsec:proof_of_regret_to_estimation}. 
\begin{restatable}[Reduction to estimation]{proposition}{regretFromVariance}\label{prop:regret_to_estimation} 
Under DTS, for any $t\in[T]$, 
	\[
	\E \left[\Delta_{t} \mid H_t, X_{1:T}\right]  \leq  \sqrt{ 2  \log(k)  \mathbb{E}\left[s_{t,I^*}^2  \mid H_t, X_{1:T} \right]}    \quad \text{and}\quad \E \left[\Delta_{t} 
	\mid X_{1:T} \right]  \leq  \sqrt{ 2  \log(k)  \E\left[ s_{t,I^*}^2  \mid X_{1:T}\right]}.
	\]
\end{restatable}

\paragraph{DTS explores the optimal arm.}
The next proposition formalizes that, regardless of the context sequence and delay $L$, DTS is expected to assign high propensity \emph{to the optimal arm} --- in the sense that the expected inverse propensity is uniformly bounded.  Although the proof is very short, we call this a proposition to reflect the critical role it plays in our analysis.

\begin{proposition}\label{prop:prob_score_of_DTS}
	Under DTS, for any $t\in[T]$,
	\[
	\E\left[  \frac{\ind(I^*=i)}{p_{t,i}}  \mid X_{1: T }  \right] =  1, \quad \forall i\in[k] \quad \text{and} \quad \E\left[  \frac{1}{p_{t,I^*}}  \mid X_{1: T }  \right] =  k.
	\]
\end{proposition}
\begin{proof}
	By the tower property, 
	\begin{align*}
	 \E\left[ \frac{ \ind(I^*=i)}{p_{t,i}} \mid X_{1: T } \right] = \E\left[ \E \left[ \frac{ \ind(I^*=i)}{p_{t,i}} \mid H_t,  X_{1: T }  \right] \mid X_{1: T } \right] &= \E\left[   \frac{ \Prob(I^*=i \mid H_t,  X_{1: T })}{p_{t,i}} \mid X_{1: T } \right] \\
	 &=  \E\left[\frac{ \Prob(I^*=i \mid H_t,  X_t )}{p_{t,i}} \mid X_{1: T } \right] =1.
	\end{align*}
	The penultimate equality uses that $X_{1:(t-1)}$ is already contained in the history $H_t$ and that $X_{(t+1):T}$ is independent of $\theta$ conditioned on $H_t$. 
	The last equality uses the definition of DTS in \eqref{eq:probability_matching_def}. We conclude, 
	\[
		\E\left[ \frac{1}{p_{t,I^*}}  \mid X_{1:T} \right]= \E\left[\sum_{i=1}^{k}   \frac{ \ind(I^*=i)}{p_{t,i}} \mid X_{1: T }  \right]=k,
 	\] 
 	where the first equality simply observes that  $\frac{1}{p_{t,I^*}} =   \frac{ \sum_{i=1}^{k}\ind(I^*=i)}{p_{t,I^*}} =\sum_{i=1}^{k}   \frac{ \ind(I^*=i)}{p_{t,i}}$. 
\end{proof}

\paragraph{A delicate balance between exploration and exploitation.}
To get some intuition for these results, let's compare them to what could be attained under alternative algorithms. First, consider an RCT which sets $p_{t,i}=1/k$ for each period $t$ and arm $i$. This algorithm explores aggressively, if naively. Assuming all arms equally likely to be optimal (i.e. $\Prob(I^*=i)=1/k$), then this method would attain the same bounds in Proposition \ref{prop:prob_score_of_DTS}, but it would not attain the low-regret property in Proposition \ref{prop:regret_to_estimation}.  Next, consider a greedy algorithm, which selects the arm $\argmax_{i \in [k]} \E[r_{\theta}(i) \mid H_t]$ in each time period. That algorithm ``exploits what is known'' and attains the bound in Proposition \ref{prop:regret_to_estimation}, but it may neglect to explore the optimal arm and does not satisfy a bound like Proposition \ref{prop:prob_score_of_DTS}. 

That DTS attains both properties reflects a delicate balance it strikes between exploration and exploitation. It is aggressive in shifting measurement effort away from poor arms, leading to Proposition \ref{prop:regret_to_estimation}, but it is still assured to explore all arms which might be optimal, leading to Proposition \ref{prop:prob_score_of_DTS}. 

\paragraph{Completing the proof.}

To complete the proof, we show that sufficient exploration of the optimal arm, in the sense of Proposition \ref{prop:prob_score_of_DTS}, controls the expected posterior variance of the optimal arm (i.e. $\E[s_{t,I^*}^2]$) which appears in Proposition \ref{prop:regret_to_estimation}. The full analysis is quite subtle, but it is possible to give a thorough proof sketch in a special case.
\begin{proof}[Proof sketch in the orthogonal case]
	Consider a special case of our formulation. To avoid writing conditional expectations, assume  $X_{1:T} = x_{1:T} \in \mathcal{X}^T$ with probability 1 for some arbitrary sequence $x_{1:T}$. Now, Assume $\Sigma_1^{-1}= \lambda I$ (representing independent beliefs) and that $\|x_t\|_0 =1$ for each $t$ (so pairs of context vectors are either orthogonal or aligned).  In this special case, the posterior covariance matrix $\Sigma_t\in \mathbb{R}^{dk\times dk}$ is diagonal with entries
	\[ 
	\sigma_{t,i,j}^2 \equiv {\rm Var}\left( \theta^{(i)}_j \mid H_t\right) = \left(\lambda + \sigma^{-2}\sum_{\ell=1}^{t-L} \ind(I_\ell=i)  x_{\ell,j}^2  \right)^{-1}, \quad \forall i\in[k], j\in [d],
	\]
	along the diagonal. Moreover, since $r_{\theta}(i) = x_{\rm pop}^\top \theta^{(i)}$, the posterior variance of population average reward can be written as
	$s_{t,i}^2  =\sum_{j=1}^{d} x_{{\rm pop}, j}^2 \cdot \sigma_{t,i,j}^2$. We then bound this as
 \begin{align*}
		\E\left[s_{t,I^*}^{2}\right] &= \sum_{j=1}^{d} x_{{\rm pop}, j}^2 \cdot \E\left[  \left(\lambda + \sigma^{-2}\sum_{\ell=1}^{t-L} \ind(I_\ell=I^*)  x_{\ell,j}^2  \right)^{-1}   \right] \\
		&\overset{(a)}{\leq} \iota  \sum_{j=1}^{d} x_{{\rm pop}, j}^2 \cdot \E\left[  \left(\lambda + \sigma^{-2}\sum_{\ell=1}^{t-L} p_{\ell, I^*}  x_{\ell,j}^2  \right)^{-1}   \right] \\
		&\overset{(b)}{\leq} \iota \sum_{j=1}^{d} x_{{\rm pop}, j}^2 \cdot \left( \lambda + \sigma^{-2} \sum_{\ell=1}^{t-L} x_{\ell,j}^2 \right)^{-2} \cdot \E\left[ \lambda +  
  \sigma^{-2}\sum_{\ell=1}^{t-L}\frac{x^2_{\ell,j}}{p_{\ell,I^*}}
  \right] \\  
		&\overset{(c)}{=} \iota \sum_{j=1}^{d} x_{{\rm pop}, j}^2 \cdot \left( \lambda + \sigma^{-2} \sum_{\ell=1}^{t-L} x_{\ell,j}^2 \right)^{-2} \cdot \left( \lambda + k\cdot \sigma^{-2} \sum_{\ell=1}^{t-L} x_{\ell,j}^2  \right) \\ 
		&\leq  \iota \sum_{j=1}^{d} x_{{\rm pop}, j}^2 \cdot \left( \lambda + \sigma^{-2} \sum_{\ell=1}^{t-L} x_{\ell,j}^2 \right)^{-1} \cdot k\\
		&= \frac{\iota \cdot k}{{\rm Precision}\left(x_{1:(t-L)}\right)}. 
	\end{align*}
  Inequality $(c)$ applies Proposition \ref{prop:prob_score_of_DTS}. Inequality $(b)$ uses Jensen's inequality. 
	 
		Inequality $(a)$ requires a detailed proof, but we can provide semi-rigorous intuition. To study both sides of the inequality $(a)$, fix any arm $i$ and define $D_{\ell, j} = [\ind(I_\ell =i) - p_{\ell,j}] x_{\ell, j}^2$. This has zero conditional mean (i.e. $\E[D_{\ell,j} | D_{1,j},\ldots, D_{\ell-1,j}] =0$) and conditional variance $v_{\ell,j} =   \E[D_{\ell,j}^2 | D_{1,j},\ldots, D_{\ell-1,j}] = p_{\ell,j}(1-p_{\ell,j}) x_{\ell, j}^4 \leq p_{\ell,j} x_{\ell, j}^2$ (by the assumption that $\|x_\ell\|_2\leq 1$.)
  Then, 
	\begin{align*}
		\sum_{\ell=1}^{t-L} \ind(I_\ell=I^*) x_{\ell,j}^2 = \sum_{\ell=1}^{t-L} p_{\ell,j}  x_{\ell,j}^2 + \sum_{\ell=1}^{t-L} D_{\ell,j} 
  &\approx \sum_{\ell=1}^{t-L} p_{\ell,i}x_{\ell,j}^2  + O\left( \sqrt{\sum_{\ell=1}^{t-L} v_{\ell,j}} \right) \\
  &=\sum_{\ell=1}^{t-L} p_{\ell,j} x_{\ell,j}^2 + O\left( \sqrt{\sum_{\ell=1}^{t-L}  p_{\ell, j} x_{\ell,j}^2 }\right).
	\end{align*}
	The approximate equality (marked $\approx$) can be loosely justified through the martingale central limit theorem. The rigorous proof, given in Appendix \ref{subsec:smoothed_model}, instead relies on a non-asymptotic martingale concentration inequalities. 
\end{proof}

This proof technique generalizes to problems with non-orthogonal context vectors, but it requires careful matrix-valued generalizations of all key inequalities. A generalization of inequality $(a)$ is given in Lemma \ref{lem:matrix_skipping_corollary}, in Appendix \ref{subsec:smoothed_model}. In proving this, we developed a new concentration inequality for matrix-valued martingales (i.e. Proposition \ref{prop:matrix_concentration}), which may be of independent interest. 
In Appendix \ref{subsec:bounding_precision}, Lemma \ref{lem:IPW-posterior-bound} presents a matrix-valued generalization of inequality $(b)$. Its proof relies on a remarkable generalization of Jensen's inequality to operator convex functions, which we restate as Lemma \ref{lem:operator_jensen}.

\section{Numerical illustration}\label{sec:numerical}
We provide numerical experiments that motivate our theory and help the reader build intuition. 
Specifically, these illustrations provide a glimpse of the challenges outlined in Section \ref{sec:failure_preview} and of DTS's intricate balance of exploration and exploitation, which we formalized in Theorem \ref{thm:main_result}. While we compare DTS with alternative algorithms, our intent is not to conduct extensive competitive benchmarking.
 
\subsection{An example with day of week effects}
Our simulations center around Example \ref{ex:day_of_week_intro}, which demonstrates the challenges faced when the context sequence exhibits nonstationary pattern. The example models a week-long experiment where observations are influenced by day-of-week effects, a routine concern in A/B testing \citep{kohavi2020trustworthy}.
\begin{example}[Day-of-week effects]\label{ex:day_of_week_intro}
	Consider an online retailer conducting a weeklong experiment to find the price that maximizes profit from selling a product in subsequent weeks.
    Demand is assumed to follow a normal distribution, implying that profit also follows a normal distribution.
    Demand varies according to the day of the week.  This scenario can be mapped to a special case of the model in Section \ref{sec:formulation}, where each context $X_t\in\{e_1, \ldots, e_7\}\subset \mathbb{R}^7$ is one of the standard basis vectors. Suppose $T=7m$ and the context at time $t$ is $X_t = e_{\lceil  t/m \rceil}$, signifying that first $m$ periods are Sunday, the next $m$ are Monday, and so on, with the final $m$ being Saturday. 
	The price $I_t$ is adjusted in each period and offered to the next customer (a time period could also represent a small batch of customers), generating reward $R_{t,I_t} =  \langle \theta^{(I_t)},X_t\rangle + W_{t,I_t}$ representing the profit earned. There is no delay in observing rewards (i.e. $L=1$).
	Let the population distribution $\Dc_{\rm pop}$ be uniform over $\{e_1, \ldots, e_7\}$. The performance of arm $i$ on day $x$ is the $x$-th component of the vector $\theta^{(i)}$, i.e., $\theta^{(i)}_x = \langle \theta^{(i)}, e_x\rangle$.
	At the end of the experiment, the decision-maker picks a single price $I_{\rm post}$ to employ across future weeks. The loss incurred due to the decision made under incomplete resolution of uncertainty about average demand is measured by
	\begin{equation}\label{eq:regret_day_of_week}
		\Delta_{\rm post}=\max_{i\in[k]}\left(\frac{\theta^{(i)}_{1}+\cdots+ \theta^{(i)}_{7}}{7}\right) -\left(\frac{\theta^{(I_{\rm post})}_{1}+\cdots+ \theta^{(I_{\rm post})}_{7}}{7}\right). 
	\end{equation}
	The reasons for learning a single price, pertaining to fairness and incentive-compatibility, are discussed in Appendix \ref{sec:standardization}.
		
	The decision-maker begins with prior belief  that  $\theta \sim N(\mu, \Sigma)$. We consider a structured prior induced from a latent variable model where $\theta^{(i)}_{x} = \theta^{\rm idio}_{i,x}+\theta^{\rm arm}_{i}+\theta^{\rm day}_{x}$ is determined by an  effect $\theta^{\rm idio}_{i,x}$ that is idiosyncratic to a specific arm and day, an effect $\theta^{\rm arm}_{i}$ associated with the chosen arm, and a shared day-of week effect $\theta^{\rm day}_{x}$. Placing an independent normal prior on the idiosyncratic, arm-specific, and day-specific effects induces a structured covariance matrix $\Sigma$. When the idiosyncratic terms have large variance, the decision-maker must be cautious of almost arbitrary nonstationary patterns. If these are believed to have smaller magnitude, the decision-maker may be able to rule out some very poor arms early in the experiment. 
\end{example}

\subsection{Attainable precision and delayed learning due to context order}\label{subsec:numerical_attainable_precision}
 
Figure \ref{fig:precision} plots attainable precision in a special case of Example \ref{ex:day_of_week_intro}. Recall this is defined as 
\begin{align*}
	{\rm Precision}(X_{1 : t}) &\triangleq \min_{i\in [k]} \,  \frac{1}{{\rm Var}\left(  r_{\theta}( i) \mid   (X_1, R_{1,i}, \ldots, X_t, R_{t, i}) \right)}  \\
	 \frac{1}{{\rm Precision}(X_{1 : t})} &= \max_{i\in [k]}  {\rm Var}\left(  r_{\theta}( i) \mid   (X_1, R_{1,i}, \ldots, X_t, R_{t, i}) \right)
\end{align*} 
and assesses the remaining uncertainty a decision-maker would have about an arm's population-level performance \emph{assuming they chose to measure that arm exclusively}.
Arms are a priori symmetric in our example, so the minimum and maximum above are redundant. We plot this in two cases. 
\begin{description}
	\item[Sequential context order:] Context 1 occurs for the first $100$ periods ('Monday'), context 2 occurs for the next $100$ periods ('Tuesday'), and so on.
	\item[Shuffled context order:] Contexts are drawn i.i.d.~across periods with uniform probabilities.
\end{description}

\begin{figure}
	\centering
	\includegraphics[width=1\linewidth]{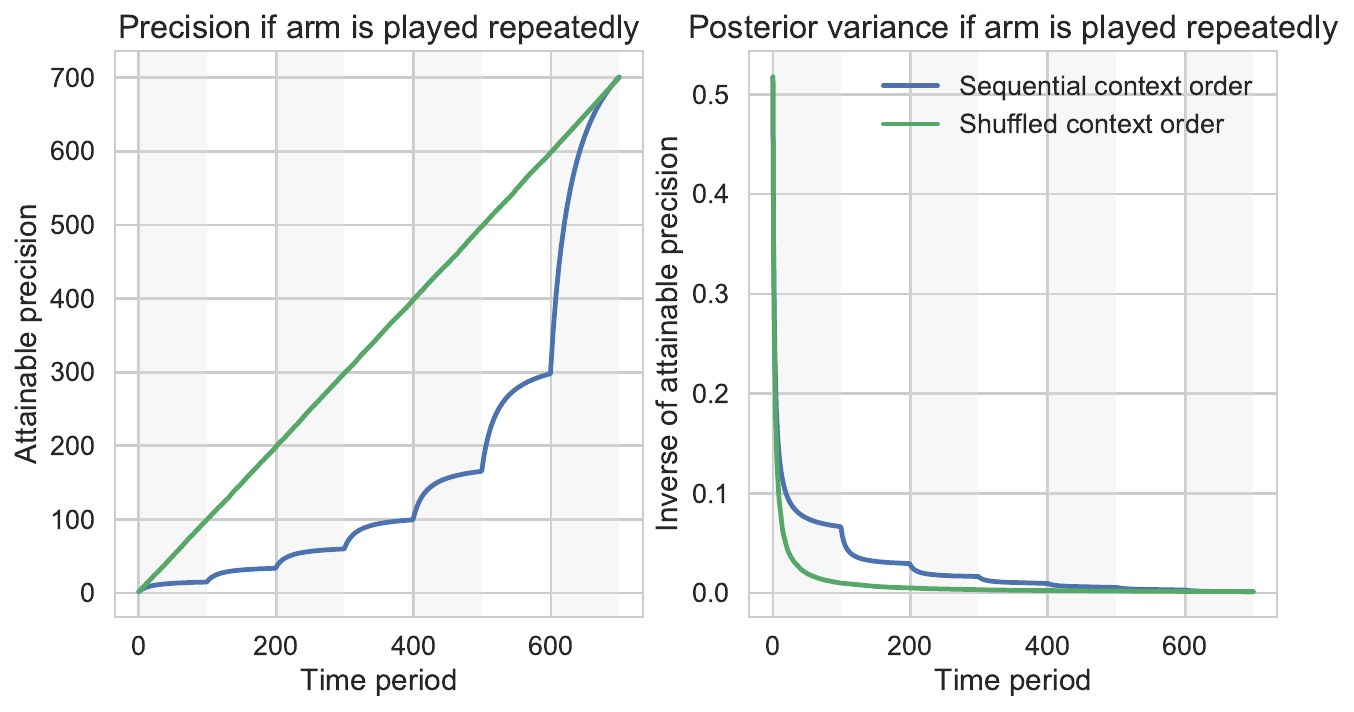}
	\caption{Attainable precision over time in Example \ref{ex:day_of_week_intro} with $m=100$ periods per context and noise variance $\sigma^2=1$. The prior variances are such that ${\rm Var}(\theta^{\rm arm}_i)=1$, ${\rm Var}(\theta^{\rm day}_x)=0$ and ${\rm Var}(\theta^{\rm idio}_{i,x})= \frac{1}{14}$ for all $i\in [k]$ and $x\in [7]$.}
	\label{fig:precision}
\end{figure}

In both cases, attainable precision at the end of the experiment is ${\rm Precision}(X_{1:T}) = \sigma^{-2} \cdot T = 700$. 
If we interpret this as a `large' value, then the bound in equation \eqref{eq:post_exp_regret_bound} of Theorem \ref{thm:main_result} suggests that 
DTS will attain low post-experiment regret in either case. 
However, the evolution of attainable precision within the experiment looks very different depending on the context order.

When contexts are shuffled, precision displays linear growth with ${\rm Precision}(X_{1:t})\approx \sigma^{-2} \cdot t$, mirroring the bounds in Subsection \ref{subsec:precision_growth}. 
The posterior variance, which is the inverse of precision, undergoes a rapid decrease following the onset of the experiment.
This indicates that if the DM chose to explore an arm  $i$ aggressively at the beginning of the experiment, they could resolve uncertainty about its population-level performance $r_{\theta}(i)$. 

The behavior of attainable precision changes substantially under a sequential context order.  
It grows slowly at the beginning of the experiment, reflecting that
resolving uncertainty about an arm's population level performance requires waiting for certain contexts to become observable. 
In fact, the figure displays fairly sharp jumps in attainable precision when new contexts become observable --- marked in Figure \ref{fig:precision} by alternating grey and white shaded columns.

\subsection{Algorithms compared}

\subsubsection{Methods for selecting arms within an experiment}
Our numerical experiments compare the following procedures for selecting arms $I_1, \ldots ,I_T$ within-the-experiment.

\begin{description}
	\item[Deconfounded Thompson sampling:] Implements Algorithm \ref{algo:DTS}. 
	\item[Deconfounded UCB:] The UCB analogue of deconfounded TS. This algorithm defines an upper confidence bound $U_{t,i} = \E[r_{\theta}(i) \mid H_{t}] + z \sqrt{{\rm Var}(r_{\theta}(i) \mid H_{t} )}$ on the population-average reward $r_{\theta}(i)$ of arm $i$, then it selects the arm $I_t = \argmax_{i \in [k]} U_{t,i}$. In this study, we use $z =3$, but alternative choices produce  qualitatively similar outcomes.
	\item[Context-unaware Thompson sampling:]  A version of TS that acts as if rewards were i.i.d.~and there were no contextual observations. It imagines each sample of arm $i$ is a draw with mean $\E[R_{t,i} \mid \theta]=r_{\theta}(i)$ and noise variance ${\rm Var}( R_{t,i} \mid \theta)= \sigma^{2} + \max_{x} {\rm Var}(\theta_x^{\rm day})$; the noise variance is inflated since the algorithm is not accounting for variance driven by context.  
	\item[Round-robin sampling:] The algorithm samples arm 1 when $t=1$, arm 2 when $t=2$, $\ldots$, arm $k$ when $t=k$, and then starts the cycle again, sampling arm $1$ when $t=k+1$ and so on.   	
	\item[Sequential elimination:] The algorithm maintains a set of contending arms, which contains all arms at initialization. At the start of any period, an arm whose posterior probability of being optimal, $\Prob(I^* =i  \mid H_t)$, falls below some threshold $\delta/k$ is removed from the set of contending arms. We set $\delta=0.05$, reflecting a goal of having less than a $5\%$ chance of eliminating the best arm. A suitable variant of round-robin sampling is used to select an arm to sample in each period from the arms still in contention.
\end{description}

\subsubsection{Methods for selecting an arm to deploy post-experiment}
Every procedure we evaluate selects an arm post-experiment in a Bayes optimal manner. 
\begin{description}
	\item[Minimizing regret:] Set $I_{\rm post} \in \argmax_{i \in [k]} \E\left[ r_{\theta}(i) \mid H_{\rm post}\right]$ to be the Bayes optimal arm for a decision-maker who wishes to maximize population-level reward. To visualize decision-quality if the experiment we stopped early, we set $\hat{I}_t \in  \argmax_{i \in [k]} \E[ r_{\theta}(i) \mid H_{t} ]$ and evaluate the regret $\E[r_{\theta}(I^*)-r_{\theta}(\hat{I}_t) ]$. Figure~\ref{fig:main_experiment_results} presents this as ``future regret if experiment were stopped.''
	\item[Minimizing the probability of incorrect selection:] Set $I_{\rm post} \in \argmax_{i \in [k]} \Prob\left( I^*=i  \mid H_{\rm post}\right)$ the Bayes optimal arm for a decision-maker who wishes to maximize the probability of correct selection. To visualize decision-quality if the experiment we stopped early, we set $\hat{I}_t \in \argmax_{i \in [k]} \Prob( I^*=i \mid H_{t} )$ and evaluate the probability of correct selection $\Prob(\hat{I}_t=I^*)$. Figure~\ref{fig:main_experiment_results} presents this as ``confidence in identity of the best arm.''
\end{description}
Because these rules are Bayes optimal, an algorithm which suffers high post-experiment regret, or attains low probability of correct selection, does so because of  inadequate information gathering within the experiment; it is not possible to improve performance by changing how decisions are made post-experiment\footnote{Consider any other rule $\tilde{\pi}_{\rm post}$ that selects an arm $\tilde{I}_{\rm post}=f(H_{\rm post})$.  Then
	\[
	\E\left[ r_{\theta}(I_{\rm post})  \right] = \E\left[ \E\left[ r_{\theta}(I_{\rm post})  \mid H_{\rm post} \right] \right]  = \E\left[ \E\left[ \max_{i \in [k]}  r_{\theta}(i)  \mid H_{\rm post} \right] \right] \geq  \E\left[ \E\left[ \max_{i \in [k]}  r_{\theta}(  \tilde{I}_{\rm post})  \mid H_{\rm post} \right] \right] = \E\left[ r_{\theta}(\tilde{I}_{\rm post})  \right].
	\] 
	A procedure that selects the arm with highest posterior mean at the end of the experiment yields greater expected reward post-experiment than any alternative, regardless of which procedure (e.g. DTS or deconfounded UCB) is used to sample arms during the experiment.  
}.

\subsection{Discussion of experiment results} 
 We simulate algorithms applied to Example \ref{ex:day_of_week_intro}. Our simulations use noise variance $\sigma^2=1$.  The latent variables $\theta^{\rm arm}_{i}$ and $\theta^{\rm day}_{x}$, $\theta^{\rm idio}_{i,x}$ have mean zero and prior standard deviation 0.5, 1.0, and 0.8 respectively. 
 We make a number of observations: 
 \begin{description}
 	\item[Results with shuffled context order.] With shuffled context order, all algorithms succeed in confidently identifying the best arm and have low post-experiment regret. Bandit algorithms like TS and UCB shift sampling effort away from clearly bad actions within the experiment and this reduces the regret they incur. 	Context unaware TS succeeds when contexts are shuffled by treating (unmodeled) contexts as if they were i.i.d.~observation noise.  Even when contexts are i.i.d., this is not statistically efficient since `controlling for' observed contexts would reduce variance. This is reflected in the fact that the regret of context unaware TS is larger than that of DTS in Figure \ref{fig:shuffled} (though, this is not a huge issue for our particular experiment parameters).
 	
 	\item[Delayed learning due to context order.] For concreteness, let's focus on round-robin sampling. In the experiment with shuffled context order, round-robin sampling quickly found a near optimal arm to deploy in the population. Due to low reward noise (i.e. small $\sigma^2$), uncertainty resolves rapidly. 
 	 With sequential context order, despite low reward noise, uncertainty about an arms' performance on Sunday only resolves at the end of the experiment.
 	 Hence, uncertainty about an arm's average performance throughout the week only resolves at the end of the experiment.
 	 The top-left of Figure \ref{fig:sequential} shows that uncertainty about the identity of the optimal arm resolves in sharp jumps at the start of each day --- a behavior that is quite different from what is depicted in Figure \ref{fig:shuffled}. At least qualitatively, this finding parallels the behavior of attainable precision in Figure \ref{fig:precision}. 
 	
 	\item[Robustness to sequential context order.] 
 	 DTS, round-robin sampling, and sequential elimination demonstrate robustness to sequential context order, while deconfounded UCB and context unaware TS appear brittle. Notably, DTS, round-robin sampling, and sequential elimination suffer tiny post-experiment regret once all contexts have been observed. In contrast, even after all days of the week have been observed, context-unaware TS and deconfounded UCB cannot identify an optimal arm to deploy post-experiment. Since all algorithms were evaluated assuming that correct posterior inferences were used for post-experiment arm selection, the failure of these algorithms indicates an inadequacy in the information they gather.
 	  	 
 	 The performance differences between deconfounded TS and a deconfounded (Bayesian) UCB in Figure \ref{fig:sequential} are quite striking, given that the literature has often emphasized the similarities between these algorithms. 
 	 A closer look at the experiment results reveals that deconfounded UCB often plays only a single arm on certain days of the week, completely failing to gather information about some arms on some days of the week. 
 	 See the next section for further discussion. 
 	 
 	\item[Aggresive exploitation.]  DTS incurs lower regret within the experiment than both round-robin sampling and sequential elimination. This is attributable to its aggressive approach in shifting effort away from arms that have a low posterior probability of being optimal given current evidence. 
 	 By comparison, sequential elimination incurs greater regret within the experiment as it cannot respond to weak initial evidence of an arm's poor performance; sequential elimination treats all arms equally unless it is highly confident that a particular arm can be ruled out.

 	 \item[Contextual regret.]  In addition to our main regret measure, we compare algorithms in terms of what we term their cumulative ``within-experiment contextual regret'': $\E\left[ \sum_{\ell=1}^{t} (R_{\ell,I^*} - R_{\ell,I_\ell} )\right]$. DTS seems to perform well according to this metric as well. Appendix \ref{sec:contextual_regret} confirms that this is always true by bounding the cumulative contextual regret of DTS. 	 
 	  One should not focus on the fact that deconfounded UCB  attains negative contextual regret in this particular experiment. This is not a general phenomenon, and it is possible to construct examples, along the lines of Example \ref{example:day-of-week-no-obs-noise}, in which it incurs large contextual regret.  	 
 \end{description}

\begin{figure}
	\centering
	\begin{subfigure}{.45\textwidth}
		\centering
		\includegraphics[width=1\linewidth]{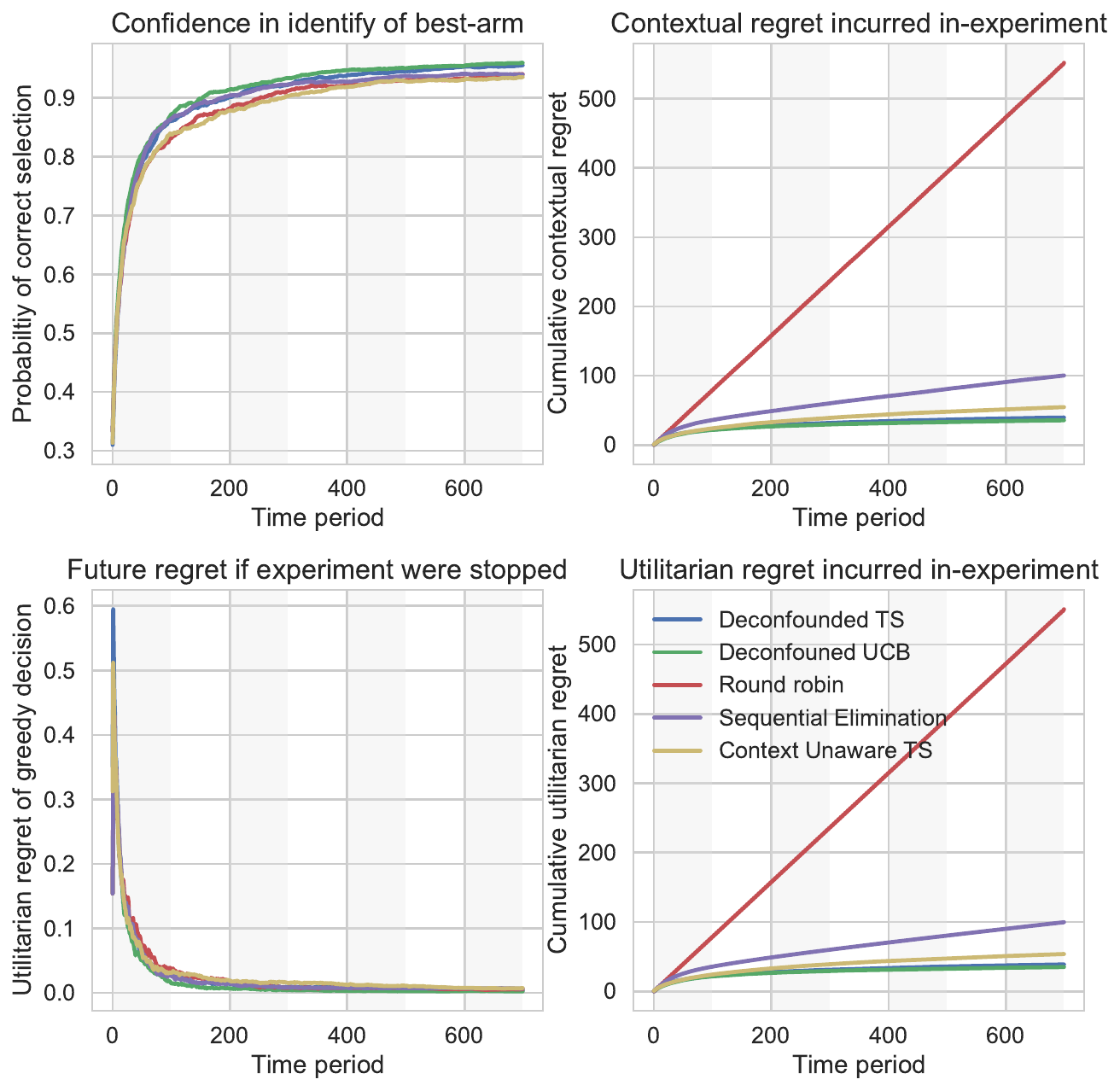}
		\caption{ Shuffled context order}
		\label{fig:shuffled}
	\end{subfigure}
	\begin{subfigure}{.45\textwidth}
		\centering
		\includegraphics[width=1\linewidth]{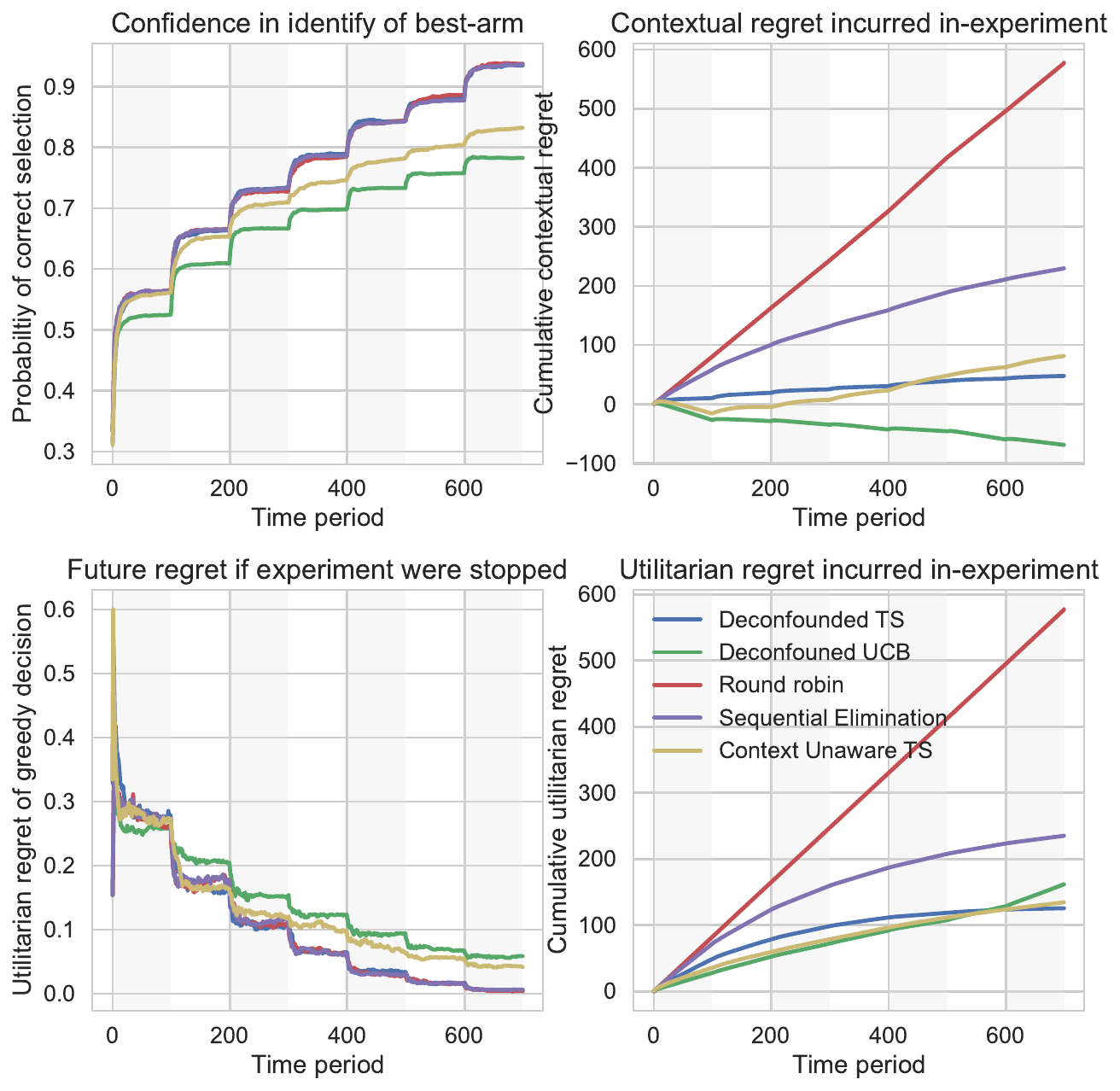}
		\caption{ Sequential context order}
		\label{fig:sequential}
	\end{subfigure}
	\caption{Algorithm performance in Example \ref{ex:day_of_week_intro} with $m=100$ periods per context and noise variance $\sigma^2=1$.  The latent variables $\theta^{\rm arm}_{i}$ and $\theta^{\rm day}_{x}$, $\theta^{\rm idio}_{i,x}$ have mean zero and prior standard deviation 0.5, 1.0, and 0.8 respectively}
	\label{fig:main_experiment_results}
\end{figure}

\section{Challenges of our model: the unexpected failure of deconfounded UCB}\label{sec:failure_preview}

As expected, our numerical experiments show that context-unaware algorithms can falter. Controlling for exogenous variation is critical to drawing accurate inferences about arms' performance.  

The numerical experiments, however, indicate that our model hosts additional surprises. 
While controlling for sources of exogenous variation is crucial, it can introduce unavoidable delays in the resolution of uncertainty as the DM anticipates relevant contexts that have yet to occur.
To illustrate this point, we present a simplified variant of Example \ref{ex:day_of_week_intro}.
\begin{example}[Simplified day-of-week effects]\label{example:day-of-week-no-obs-noise} 
	Consider a two-day experiment with $k=2$ arms and context set $\Xc = \{ e_1, e_2\}\subset\mathbb{R}^2$. The context sequence is deterministic, with $X_{t} = e_1$ for $t\leq \lfloor \frac{T}{2} \rfloor$, $X_t = e_2$ for $t> \lfloor \frac{T}{2} \rfloor$. The goal is to identify the best arm under equal context weights $x_{\rm pop}=(0.5, 0.5)$.  The components of vector $\theta=(\theta^{(i)}_{x})_{i\in [2], x\in [2]}$ are independent with $\theta^{(i)}_x = \langle \theta^{(i)}, e_x\rangle$ being the performance of arm $i$ on day $x$.  The reward at time $t$ is $R_{t,I_t} =  \langle \theta^{(I_t)},X_t\rangle$  (i.e. $\sigma=0$ so there is no reward noise\footnote{Technically, we assumed $\sigma>0$ at the end of the problem fsormulation,  writing that this allowed us to write expressions like $1/\sigma^2$, which appear often in the analysis. One could take $\sigma$ to be extremely close to 0 in this example, but the presentation is much cleaner if it equals zero exactly.}).
	Reward observations are not subject to delay (i.e. $L=1$). 
\end{example}
It is straightforward to design a learning procedure for this example. With no observation noise, the DM merely needs to play both arms once in each of the two contexts.
However, unlike in an i.i.d.~bandit model, the DM cannot opt for aggressive exploration to rapidly resolve uncertainty.
Understanding an arm's population-level performance requires waiting until the second half of the experiment when the second context becomes observable.
Before that, the DM remains uncertain.

Algorithms are differentiated by how they explore when faced with this uncertainty about population-level performance that they cannot rapidly resolve. 
The following lemma shows that deconfounded UCB continues to sample one arm repeatedly during the first half of the experiment. Because of this failure of information gathering, it can't evaluate one arm's population-level performance \emph{even at the end of the experiment}.
\begin{restatable}[Failure of deconfounded UCB]{lemma}{contextAwareUCB}\label{lem:deconfounded_ucb}
	Consider Example \ref{example:day-of-week-no-obs-noise}. Suppose that
	 $\theta^{(1)}_{ x}\sim N(0,1)$ and $\theta^{(2)}_{x}\sim N(0,3)$ for $x\in [2]$.
	If, for any fixed $z>0$,  $I_t \in \argmax_{i \in [k]} \E[r_{\theta}(i) \mid H_t]  +  z\sqrt{ {\rm Var}(r_{\theta}(i) \mid H_t)}$ holds for every $t$, there is an absolute numerical constant $c>0$ such that for all $T\in \mathbb{N}$,  $\E \left[\Delta_{\rm post }\right] \geq c$. 
\end{restatable}
\begin{proof}[Proof sketch]
	During the first half the experiment, when the context is $e_1$, deconfounded UCB plays only action 1. The UCB for action 1 exceeds that of action 2, and this UCB stays very large until after the second context is observed. Since the reward of arm 2 in context $e_1$ is  never observed, the DM may fail to deploy an optimal arm in the population.  A complete proof is provided in Appendix \ref{sec:failure}.
\end{proof}
Intuitively, it seems that DTS might avoid this information-gathering failure. During the first half of the experiment, DTS would continue (randomly) sampling both arms, only shifting measurement effort away from an under-performing arm once its posterior probability of being optimal is low and further information gathering is not useful. Our theory confirms this intuition, demonstrating that, despite its aggressive exploration, DTS gathers enough information to ensure low post-experiment regret across a broad class of problems.

It is likely possible to modify deconfounded UCB so that it performs well in this  straightforward example.\footnote{One can define an algorithm that plays arms randomly with probabilities that depend on upper confidence bounds. One can also force the algorithm to continue sampling all arms with high UCBs, eliminating arms once it is clear that they underperform. These make the decision-making logic similar to Thompson sampling or sequential elimination, respectively.} We leave this to future work, and instead focus on showing that DTS explores efficiently without any such modifications.

\section{Conclusion}\label{sec:conclusion}
\subsection{Closing thoughts}
This paper proposes a new way to model adaptive experiments conducted in the presence of nonstationary variation. Out of this model comes a more robust variant of the prominent Thompson sampling algorithm. We provide several theoretical results that provide assurances of its robustness.  At a casual glance, one might expect developing this theory to require a routine -- if intricate -- exercise in adapting widely used arguments in the literature. Perhaps surprisingly, this problem class raises many new subtleties, as is reflected in the failure of deconfounded UCB, the departure of learning dynamics in Section \ref{sec:numerical} from those in i.i.d.~bandit problems, and the original theorem statement and proof in Section \ref{sec:main_result}.

Our model is quite flexible. Special cases of it, like Example \ref{ex:latent_confounders} in the introduction or Example \ref{ex:time_zones} in the appendix, differ significantly. The extensions covered below provide even more flexibility. On the positive side, this flexibility expands the scope of problems to which DTS and our theory can be applied. Unfortunately, it also leaves a practitioner with many subtle modeling choices. A nice complement to this paper would be one focuses on a narrow real-world use case and carefully documents many of the modeling choices involved.

\subsection{Extensions}\label{subsec:policy_learning}
We close by mentioning two extensions that broaden the applicability of DTS. 

\paragraph{Policy learning.} 
Thompson sampling can be readily applied to contextual bandit problems where the goal is to learn an optimal policy that segments or personalized its decisions on the basis of observed contexts. 
In proposing DTS, we have shown how to adapt Thompson sampling so as to control for exogenous sources of variation while learning a stable decision-rule: one which does not react to evolving context. 
Appendix \ref{sec:policy_learning} provides a full discussion of and motivation for this difference.  In that section, we also extend DTS to learn policies that are reactive to some parts of the context but not others. We explain how to provide a more conventional regret bound for that algorithm, but are not certain how to extend the proof of Theorem \ref{thm:main_result} to treat this generalization.

\paragraph{Top-two sampling and a prioritization of within-experiment regret.}

We have evaluated DTS in terms of two broad performance criteria: the regret incurred (or reward accrued) during the experiment and the regret incurred (or reward accrued) post-experiment. For those who wish to prioritize attaining very low post-experiment experiment regret, it may be helpful to consider Top-two sampling \citep{russo2020simple} variants of DTS that explore more aggressively. Top-two DTS can be defined succinctly. At each time period $t\in \mathbb{N}$, it selects an arm to measure through the following procedure:
\begin{center}
	\emph{Continue sampling from the probability mass function $\Prob(I^* = \cdot \mid H_t)$ until two distinct arms are chosen.\\ Flip a (biased) coin to select one among these two.}
\end{center}
This procedure bootstraps standard randomized arm selection by DTS, defining a new way of sampling arms by running it as a subroutine. We denote the first arm sampled by top-two DTS by $\hat{I}_t$ and call this the ``leader''. Denote the second arm sampled by $\hat{J}_t$ and call this the challenger. The overall sampling probabilities obey the formula
\[ 
\Prob(I^*=i \mid H_t) = \beta\underbrace{\Prob(I^* = i \mid H_t)}_{ \text{prob. leader is } i } + (1-\beta)\underbrace{\sum_{j\neq i} \Prob(I^* = j) \Prob(I^*=i \mid I^* \neq j, H_t)}_{\text{prob. challenger is } i }. 
\]
To understand the intuition behind this modification, consider a scenario in which the DM is 95\% confident that in the identify of the optimal arm; For instance, $\Prob(I^*=1 \mid H_t) = 0.95$. In such scenarios, standard DTS plays arm 1 95\% of the time, rarely gathering information about other arms. The top-two modification encourages the algorithm to more aggressively explore the most promising challengers to arm 1. This change can reduce the length of experiment (i.e. $T$ in our formulation) required to reach very high confidence. 

A burgeoning body of theory establishes senses in which this kind of procedure is asymptotically optimal  \citep{russo2020simple, qin2017improving, shang2020fixed, jourdan2022top}.  
Most of that theory involves problems without contexts, but a a companion to this paper studies asymptotic efficiency of top-two DTS in problems with contextual variation.

\paragraph{Beyond Gaussian noise.}

Our results require a Gaussian prior and noise. 
This case is especially tractable analytically, allowing for an especially efficient implementation of DTS that avoids the need for approximate posterior sampling. 
However, we conjecture that an analogue of our theoretical results should hold more generally. 
An analogue of Proposition \ref{prop:contextual}, in the appendix, holds when reward noise is sub-Gaussian and the norm of $\theta$ is bounded almost surely.
But the proof of Theorem \ref{thm:main_result} relies on the analytical form of the Guassian posterior and we do not know how to generalize it.

\paragraph{Choosing a prior.}
  The choice of a bandwidth parameter in the prior displayed in Figure 1, for instance, is a delicate choice. Yet, most choices are likely to offer more robustness than applying vanilla Thompson sampling, an extreme special case of that prior under which is there no nonstationarity in rewards.

One possibility is to set prior parameters using data from past experiments. An online retailer who regularly conducts pricing experiments can use data from these past experiments to calibrate hyper-parameters governing the structure and severity of plausible nonstationarity. For more insights into this 'empirical Bayesian' perspective, refer to \cite{azevedo2019empirical}, \cite{dimmery2019shrinkage}, \cite{bastani2022meta}, and \cite{mcdonald2023impatient}.

{\footnotesize 
	\setlength{\bibsep}{2pt plus 0.4ex}
	\bibliographystyle{abbrvnat}
	\bibliography{references}

\begin{thebibliography}{62}
\providecommand{\natexlab}[1]{#1}
\providecommand{\url}[1]{\texttt{#1}}
\expandafter\ifx\csname urlstyle\endcsname\relax
  \providecommand{\doi}[1]{doi: #1}\else
  \providecommand{\doi}{doi: \begingroup \urlstyle{rm}\Url}\fi

\bibitem[Abbasi-Yadkori et~al.(2018)Abbasi-Yadkori, Bartlett, Gabillon, Malek,
  and Valko]{abbasi2018best}
Y.~Abbasi-Yadkori, P.~Bartlett, V.~Gabillon, A.~Malek, and M.~Valko.
\newblock Best of both worlds: Stochastic \& adversarial best-arm
  identification.
\newblock In \emph{Conference on Learning Theory}, pages 918--949. PMLR, 2018.

\bibitem[Abbasi-Yadkori et~al.(2022)Abbasi-Yadkori, Gyorgy, and
  Lazic]{yadkori22}
Y.~Abbasi-Yadkori, A.~Gyorgy, and N.~Lazic.
\newblock A new look at dynamic regret for non-stationary stochastic bandits.
\newblock \emph{arXiv preprint arXiv:2201.06532}, 2022.

\bibitem[Adusumilli(2023)]{adusumilli2023risk}
K.~Adusumilli.
\newblock Risk and optimal policies in bandit experiments, 2023.

\bibitem[Agrawal and Goyal(2013)]{agrawal2013thompson}
S.~Agrawal and N.~Goyal.
\newblock Thompson sampling for contextual bandits with linear payoffs.
\newblock In \emph{International Conference on Machine Learning}, pages
  127--135, 2013.

\bibitem[Amadio(2020)]{amadio2020bandits}
B.~Amadio.
\newblock Multi-armed bandits and the stitch fix experimentation platform,
  2020.
\newblock URL
  \url{https://multithreaded.stitchfix.com/blog/2020/08/05/bandits/}.
\newblock Accessed: May 31, 2023.

\bibitem[Araman and Caldentey(2022)]{araman2022diffusion}
V.~F. Araman and R.~A. Caldentey.
\newblock Diffusion approximations for a class of sequential experimentation
  problems.
\newblock \emph{Management Science}, 68\penalty0 (8):\penalty0 5958--5979,
  2022.

\bibitem[Athey and Wager(2021)]{athey2021policy}
S.~Athey and S.~Wager.
\newblock Policy learning with observational data.
\newblock \emph{Econometrica}, 89\penalty0 (1):\penalty0 133--161, 2021.

\bibitem[Athey et~al.(2022)Athey, Byambadalai, Hadad, Krishnamurthy, Leung, and
  Williams]{athey2022contextual}
S.~Athey, U.~Byambadalai, V.~Hadad, S.~K. Krishnamurthy, W.~Leung, and J.~J.
  Williams.
\newblock Contextual bandits in a survey experiment on charitable giving:
  Within-experiment outcomes versus policy learning, 2022.

\bibitem[Auer et~al.(2002{\natexlab{a}})Auer, Cesa-Bianchi, and
  Fischer]{auer2002finite}
P.~Auer, N.~Cesa-Bianchi, and P.~Fischer.
\newblock Finite-time analysis of the multiarmed bandit problem.
\newblock \emph{Machine learning}, 47:\penalty0 235--256, 2002{\natexlab{a}}.

\bibitem[Auer et~al.(2002{\natexlab{b}})Auer, Cesa-Bianchi, Freund, and
  Schapire]{auer2002nonstochastic}
P.~Auer, N.~Cesa-Bianchi, Y.~Freund, and R.~E. Schapire.
\newblock The nonstochastic multiarmed bandit problem.
\newblock \emph{SIAM journal on computing}, 32\penalty0 (1):\penalty0 48--77,
  2002{\natexlab{b}}.

\bibitem[Auer et~al.(2008)Auer, Jaksch, and Ortner]{auer2008near}
P.~Auer, T.~Jaksch, and R.~Ortner.
\newblock Near-optimal regret bounds for reinforcement learning.
\newblock \emph{Advances in neural information processing systems}, 21, 2008.

\bibitem[Azevedo et~al.(2019)Azevedo, Deng, Montiel~Olea, and
  Weyl]{azevedo2019empirical}
E.~M. Azevedo, A.~Deng, J.~L. Montiel~Olea, and E.~G. Weyl.
\newblock Empirical {B}ayes estimation of treatment effects with many a/b
  tests: An overview.
\newblock In \emph{AEA Papers and Proceedings}, volume 109, pages 43--47, 2019.

\bibitem[Bastani et~al.(2022)Bastani, Simchi-Levi, and Zhu]{bastani2022meta}
H.~Bastani, D.~Simchi-Levi, and R.~Zhu.
\newblock Meta dynamic pricing: Transfer learning across experiments.
\newblock \emph{Management Science}, 68\penalty0 (3):\penalty0 1865--1881,
  2022.

\bibitem[Besbes et~al.(2015)Besbes, Gur, and Zeevi]{besbes15}
O.~Besbes, Y.~Gur, and A.~Zeevi.
\newblock Non-stationary stochastic optimization.
\newblock \emph{Operations Research}, 63\penalty0 (5):\penalty0 1227--1244,
  2015.

\bibitem[Beygelzimer et~al.(2011)Beygelzimer, Langford, Li, Reyzin, and
  Schapire]{beygelzimer2011contextual}
A.~Beygelzimer, J.~Langford, L.~Li, L.~Reyzin, and R.~Schapire.
\newblock Contextual bandit algorithms with supervised learning guarantees.
\newblock In \emph{Proceedings of the Fourteenth International Conference on
  Artificial Intelligence and Statistics}, pages 19--26. JMLR Workshop and
  Conference Proceedings, 2011.

\bibitem[Bubeck and Slivkins(2012)]{bubeck2012best}
S.~Bubeck and A.~Slivkins.
\newblock The best of both worlds: Stochastic and adversarial bandits.
\newblock In \emph{Conference on Learning Theory}, pages 42--1. JMLR Workshop
  and Conference Proceedings, 2012.

\bibitem[Bubeck et~al.(2009)Bubeck, Munos, and Stoltz]{bubeck2009pure}
S.~Bubeck, R.~Munos, and G.~Stoltz.
\newblock Pure exploration in multi-armed bandits problems.
\newblock In \emph{International conference on Algorithmic learning theory},
  pages 23--37. Springer, 2009.

\bibitem[Caria et~al.(2020)Caria, Kasy, Quinn, Shami, Teytelboym,
  et~al.]{caria2020adaptive}
S.~Caria, M.~Kasy, S.~Quinn, S.~Shami, A.~Teytelboym, et~al.
\newblock An adaptive targeted field experiment: Job search assistance for
  refugees in jordan, 2020.

\bibitem[Chapelle and Li(2011)]{chapelle2011empirical}
O.~Chapelle and L.~Li.
\newblock An empirical evaluation of {T}hompson sampling.
\newblock \emph{Advances in neural information processing systems},
  24:\penalty0 2249--2257, 2011.

\bibitem[Cheung et~al.(2019)Cheung, Simchi-Levi, and Zhu]{cheung19}
W.~C. Cheung, D.~Simchi-Levi, and R.~Zhu.
\newblock Learning to optimize under non-stationarity.
\newblock In \emph{International Conference on Artificial Intelligence and
  Statistics}, pages 1079--1087. PMLR, 2019.

\bibitem[Chick et~al.(2021)Chick, Gans, and Yapar]{chick2021bayesian}
S.~E. Chick, N.~Gans, and {\"O}.~Yapar.
\newblock Bayesian sequential learning for clinical trials of multiple
  correlated medical interventions.
\newblock \emph{Management Science}, 2021.

\bibitem[Degenne et~al.(2019)Degenne, Nedelec, Calauzenes, and
  Perchet]{degenne19a}
R.~Degenne, T.~Nedelec, C.~Calauzenes, and V.~Perchet.
\newblock Bridging the gap between regret minimization and best arm
  identification, with application to a/b tests.
\newblock In \emph{Proceedings of the Twenty-Second International Conference on
  Artificial Intelligence and Statistics}, volume~89, pages 1988--1996, 2019.

\bibitem[Dimmery et~al.(2019)Dimmery, Bakshy, and Sekhon]{dimmery2019shrinkage}
D.~Dimmery, E.~Bakshy, and J.~Sekhon.
\newblock Shrinkage estimators in online experiments.
\newblock In \emph{Proceedings of the 25th ACM SIGKDD International Conference
  on Knowledge Discovery \& Data Mining}, pages 2914--2922, 2019.

\bibitem[Dud{\'\i}k et~al.(2011)Dud{\'\i}k, Hsu, Kale, Karampatziakis,
  Langford, Reyzin, and Zhang]{dudik2011efficient}
M.~Dud{\'\i}k, D.~Hsu, S.~Kale, N.~Karampatziakis, J.~Langford, L.~Reyzin, and
  T.~Zhang.
\newblock Efficient optimal learning for contextual bandits.
\newblock In \emph{Proceedings of the Twenty-Seventh Conference on Uncertainty
  in Artificial Intelligence}, pages 169--178, 2011.

\bibitem[Fan and Glynn(2021)]{fan2021diffusion}
L.~Fan and P.~W. Glynn.
\newblock Diffusion approximations for thompson sampling.
\newblock \emph{arXiv preprint arXiv:2105.09232}, 2021.

\bibitem[Farias et~al.(2022)Farias, Moallemi, Peng, and
  Zheng]{farias2022synthetically}
V.~Farias, C.~Moallemi, T.~Peng, and A.~Zheng.
\newblock Synthetically controlled bandits.
\newblock \emph{arXiv preprint arXiv:2202.07079}, 2022.

\bibitem[Frazier et~al.(2008)Frazier, Powell, and
  Dayanik]{frazier2008knowledge}
P.~I. Frazier, W.~B. Powell, and S.~Dayanik.
\newblock A knowledge-gradient policy for sequential information collection.
\newblock \emph{SIAM Journal on Control and Optimization}, 47\penalty0
  (5):\penalty0 2410--2439, 2008.

\bibitem[Jamieson and Talwalkar(2016)]{jamieson2016non}
K.~Jamieson and A.~Talwalkar.
\newblock Non-stochastic best arm identification and hyperparameter
  optimization.
\newblock In \emph{Artificial Intelligence and Statistics}, pages 240--248.
  PMLR, 2016.

\bibitem[Joulani et~al.(2013)Joulani, Gyorgy, and
  Szepesv{\'a}ri]{joulani2013online}
P.~Joulani, A.~Gyorgy, and C.~Szepesv{\'a}ri.
\newblock Online learning under delayed feedback.
\newblock In \emph{International Conference on Machine Learning}, pages
  1453--1461. PMLR, 2013.

\bibitem[Jourdan et~al.(2022)Jourdan, Degenne, Baudry, de~Heide, and
  Kaufmann]{jourdan2022top}
M.~Jourdan, R.~Degenne, D.~Baudry, R.~de~Heide, and E.~Kaufmann.
\newblock Top two algorithms revisited.
\newblock \emph{Advances in Neural Information Processing Systems},
  35:\penalty0 26791--26803, 2022.

\bibitem[Kandasamy et~al.(2018)Kandasamy, Krishnamurthy, Schneider, and
  P{\'o}czos]{kandasamy2018parallelised}
K.~Kandasamy, A.~Krishnamurthy, J.~Schneider, and B.~P{\'o}czos.
\newblock Parallelised {B}ayesian optimisation via {T}hompson sampling.
\newblock In \emph{International Conference on Artificial Intelligence and
  Statistics}, pages 133--142. PMLR, 2018.

\bibitem[Kaufmann et~al.(2016)Kaufmann, Capp{\'e}, and
  Garivier]{kaufmann2016complexity}
E.~Kaufmann, O.~Capp{\'e}, and A.~Garivier.
\newblock On the complexity of best-arm identification in multi-armed bandit
  models.
\newblock \emph{The Journal of Machine Learning Research}, 17\penalty0
  (1):\penalty0 1--42, 2016.

\bibitem[Kim and Nelson(2006)]{kim2006selecting}
S.-H. Kim and B.~L. Nelson.
\newblock Selecting the best system.
\newblock \emph{Handbooks in operations research and management science},
  13:\penalty0 501--534, 2006.

\bibitem[Kohavi et~al.(2020)Kohavi, Tang, and Xu]{kohavi2020trustworthy}
R.~Kohavi, D.~Tang, and Y.~Xu.
\newblock \emph{Trustworthy online controlled experiments: A practical guide to
  a/b testing}.
\newblock Cambridge University Press, 2020.

\bibitem[Krishnamurthy et~al.(2023)Krishnamurthy, Zhan, Athey, and
  Brunskill]{krishnamurthy2023proportional}
S.~K. Krishnamurthy, R.~Zhan, S.~Athey, and E.~Brunskill.
\newblock Proportional response: Contextual bandits for simple and cumulative
  regret minimization, 2023.

\bibitem[Kuang and Wager(2023)]{kuang2023weak}
X.~Kuang and S.~Wager.
\newblock Weak signal asymptotics for sequentially randomized experiments,
  2023.

\bibitem[Lai and Robbins(1985)]{lai1985asymptotically}
T.~L. Lai and H.~Robbins.
\newblock Asymptotically efficient adaptive allocation rules.
\newblock \emph{Advances in applied mathematics}, 6\penalty0 (1):\penalty0
  4--22, 1985.

\bibitem[Lattimore and Szepesv{\'a}ri(2020)]{lattimore2020bandit}
T.~Lattimore and C.~Szepesv{\'a}ri.
\newblock \emph{Bandit Algorithms}.
\newblock Cambridge University Press, 2020.

\bibitem[Li et~al.(2010)Li, Chu, Langford, and Schapire]{li2010contextual}
L.~Li, W.~Chu, J.~Langford, and R.~E. Schapire.
\newblock A contextual-bandit approach to personalized news article
  recommendation.
\newblock In \emph{Proceedings of the 19th international conference on World
  wide web}, pages 661--670, 2010.

\bibitem[Lieb(1973)]{lieb1973convex}
E.~H. Lieb.
\newblock Convex trace functions and the wigner-yanase-dyson conjecture.
\newblock \emph{Advances in Mathematics}, 11\penalty0 (3):\penalty0 267--288,
  1973.

\bibitem[McDonald et~al.(2023)McDonald, Maystre, Lalmas, Russo, and
  Ciosek]{mcdonald2023impatient}
T.~M. McDonald, L.~Maystre, M.~Lalmas, D.~Russo, and K.~Ciosek.
\newblock Impatient bandits: Optimizing recommendations for the long-term
  without delay.
\newblock In \emph{Proceedings of the 29th ACM SIGKDD International Conference
  on Knowledge Discovery \& Data Mining}, 2023.

\bibitem[Mellor and Shapiro(2013)]{mellor13}
J.~Mellor and J.~Shapiro.
\newblock Thompson sampling in switching environments with {B}ayesian online
  change point detection.
\newblock In \emph{Artificial Intelligence and Statistics}, pages 442--450.
  PMLR, 2013.

\bibitem[Min and Russo(2023)]{min23c}
S.~Min and D.~Russo.
\newblock An information-theoretic analysis of nonstationary bandit learning.
\newblock In \emph{Proceedings of the 40th International Conference on Machine
  Learning}, pages 24831--24849. PMLR, 2023.

\bibitem[Qin et~al.(2017)Qin, Klabjan, and Russo]{qin2017improving}
C.~Qin, D.~Klabjan, and D.~Russo.
\newblock Improving the expected improvement algorithm.
\newblock \emph{Advances in Neural Information Processing Systems},
  2017:\penalty0 5382--5392, 2017.

\bibitem[Rubin(1979)]{rubin1979using}
D.~B. Rubin.
\newblock Using multivariate matched sampling and regression adjustment to
  control bias in observational studies.
\newblock \emph{Journal of the American Statistical Association}, 74\penalty0
  (366a):\penalty0 318--328, 1979.

\bibitem[Rusmevichientong and Tsitsiklis(2010)]{rusmevichientong2010linearly}
P.~Rusmevichientong and J.~N. Tsitsiklis.
\newblock Linearly parameterized bandits.
\newblock \emph{Mathematics of Operations Research}, 35\penalty0 (2):\penalty0
  395--411, 2010.

\bibitem[Russo(2020)]{russo2020simple}
D.~Russo.
\newblock Simple {B}ayesian algorithms for best-arm identification.
\newblock \emph{Operations Research}, 68\penalty0 (6):\penalty0 1625--1647,
  2020.

\bibitem[Russo and Van~Roy(2016)]{russo2016information}
D.~Russo and B.~Van~Roy.
\newblock An information-theoretic analysis of {T}hompson sampling.
\newblock \emph{The Journal of Machine Learning Research}, 17\penalty0
  (1):\penalty0 2442--2471, 2016.

\bibitem[Russo and Van~Roy(2018)]{russo2018learning}
D.~Russo and B.~Van~Roy.
\newblock Learning to optimize via information-directed sampling.
\newblock \emph{Operations Research}, 66\penalty0 (1):\penalty0 230--252, 2018.

\bibitem[Russo and Zou(2019)]{russo2019much}
D.~Russo and J.~Zou.
\newblock How much does your data exploration overfit? controlling bias via
  information usage.
\newblock \emph{IEEE Transactions on Information Theory}, 66\penalty0
  (1):\penalty0 302--323, 2019.

\bibitem[Russo et~al.(2018)Russo, Van~Roy, Kazerouni, Osband, and
  Wen]{russo2018tutorial}
D.~J. Russo, B.~Van~Roy, A.~Kazerouni, I.~Osband, and Z.~Wen.
\newblock A tutorial on {T}hompson sampling.
\newblock \emph{Foundations and Trends{\textregistered} in Machine Learning},
  11\penalty0 (1):\penalty0 1--96, 2018.

\bibitem[Scott(2010)]{scott2010modern}
S.~L. Scott.
\newblock A modern {B}ayesian look at the multi-armed bandit.
\newblock \emph{Applied Stochastic Models in Business and Industry},
  26\penalty0 (6):\penalty0 639--658, 2010.

\bibitem[Shang et~al.(2020)Shang, Heide, Menard, Kaufmann, and
  Valko]{shang2020fixed}
X.~Shang, R.~Heide, P.~Menard, E.~Kaufmann, and M.~Valko.
\newblock Fixed-confidence guarantees for {B}ayesian best-arm identification.
\newblock In \emph{International Conference on Artificial Intelligence and
  Statistics}, pages 1823--1832. PMLR, 2020.

\bibitem[Suk and Kpotufe(2022)]{suk22}
J.~Suk and S.~Kpotufe.
\newblock Tracking most significant arm switches in bandits.
\newblock In \emph{Conference on Learning Theory}, pages 2160--2182. PMLR,
  2022.

\bibitem[Thompson(1933)]{thompson1933likelihood}
W.~R. Thompson.
\newblock On the likelihood that one unknown probability exceeds another in
  view of the evidence of two samples.
\newblock \emph{Biometrika}, 25\penalty0 (3/4):\penalty0 285--294, 1933.

\bibitem[Tropp(2011)]{tropp2011freedman}
J.~A. Tropp.
\newblock Freedman's inequality for matrix martingales.
\newblock \emph{Electronic Communications in Probability}, 16:\penalty0
  262--270, 2011.

\bibitem[Tropp(2012)]{tropp2012user}
J.~A. Tropp.
\newblock User-friendly tail bounds for sums of random matrices.
\newblock \emph{Foundations of computational mathematics}, 12\penalty0
  (4):\penalty0 389--434, 2012.

\bibitem[Tropp(2015)]{tropp2015introduction}
J.~A. Tropp.
\newblock An introduction to matrix concentration inequalities.
\newblock \emph{Foundations and Trends{\textregistered} in Machine Learning},
  8\penalty0 (1-2):\penalty0 1--230, 2015.

\bibitem[Trovo et~al.(2020)Trovo, Paladino, Restelli, and Gatti]{trovo20}
F.~Trovo, S.~Paladino, M.~Restelli, and N.~Gatti.
\newblock Sliding-window {T}hompson sampling for non-stationary settings.
\newblock \emph{Journal of Artificial Intelligence Research}, 68:\penalty0
  311--364, 2020.

\bibitem[Wu and Wager(2022)]{wu2022thompson}
H.~Wu and S.~Wager.
\newblock Thompson sampling with unrestricted delays.
\newblock \emph{arXiv preprint arXiv:2202.12431}, 2022.

\bibitem[Zhong et~al.(2023)Zhong, Cheung, and Tan]{zhong2023achieving}
Z.~Zhong, W.~C. Cheung, and V.~Y.~F. Tan.
\newblock Achieving the pareto frontier of regret minimization and best arm
  identification in multi-armed bandits, 2023.

\bibitem[Zhou et~al.(2019)Zhou, Xu, and Blanchet]{zhou2019learning}
Z.~Zhou, R.~Xu, and J.~Blanchet.
\newblock Learning in generalized linear contextual bandits with stochastic
  delays.
\newblock \emph{Advances in Neural Information Processing Systems},
  32:\penalty0 5197--5208, 2019.

\end{thebibliography}
}

\appendix 

\section{Additional examples}\label{sec:more_examples}

We illustrate two very different models of exogenous variation that can be viewed as special cases of our general problem formulation. The first example considers a bandit experiment where a single observable factor --- a user's country --- explains the non-stationary pattern of rewards. Of course, this is a simplified example. One may include many other observable features and also create more intricate models that combine observable factors with the latent ones modeled in Example \ref{ex:latent_confounders}. 
\begin{example}[Comprehensible observed contexts]\label{ex:time_zones}
	A video streaming website is testing a small change to the layout of its homepage. 
	The platform operates in $d$ different countries, and the context $X_t \in \{e_1, \ldots, e_d\}$ is a standard basis vector that encodes a user's country.  We assume this is a recorded feature.
	
	The target population context vector $x_{\rm pop} = (x_{{\rm pop},1},\ldots, x_{{\rm pop},d}) \in \mathbb{R}^d$ measures the long-term fraction of user visits among those who hail from each country. The platform estimates this by querying a database that records all user visits over the past several months. (Implicit in this approach is an assumption that $x_{\rm pop}$ is a reasonable reflection of the users who will visit over the next few months.)
	
	Individuals tend to visit this video streaming website between 7pm-11pm in their local timezone. 
	Due to timezone differences, the mix of countries among users arriving during a particular hour within the experiment may not reflect the population proportions. Thankfully, the decision-maker can use Bayesian inference to project the population level performance of each treatment arm. We illustrate this in the case when components of $\theta$ are independent. In that case,	 	
	\begin{equation}\label{eq:time_zones_population_level}
		\E\left[ r_{\theta}(i)  \mid H_t \right]  = \sum_{c=1}^{d} x_{{\rm pop},c} \E\left[\theta^{(i)}_{c} \mid H_t \right]     
	\end{equation}
	where
	\[
	\E\left[\theta^{(i)}_{c} \mid H_t \right]  = \frac{  {\rm Var}\left(\theta^{(i)}_{c}\right)^{-1} \E\left[\theta^{(i)}_{c}\right]   + \sigma^{-2}\sum_{\ell=1}^{t-L} \ind(X_\ell= c, I_\ell=i) R_\ell}{ {\rm Var}\left(\theta^{(i)}_{c}\right)^{-1}+ \sigma^{-2}\sum_{\ell=1}^{t-L} \ind(X_\ell= c, I_\ell=i) }.
	\]
	As the volume of data grows, the prior washes away and country/arm-specific means are estimated through an empirical averaging. The population average reward is estimated in \eqref{eq:time_zones_population_level}. 
	This is a Bayesian analogue of a very common technique known as post-stratification. 
\end{example}

The next example illustrates that it is possible to combine the modeling approaches taken in Examples \ref{ex:time_zones} and \ref{ex:latent_confounders}.
\begin{example}[A mixture of latent and observed factors]\label{ex:combined}
	The  context at time $t$ is a tuple $X_t=(1, Z_t,e_t) \in \mathbb{R}^{1+p+T}$, where the vector $Z_{t} \in \mathbb{R}^p$ encodes other observable user features, like the country in Example~\ref{ex:time_zones}, and $e_t \in \mathbb{R}^T$ is the $t^{\rm th}$ standard basis vector and indicates the current time period. Take $x_{\rm pop} = (1 \, ,\, z_{\rm pop} \, , \, \frac{1}{T} \sum_{t=1}^{T} e_t)$, where $z_{\rm pop}$ is a population effect. 
	Rather than specify a prior mean and covariance over the latent parameters $\theta = (\theta^{(1)}, \ldots, \theta^{(k)})$, it is more interpretable to write  $\theta^{(i)} = (\alpha^{(i)}\, ,\,  \gamma^{(i)} + \beta \, , \, \epsilon)$  and specify a prior mean and covariance for jointly Gaussian latent parameters $(\alpha^{(i)})_{i \in[k]} \in \mathbb{R}^k$,  $(\gamma^{(i)})_{i\in [k]} \in \mathbb{R}^{kp}$, $\beta \in \mathbb{R}^p$, and $\epsilon = (\epsilon_t)_{t\in [T]} \in \mathbb{R}^T$. Under this model, potential arm reward,
	\[
	R_{t,i} = \underbrace{\alpha^{(i)}}_{\text{arm eff.}}  + \underbrace{\langle  \gamma^{(i)} \,,\,  Z_t \rangle}_{\text{interaction eff.}} + \underbrace{\langle \beta \, , Z_t \, \rangle}_{ \text{context eff.} } + \underbrace{\epsilon_t}_{\text{time eff.} } + W_{t,i},
	\]
	is determined by an arm-specific effect, an interaction effect $\langle  \gamma^{(i)} \,,\,  Z_t \rangle$ between an arm-specific parameter and the observable user features, and arm-shared effects  explained by observable user features or a latent time trend. 	A prior $\gamma^{(i)} \sim N(0,  \lambda^2 I )$ where $\lambda$ is a small scalar  causes the decision-maker to shrink the posterior mean of arm-specific parameters toward zero. 
	The population mean reward of an arm,	
	\[
	r_{\theta}(i) = \alpha^{(i)} + \langle \gamma^{(i)} \, , z_{\rm pop} \, \rangle +  \underbrace{ \langle \beta \, , z_{\rm pop} \, \rangle   + \frac{1}{T}\sum_{t=1}^{T} \epsilon_t }_{\text{independent of arm}},
	\]
	measures how arm $i$ would have performed in hindsight over the past $T$ periods among a cohort of users whose average observable features match $z_{\rm pop}$. 
\end{example}

\section{An additional theoretical theoretical guarantee: a bound on contextual regret}\label{sec:contextual_regret}

Define the contextual regret of an algorithm $\Delta_t(X_t) = r_{\theta}(I^*, X_t) -  r_{\theta}(I_t, X_t)$ to be the shortfall in performance of the chosen arm $I_t$ in some context within the experiment, relative to the reward that would have been earned under the utilitarian optimal arm $I^*$. In fact,  
$\E[\Delta_{t}(X_t)] = \E \left[R_{t,I^*} -  R_{t,I_t} \right]$. 

Bounds on cumulative contextual regret can be interpreted as a limit on the decrease in reward that results from the necessity to experiment in order to learn $I^*$. 
Remark \ref{rem:contextual_regret} below highlights that caution is needed when comparing algorithms in terms of their contextual regret, as it is possible to attain negative contextual regret by systematically violating the reasons the experimenter aimed to deploy a stable treatment arm in the first place. 

Somewhat remarkably, the next result bounds the shortfall in reward accrued under DTS in terms of the number of actions, placing no conditions at all on the dimension of the context space or the pattern of nonstationarity in rewards that the context sequence may induce.  

\begin{proposition}[Bound on within-experiment contextual regret]\label{prop:contextual} Fix any context sequence $x_{1:T} \in \Xc^T$.	If $L=1$ (no observation delay), then under DTS, 
	\begin{equation}\label{eq:contextual_regret_bound}
		\frac{\E\left[\sum_{t=1}^{T}  \Delta_t(X_t) \mid X_{1:T} = x_{1:T} \right]}{T} \leq    \sigma_R \sqrt{  \frac{ 2 k \log(k)  }{ T }},
	\end{equation}
	where $\sigma_R^2 = \max_{t\in [T], i \in [k]} {\rm Var}(R_{t,i} \mid X_t=x_t)$. 
\end{proposition}
 \begin{proof} The proof follows the information-theoretic analysis of \cite{russo2016information}. While that paper studies vanilla Thompson sampling in i.i.d.~environments, the proof applies without substantial changes to DTS in nonstationary environments.   
 	
 	We use $\ent(Z)$ to denote the entropy of a random variable $Z$ and $\I(Z_1,Z_2)$ to denote mutual information between $Z_1$ and $Z_2$. 
 	Let $G_t =  \I_{ \Prob(\cdot \mid H_t, X_t  ) }\left(   I^* ;    (I_t, R_{t,I_t}  )  \right)$ be the mutual information (or 'information gain') between $I^*$ and the observation  $(I_t, R_{t,I_t}  )$ under the conditional probability measure $ \Prob(\cdot \mid H_t, X_t  )$. This is a random variable due to the randomness in $\Prob(\cdot \mid H_t, X_t  )$. The convention in information theory is to integrate over that randomness, with conditional mutual information defined as  $\I\left(   I^* ;    (I_t, R_{t,I_t}  ) \mid  H_t, X_t  \right) =\E[G_t]$. 
 	
 	Following Proposition 3, and Corollary 1, of \cite{russo2016information},  the probability matching property of DTS, $\Prob(I_t = i \mid H_t, X_t) = \Prob(I^* = i \mid H_t, X_t)$ implies the following bound on the so-called `information ratio':
 	\begin{equation}\label{eq:info_ratio_bound}
 		\Gamma_t =  \frac{\left(\E \left[ R_{t, I^*} - R_{t,I_t}  \mid H_t, X_t \right] \right)^2}{  \I_{t}\left(   I^* ;    (I_t, R_{t,I_t}  )  \right)   }   \leq 2 \sigma_R^2 k \triangleq \bar{\Gamma}.
 	\end{equation}
 	
 	Re-arranging this expression summing over $t$  
 	\begin{align*}
 		\E\left[\sum_{t=1}^{T}  \Delta_t(X_t)  \right]=	\E\left[ \sum_{t=1}^{T}  R_{t, I^*} - R_{t,I_t}    \right] = \E\left[ \E\left[ \sum_{t=1}^{T}  R_{t, I^*} - R_{t,I_t} \mid H_t, X_t   \right] \right] &= \E\left[ \sum_{t=1}^{T}  \sqrt{\Gamma_t G_t} \right] \\
 		&\leq   \sqrt{ T\cdot \bar{\Gamma}  \cdot \E\left[ \sum_{t=1}^{T} G_t \right]}. 
 	\end{align*}
 	Now we show that expected cumulative information gain is bounded by prior entropy. We have, 
 	\[
 	\E\left[ G_t \right] =  \I\left(   I^* ;    (I_t, R_{t,I_t}  ) \mid  H_t, X_t  \right) =  \I\left(   I^* ;    (X_t, I_t, R_{t,I_t}  ) \mid  H_t  \right)   - \I\left(   I^* ;    X_t \mid  H_t  \right)  =  \I\left(   I^* ;    ( X_t , I_t, R_{t,I_t}) \mid  H_t \right),
 	\]
 	where the first equality uses the chain rule and the second uses that $X_t$ is independent of $\theta$. Now, since $H_t=(X_{\ell}, I_\ell, R_{\ell, I_\ell})_{\ell \leq t-1}$, the chain rule and non-negativity of conditional entropy imply,
 	\[
 	\E\left[ \sum_{t=1}^{T} G_t \right] = \sum_{t=1}^{T} \I\left(   I^* ;    ( X_t , I_t, R_{t,I_t}) \mid  H_t \right) = \I\left(   I^* ;    H_{T+1} \right) = \ent(I^*) - \ent(I^* \mid H_{T+1} ) \leq \ent(I^*).
 	\]
 	The final claim follows from using the coarse upper bound $\ent(I^*)\leq \log(k)$ and dividing by $T$. 
 \end{proof}

 \begin{remark}[Care is needed when interpreting contextual regret]\label{rem:contextual_regret}
 	Imagine treatment arms represent possible prices, rewards reflect revenue earned by displaying a price to a customer, and context observations are features of the customer. Suppose those customer features are predictive of the customer's race. It is plausible that pricing based on race would increase revenue, but the company understands that this to be illegal, unethical, and reputationally damaging.  For that reason, they seek to deploy a fixed price, $I_{\rm post}$, after the experiment. In this setting, a bandit algorithm could attain low---even negative --- within-experiment contextual regret by targeting its prices based on customer features. But then the algorithm's decision-making within the experiment clearly goes against the way the company hopes to make decisions post-experiment. More examples like this are discussed in Appendix \ref{sec:standardization}. 
 \end{remark}

\section{Discussion of adversarial nonstationary bandits}\label{sec:adversarial}

A special case of our formulation produces a Bayesian analogue of common adversarial bandit models \citep{auer2002nonstochastic,lattimore2020bandit} . Assume the context at time $t$ is the $t^{\rm th}$ standard basis vector: $X_t  = e_t \in \mathbb{R}^T$. The reward at time $t$ is then 
\[
R_{t,I_t} = \theta_{t}^{(I_t)} + W_{t,I_t}.
\]
If one chooses $x_{\rm pop}=(1/T, \ldots, 1/T)$, then
\[
I^* = \argmax_{i \in [k]} \frac{1}{T} \sum_{t=1}^{T} \theta_{t}^{(i)}
\]  
is the best-arm in hindsight over the course of the experiment and per-period within-experiment contextual regret (See Appendix~\ref{sec:contextual_regret}),
\[ 
\frac{1}{T} \sum_{t=1}^{T} \Delta_t(X_t) = \max_{i\in [k]}  \frac{1}{T} \sum_{t=1}^{T} \theta_{t}^{(i)} - \frac{1}{T}\sum_{t=1}^{T}\theta_{t}^{(I_t)} 
\]
benchmarks the performance of selected arms within the experiment against the best fixed arm.
This matches the performance measure in the adversarial bandit literature.
Post-experiment utilitarian regret assesses whether the algorithm can select an arm $I_{\rm post}$ at the end of the experiment whose hindsight performance is competitive with that of the hindsight-optimal arm $I^*$. 

In this special case, our bound on contextual regret in Proposition \ref{prop:contextual} is then reminiscent of results in the adversarial bandit literature. 
Indeed that case, $O(\sqrt{k \log(k)/T})$ regret bounds are well known, even when rewards are picked by an adversary \citep{auer2002nonstochastic}. What distinguishes Proposition \ref{prop:contextual} is that it applies to a very different algorithm. 

Our model and algorithm deviates from the adversarial bandit literature in two substantive ways. Both may allow the DM to write off arms with poor population-level performance earlier in the experiment than would be possible in a typical adversarial model:
\begin{enumerate}
	\item A structured prior distribution over $\theta$ may restrict the form of nonstationarity that is plausible. Classical i.i.d.~bandits are an extreme special case in which the covariance structure over $\theta_1,\ldots, \theta_T$ is degenerate. Other structured priors, like Example \ref{ex:latent_confounders}, would guide an algorithm like DTS to guard against particular forms of nonstationarity.
	\item Our formulation accommodates rich contextual observations that capture features beyond the current time period. Example \ref{ex:time_zones}, presented in Appendix \ref{sec:more_examples}, provides a simple illustration. In that example, it may be possible to infer an arm's population-level performance early in the experiment. 
\end{enumerate}

It is also worth emphasizing that a bound on post-experiment regret, like  Theorem \ref{thm:main_result}, cannot be deduced from bonds on cumulative within-experiment contextual regret, like Proposition \ref{prop:contextual}. So-called ``online-to-batch'' conversions do not work when the rewards sequence is nonstationary.

\section{Comparison to contextual bandits and extension to policy learning problems}\label{sec:policy_learning}

\subsection{Comparison to linear contextual bandit models} The information structure of our problem corresponds to that in a classical linear contextual bandit problem \citep{li2010contextual,agrawal2013thompson}, but the learning objective differs. In a typical linear contextual bandit problem, the DM wishes to learn an optimal treatment-rule $\pi^*: \Xc \to [k]$ satisfying $\pi^*(x) \in \argmax_{i \in [k]} \E[R_{t,i} \mid \theta, X_t=x]$ for each $x\in \Xc$.
TS for contextual bandit problem selects an arm $I_t$ at time $t$ randomly with sampling probabilities,
\[
\Prob(I_t = i \mid H_t, X_t) = \Prob(\pi^*(X_t) = i \mid H_t, X_t),
\]
which are matched to the posterior distribution of the reward maximizing arm \emph{in the current context}.

Practitioners are, inevitably, faced with question of which features to include in the context vector $x\in \Xc$. If one is using contextual TS, including a feature has two implications:
\begin{description}
	\item[Inference:] Including a feature in the context vectors directs the algorithm to `control' for past variation in this feature when making inferences about the reward an arm will generate in the future. 
	\item[Reactivity:] Including a feature in the context vectors directs the algorithm to segment decisions it makes on the basis of this feature. In practice, this could mean that individuals who are different along this dimension receive different treatments or that, across several interactions, an individual receives different treatments as this feature changes. 
\end{description}
Under our model, these two issues are decoupled. Deconfounded Thompson sampling is designed to account for contextual variation when performing inference while still learning a population level decision-rule that is not reactive to context. Appendix \ref{sec:failure} presents an example in which contextual TS fails to gather the information necessary to select a good population-level arm $I_{\rm post}$; simply put, its exploration is directed toward a different goal.

Why would an experimenter aim to learn a policy that does not react to (some components of) an observed context? 
One reason, which cuts across applications, is that this can vastly reduce data requirements. 
Beyond this, we provide a substantive discussion in Section \ref{sec:standardization}. We now extend DTS to react to \emph{particular} components of the context.

\subsection{Generalization of DTS}

We sketch a generalization of DTS which aims, suppose the goal is to identify the best policy from a pre-specified class $\Pi$. Each element $\pi \in \Pi$ is a mapping from $\Xc$  to $[k]$.  Overloading notation, take  
\begin{equation}\label{eq:policy_level_reward}
r_{\theta}(\pi) = \int_{\Xc} r_{\theta}( \pi(x) \,,\, x ) \mathrm{d}\Dc_{\rm pop}(x) 
\end{equation}
to be the average reward accrued by $\pi$ under the population,  generalizing \eqref{eq:pop_avg_reward}. Define $\pi^* \in \argmax_{\pi \in \Pi} r_{\theta}(\pi)$ to the policy within the policy class which maximizes average reward under the true parameter $\theta$. To simplify the presentation, assume that this maximum exists and is unique almost surely. 

This objective can interpolate between two extremes:
\begin{enumerate}
	\item \emph{Complete standardization:} The policy class $\Pi=\{\pi^{(1)}, \ldots, \pi^{(k)} \}$ has just $k$ elements. Each $\pi^{(i)}$ maps any $x\in \Xc$ to $\pi(x)=i$, corresponding to a decision-rule that does not segment its decisions on the basis of context. This models the hypothetical A/B test considered in the introduction and recovers our formulation in Section \ref{sec:formulation}. 
	\item \emph{Complete personalization:} This policy class $\Pi$ contains all possible functions mapping $\Xc$ to $[k]$. This is a common formulation in contextual bandit models \citep{li2010contextual}. Given perfect knowledge of $\theta$,  the optimal policy plays $\pi^*(x) \in \argmax_{i \in [k]} r_{\theta}(i,x)$. In this sense, optimal decision-making completely decouples across contexts.
\end{enumerate}
The next two examples illustrate policy classes in between these two extremes. 
\begin{example}[Segmentation]
	The context space is divided into $m$ disjoints  segments as $\Xc= \Xc_1 \cup \cdots \cup \Xc_m$. Segments may, for instance, represent distinct geographical regions.  The policy class $\Pi$ consists of all rules $\pi:\Xc \to [k]$ obeying for each segment $j$ the constraint $\pi(x) = \pi(x')$ for all $x,x' \in \Xc_j$. That is, the policy class consists of rules that associate each segment with an action.
\end{example}
\begin{example}[Protected features]
	A context $x = x_{1:d} = (x_1,\ldots, x_d)$ is divided into two parts. The policy can be react to the first $d_0$ features when selecting actions but features $x_{d_0+1 : x_d}$ are protected attributes that may only be used to deconfound inferences when looking at past data. Formally, the policy class 
	\[
	\Pi = \left\{ \pi:\Xc \to [k]  \mid  x_{1:d_0} = x'_{1:d_0} \implies \pi(x) = \pi(x')     \right\}
	\]
	consists of all decision-rules whose output is invariant to the protected attributes.	 
\end{example}
Natural justifications for constraining $\Pi$ are discussed at length in Section \ref{sec:standardization}.

Let us generalize DTS to treat such problems. We view DTS as a rule for selecting a sequence of policies $(\pi_1,\ldots, \pi_T, \pi_{\rm post})$. Within the experiment,  the arm selected at time $t$ is determined as $I_t=\pi_t(X_t)$; one could equivalently view DTS as a rule for selecting these arms.  In constructing the reward measure $r_{\theta}(\pi_{\rm post})$, we implicitly assume post-experiment decisions are by applying $\pi_{\rm post})$ to the observed context.  The model defining reward realizations within the experiment is the same as before.
Building on the definitions of DTS in \eqref{eq:probability_matching_def} and \eqref{eq:bayes_selection}, generalized DTS randomly samples a policy at time $t$ according to
\begin{equation}\label{eq:probability_matching_policy}
\Prob(\pi_t = \pi \mid H_t) = \Prob(\pi^* = \pi \mid H_t)   \quad \forall \pi \in \Pi 
\end{equation}
and selects a policy to deploy in the population according to
\[
\pi_{\rm post} \in \argmax_{\pi \in \Pi} \E\left[ r_{\theta}(\pi) \mid H_{\rm post} \right]. 
\]
With a completely standardized policy class, this algorithm is DTS. With a completely personalized policy class, it is the standard definition of Thompson sampling in contextual bandits. This is a purely intellectual definition of the algorithm and whether it can be implemented efficiently depends on the structure of the policy class and reward model. 

The next result generalizes Proposition \ref{prop:contextual}, which bounds the the within-experiment contextual regret of DTS. It depends on the entropy of the optimal policy, which is always bounded as $\ent(\pi^*) \leq \log(|\Pi|)$. Under complete standardization, $\ent(\pi^*) \leq \log(k)$, recovering Proposition \ref{prop:contextual}. Under complete personalization, entropy scales with the dimension of the feature vectors, and this proposition roughly yields a bound on the order of $\sigma_{R}\sqrt{kd/T}$. In between these extremes, the entropy term reflects the complexity of the policy class. Similar results that depend on the logarithm of the size of the policy class, rather than entropy, are known in the non-stochastic bandit literature \citep{beygelzimer2011contextual}. 
The novelty in this result is in providing a similar guarantee for a very different type of algorithm, using a different (information-theoretic) proof technique. 
\begin{proposition}[Generalized within-experiment contextual regret]\label{prop:contextual_generalized} Assume $L=1$ (no observation delay). Furthermore, assume the policy class $\Pi$ is finite. Define the within-experiment contextual regret by $\Delta_{t}(X_t) = r_{\theta}\left( \pi^*(X_t) \, , \, X_t \right) - r_{\theta}\left( \pi_t(X_t) \, , \, X_t \right)$. Then, 
	\[
	\frac{\E\left[  \sum_{t=1}^{T}  \Delta_{t}(X_t) \mid  X_{1 :T} = x_{1:T}   \right]}{T}  \leq  \sigma_{R}\sqrt{\frac{ 2 \cdot k \cdot \ent(\pi^*) }{ T}}.
	\]
	where $\sigma_{R}^2= \sup_{t \in [T], i\in [k]} {\rm Var}(R_{t,i} \mid X_t=x_t)$. 
\end{proposition}
This result offers some assurances, but, unfortunately, we do not know how to extend our analysis of utilitarian regret in Theorem \ref{thm:main_result} to analyze this generalized form of DTS. 
Here is a short proof sketch; the same kind of argument was used to prove Lemma 4.12 of \cite{min23c}. 
\begin{proof}
	The proof is the same as that of Proposition \ref{prop:contextual}. We detail only the changes and do not rewrite the proof. Define $I^*_t = \pi^*(X_t)$. The probability matching property with respect to policies in \eqref{eq:probability_matching_policy} implies that $\Prob(I_t = i \mid X_t, H_t) = \Prob(I^*_t = i \mid X_t, H_t)$. Using this, we have the following bound on the so-called `information ratio':
	\begin{equation}
		\Gamma_t  \triangleq  \frac{\left(\E \left[ R_{t, I^*_t} - R_{t,I_t}  \mid H_t, X_t \right] \right)^2}{  \I_{t}\left(   \pi^* ;    (I_t, R_{t,I_t}  )  \right)   }  \leq  \frac{\left(\E \left[ R_{t, I^*_t} - R_{t,I_t}  \mid H_t, X_t \right] \right)^2}{  \I_{t}\left(   I^*_t ;    (I_t, R_{t,I_t}  )  \right)   } \leq 2 \sigma_R^2 k \triangleq \bar{\Gamma},
	\end{equation}
    where the first step is the data processing inequality. The second step is the same as in \eqref{eq:info_ratio_bound} in the proof of Proposition \ref{prop:contextual} and follows using the argument as in Proposition 3, and Corollary 1, of \cite{russo2016information}. 
    From here we use the same argument as in the proof of Proposition \ref{prop:contextual},  but now we define the information gain $G_t =  \I_{t }\left(   \pi^* ;    (I_t, R_{t,I_t}  )  \right)$ as being relative the optimal policy $\pi^*$ rather than the optimal arm $I^*_t$. 	
\end{proof}

\subsection{Reasons for standardization}\label{sec:standardization}
Continuing the discussion above, we might say that DTS implements a \emph{standardized} decision at the end of an experiment, since the arm $I_{\rm post}$ is applied across all future contexts. Despite substantial possible benefits of personalization, 
public policies, operations processes, medical procedures, products, and prices are often relatively standardized. The reasons for this are varied and may be difficult to incorporate into a reward measure associated with an individual's response to the treatment decision: 
	\begin{itemize}
		\item \emph{Operational benefits:} In the example described in Figure \ref{fig:spotify}, selecting a single UI and ML algorithm allows product designers and engineers to maintain and iterate on a standard product. Standardization is ubiquitous in mass-manufactured physical goods or in repeated operations involving humans because of efficiency benefits. 
		\item \emph{Fairness, ethical, or legal constraints:}  In the year 2000, Amazon tested strategies which charged customers different prices for the same good.\footnote{\url{https://www.computerworld.com/article/2588337/amazon-apologizes-for-price-testing-program-that-angered-customers.html}} They faced backlash from customers who believed the practice to be unfair. They appear not to have engaged in the practice since. Many forms of unequal treatment are not only perceived to be unfair, but are  illegal in many countries. 
		\item \emph{Incentive compatibility constraints:} Consider an experiment designed to learn how to price. If the experiment selects a policy or pricing mechanism that charges different prices based on timing or past customer behavior, this mechanism may not be incentive compatible. Customers may respond optimally by modifying behavior to avoid price increases. 
		\item \emph{Social benefits:} On a social media platform, a dating app, or a two sided marketplace, standardizing the product for those who are posting content may improve the experience for those who consume that content. Digital education opens up the possibility of personalizing course content. However, a hidden cost of this is that students would not be able to easily discuss with each other.  
		\item \emph{Consistency benefits:} Users may expect a consistent and familiar experience. In the product testing example in Figure \ref{fig:spotify}, changing the UI based on the user's last ten minutes of usage, or whether it is currently morning or evening, might create an erratic and frustrating experience. 
		\item \emph{Sample complexity benefits:} Much less data may be required to select a single arm than to identify a more complex policy. Our theory makes this formal. 
	\end{itemize}
	Most of these considerations cannot be captured through a policy-level reward function in the form \eqref{eq:policy_level_reward}. Rather than modify the objective function, we have incorporated them via constraints on the policy class. 

\section{Failure of alternative algorithms}\label{sec:failure}

\subsection{Failure of deconfounded UCB: Proof of Lemma \ref{lem:deconfounded_ucb}}
We being by restating the claim in Section \ref{sec:failure_preview}. 
\contextAwareUCB*
\begin{proof}
	Let $U_{t,i}= \E[r_{\theta}(i) \mid H_t] + z \cdot \sqrt{ {\rm Var}(r_{\theta}(i) \mid H_t )}$ denote the UCB of arm $i$. Since $U_{1,2} > U_{1,1}$,  the initial arm selection is $I_1 = 2$. When $\theta^{(2)}_{1}>0$, the posterior mean and standard deviation satisfy  $m_{2,2} = \frac{\theta^{(2)}_{1}}{2} > 0 = m_{2,1}$ and $s_{2,2} = \frac{1}{2}\sqrt{{\rm Var}\left(\theta^{(2)}_1 \mid H_2\right) + {\rm Var}\left(\theta^{(2)}_2 \mid H_2\right) } =\sqrt{\frac{3}{2}} >\sqrt{\frac{1}{2}} = s_{2,1}.$  This implies $U_{2,2} > U_{2,1}$ and arm $I_2=2$ is again selected. This process repeats, showing that if $\theta^{(2)}_{1}>0$, then arm $2$ is chosen for each of the first $\lfloor\frac{T}{2}\rfloor$ periods. We can lower bound simple regret by imagining that the decision-maker has perfect knowledge of $\theta^{(2)}_{1}$, $\theta^{(2)}_{2}$ and $\theta^{(1)}_{2}$ when selecting $I_{\rm post} \in \argmax_{i\in[2]} \E\left[\frac{\theta^{(i)}_1 + \theta^{(i)}_2}{2} \mid H_{\rm post}\right]$, resulting in:
	\[ 
	\E\left[\Delta_{\rm post} \right] \geq \E\left[  \left( \max\left\{ \frac{\theta^{(1)}_{1}  + \theta^{(1)}_{2} }{2}\, ,\, \frac{\theta^{(2)}_{1}  + \theta^{(2)}_{2} }{2} \right\} - \max\left\{ \frac{ \theta^{(1)}_{2} }{2}\, ,\, \frac{\theta^{(2)}_{1}  + \theta^{(2)}_{2} }{2} \right\}    \right)  \ind\left( \theta^{(2)}_{1}>0  \right)    \right] >0.   
	\]
	The strict inequality is due to the gap in Jensen's inequality, reflecting the value of having perfect information about $\theta_{1}^{(1)}$ when making a  decision. 
\end{proof}

\subsection{Failure of context-unaware algorithms}\label{subsec:unaware_algo}

Section \ref{sec:numerical} showed that a context-unaware version of Thompson sampling can fail. Here, we make that observation formal, just as we have for deconfounded UCB.

We define context-unaware Thompson sampling to be an arm algorithm that chooses arm at time $t$ according to
\begin{equation}\label{eq:ts_variant_incorrect}
	I_t \in \argmax_{i\in [k]} \nu_{t,i}    \quad \text{where}  \quad \nu_{t,i} \mid H_t \sim N\left(\tilde{m}_{t,i}\,, \, \tilde{s}^2_{t,i}\right),
\end{equation}
where $\tilde{m}_{t,i}$ and $\tilde{s}^2_{t,i}$ are parameters of a pseudo-posterior, defined below. In \eqref{eq:ts_variant_incorrect}, $\nu_{t,1},\nu_{t,2},\ldots,\nu_{t,k}$ are sampled independently

The pseudo-posterior is updated as if observations were i.i.d. From the algebra of Bayes rule for Gaussian, when $\sigma^2>0$, we define this as 
\[
\tilde{s}^2_{t,i}=\left( \tilde{s}^{-2}_{1,i} + \sigma^{-2}\sum_{\ell=1}^{t-1} \ind(I_\ell = i)  \right)^{-1} 
\quad \text{and} \quad \tilde{m}_{t,i}=\tilde{s}^2_{t,i}\left( \sigma^{-2}\sum_{\ell=1}^{t-1} \ind(I_\ell =i) R_{\ell} \right),
\]
where $ \tilde{s}^{-2}_{1,i}>0$ is some initial value. The natural definition when there is no observation noise (i.e. $\sigma^2=0$) is derived by taking the limit as  $\sigma^{2}\downarrow 0$. In particular, we set $\tilde{s}^2_{t,i}=0$ if arm $i$ has been played previously and $\tilde{m}_{t,i}$ to be $0$ if arm $i$ was never played previously and to be the empirical average reward otherwise.

The next lemma formalizes that this algorithm risks confounding. The same result applies to a context-unaware UCB algorithm, which forms UCBs based on $\tilde{m}_{t,i}$ and $ \tilde{s}^2_{t,i}$. 
At a high-level, these algorithms fail because the way they perform inference does not reflect the problem's true information structure. The proof is provided at the end of this subsection. 

\begin{restatable}[Failure of context-unaware TS]{lemma}{contextUnaware}\label{lem:context_unaware_ts}		
	Consider Example \ref{example:day-of-week-no-obs-noise}, presented in Section \ref{sec:failure_preview}. Suppose the components of the vector $\theta=(\theta^{(i)}_{x})_{i\in [2], x\in [2]}$ are independent with $\theta^{(1)}_{x}\sim N(0,1)$ and $\theta^{(2)}_{x}\sim N(0,2)$ for $x\in [2]$. 
	Let $L=1$ (no delay).
	If \eqref{eq:ts_variant_incorrect} holds, there exists an absolute numerical constant $c>0$ such that  for all $T\in \mathbb{N}$, $\E\left[ \Delta_{\rm post}\right]  \geq c$. 
\end{restatable}

The next remark interprets the failure of context-unaware TS in terms of confounding, using the potential outcomes formalism of \cite{rubin1979using}. 
\begin{remark}[Interpretation as confounding] 
	One can view the failure of context-unaware TS as being driven by 'confounding' due to omitted contextual variables.  Let $\tau$ denote a time drawn uniformly at random from $\{1,\ldots, T\}$, independent of all else. Then the tuple $(I_{\tau}, X_{\tau},  R_{\tau, I_\tau} )$ looks like a random example selected from the data collected by context-unaware TS. 
	By the model assumptions, the following conditional unconfoundedness  (also known as ignorability) condition holds:
	\[
	I_{\tau} \perp (R_{\tau,i} : i \in [k]) \mid X_{\tau}.
	\]  
    Conditioned on the context, the chosen arm is independent of potential reward outcomes. But context-unaware TS performs inferences without conditioning on contexts, and due to the co-occuring patterns in the contexts sequence and the sequence of chosen arms
	\[ 
	I_{\tau}  \not\perp  (R_{\tau,i} : i \in [k]).
	\]
\end{remark}

We now prove Lemma \ref{lem:context_unaware_ts}.
\begin{proof}[Proof of Lemma \ref{lem:context_unaware_ts}.]
	It is not hard to show that $\E[\Delta_{\rm post}]>0$ for any fixed $T$. To show the result, then, it is without loss of generality to assume $T\geq 4$. Let $\Theta'$ denote the set of parameter vectors satisfying the following properties: 
	\begin{enumerate}
		\item $\frac{\theta^{(2)}_{1}+\theta^{(2)}_{2}}{2} > \frac{\theta^{(1)}_{1}+\theta^{(1)}_{2}}{2}$:  This implies that the optimal arm is $I^*=2$ 
		\item  $\min\left\{  \theta^{(1)}_{1}, \theta^{(1)}_{2}\right\} > \theta^{(2)}_{1}$: This implies arm 1 appears to be the best if arm 2 is  only measured in the context $e_1$. 	
		\item $\theta^{(1)}_{1}<0$: This implies that if arm 1 is sampled in the first period, arm 2 has at least a $\frac{1}{2}$ chance of being sampled in the second period. 
	\end{enumerate}
	In the first period, let $\nu_{1,1}\sim N\left( \tilde{m}_{1,1}, \tilde{s}_{1,1}^2\right)$ and $\nu_{1,2}\sim N\left( \tilde{m}_{1,2}, \tilde{s}_{1,2}^2\right)$ denote the sampled parameters, and we denote the probability of playing arm 1 by
	\[
	c_0 \triangleq \Prob( \nu_{1,1}> \nu_{1,2}) > 0.
	\]
	Conditioned on the event that $\theta\in \Theta'$ and $I_1 = 1$, we have $\left(\tilde{m}_{2,1},\tilde{s}_{2,1}^2 \right)= \left(\theta^{(1)}_{1},0\right)$ and the probability of playing arm 2 in the second period is
	\[
	\Prob\left(\nu_{2,2} > \theta^{(1)}_{1}  \mid \theta \in \Theta', I_1=1  \right) \geq \frac{1}{2}
	\]
	where the inequality holds due to Condition 3 above. Conditioned on the event that $\theta\in \Theta'$, $I_1 = 1$ and $I_2 = 2$, we have $\left(\tilde{m}_{3,1},\tilde{s}_{3,1}^2\right) = \left(\theta^{(1)}_{1},0\right)$ and $\left(\tilde{m}_{3,2}, \tilde{s}_{3,2}^2\right) = \left(\theta^{(2)}_{1},0\right)$. Due to Condition 2, TS will always measure arm 1 afterwards. Hence,
	\begin{align*}
		\E[\Delta_{\rm post}] 
		&\geq \E\left[\left(\frac{\theta^{(2)}_{1} + \theta^{(2)}_{2}}{2} - \frac{\theta^{(1)}_{1} + \theta^{(1)}_{2}}{2}\right) \ind(\theta\in\Theta')\ind(I_1=1,I_2=2)\right]\\
		&\geq \frac{c_0}{2}\E\left[\left(\frac{\theta^{(2)}_{1} + \theta^{(2)}_{2}}{2} - \frac{\theta^{(1)}_{1} + \theta^{(1)}_{2}}{2}\right) \ind(\theta\in\Theta')\right] > 0.
	\end{align*}
	This completes the proof.
\end{proof}

\subsection{Failure of contextual bandit algorithms}
The goal in our formulation is to select one among a very restricted set of decision-rules: those that choose a common action, irrespective of context. Experimentation should be tailored to this objective. Here, we give insight into potential failures when an exploration algorithm is designed with a different learning target in mind. Consider the following example. There are three actions, and the decision-maker would like to identify the best action to employ on average, across all contexts. Imagine that the context set describes two customer segments. Action 1 appeals to one segment, but is highly unappealing to the other. For action 2, the situation is reversed. Action 3 is not ideal for either segment, but is also not disliked by either. When personalization is inappropriate or costly, action 3 may be the preferred communal option. 

The next example does not align with our formulation, because we take the prior distribution to be non-Gaussian. Similar issues can arise with a Gaussian prior, but its unbounded nature always allows for a nonzero -- even if very small -- chance that the mainstream action is better even for a specific segment. We omit analytical calculations of this more intricate case, since Example \ref{ex:less_personalized} seems already to capture the main intuition.  

\begin{example}[A mainstream action]\label{ex:less_personalized}
	Consider a problem with $k=3$ arms and context set $\Xc = \{e_1, e_2 \}\subset \mathbb{R}^2$. The population distribution $\Dc_{\rm pop}$ is uniform over $\Xc$ and $(X_t)_{t\in \mathbb{N}}$ are drawn i.i.d. from $\Dc_{\rm pop}$. For the first two arms ($i\in [2]$) and $x\in[2]$, it holds almost surely that 
	\[ 
	\theta^{(i)}_x = \ind(x=i) - \ind(x\neq i).  
	\]
	The third arm ($i=3$), is insensitive to context, with $\theta_1^{(3)} = \theta_2^{(3)} =U$ where $U\sim {\rm Uniform}[0,1]$. Rewards are noiseless, with $R_{t,I_t}= r_{\theta}(I_t, X_t)=\langle \theta^{(I_t)},X_t\rangle$. Hence, if $X_t = e_1$, then $R_{t,I_t} = \theta^{(I_t)}_1$; otherwise $R_{t,I_t} = \theta^{(I_t)}_2$. Observations are not subject to delay (i.e. $L=1$).
\end{example}
The next lemma formalizes that contextual Thompson sampling, which selects an action according to the posterior probability it is the optimal action for the current context, has simple regret that does not vanish even as the horizon grows. The same result applies to appropriate contextual versions of UCB. The simple reason is that action 3 is never sampled, because it does not maximize the reward in either context. This means no information about $\theta^{(3)}$ is gathered and the decision-maker cannot determine whether action 3 is the best arm to select. If the goal is to identify the best policy within a restricted class (i.e. those that select the same arm, irrespective of context), the exploration algorithm needs to be designed so that it gathers the right information for this task. The proof follows from this argument and is omitted for brevity. At a high level, contextual TS fails here because it does not reflect the true decision-objective. 

\begin{lemma}[Failure of contextual TS]
\label{lem:contextual_ts}
Consider Example \ref{ex:less_personalized}. 
	Contextual TS at time $t$ chooses an arm $I_t$ such that for each $i\in [k]$, 
	$
	\Prob(I_t = i \mid H_t,X_t) =  \Prob\left( \argmax_{j\in[k]} r_{\theta}(j, X_t ) =i \mid H_t, X_t  \right)
	$.
	There is an absolute numerical constant $c>0$ such that for all $T\in \mathbb{N}$, $	\E\left[ \Delta_{\rm post}  \right] \geq c.$ 
\end{lemma}

\section{Proof of Theorem \ref{thm:main_result}}

We begin by restating the theorem. 
\utilitarianRegret*

The proof is broken into several parts. 

\subsection{Proof of \eqref{eq:post_exp_regret_bound}}
The more delicate result is \eqref{eq:within_exp_regret_bound}, with \eqref{eq:post_exp_regret_bound} following essentially as a corollary.  Notice that the right hand side of \eqref{eq:post_exp_regret_bound} matches the right hand side of \eqref{eq:within_exp_regret_bound} if we could set $t=T+L$. 
\begin{proof}[Argument deriving \eqref{eq:post_exp_regret_bound} from \eqref{eq:within_exp_regret_bound}]
	Recall that DTS's decision at time $t$ does not depend on the context at time $t$ or even contexts in the past $L-1$ periods (see \eqref{eq:context_indep_probs}).
Let $\tilde{H}_t \triangleq (X_{1:(t-L)}, I_{1:(t-L)}, R_{1:(t-L)})$ be the effective history used by DTS at time $t\in [T]$ (where $\tilde{H}_t=\emptyset$ for $t\in [L]$).
With some abuse of notation, extend the definition of $\tilde{H}_t$ for $t\in \{T+1, \ldots, T+L\}$ as $\tilde{H}_t = (X_{1:(t-L)}, I_{1:(t-L)},  R_{1:(t-L)} )$ 
Recall that for all  $t\in [T]$, the definition of DTS is that $\Prob\left(I_t = i \mid \tilde{H}_t\right) = \Prob\left(I^* = i \mid \tilde{H}_t\right)$. Extend this definition for $t\in \{T+1, \ldots, T+L\}$.

Define the greedy decision based on $\tilde{H}_t$ at time $t\in[T+L]$ by 
	\[ 
	\tilde{I}^*_t \in \argmax_{i \in [k]} \E\left[r_{\theta}(i) \mid \tilde{H}_t \right].
	\]
	Recall that $I_{\rm post}$ is the arm chosen by DTS at the end of the experiment. Since $H_{\rm post} = (X_{1:T}, I_{1:T}, R_{1:T}) = \tilde{H}_{T+L}$, we have $I_{\rm post} = \tilde{I}_{T+L}^*$

	Now, for $t\in [T+L]$, define the two performance measures 
	\[
	\Delta_t^{\rm explore}= r_{\theta}(I^*) - r_{\theta}(I_t)  \quad\text{and}\quad  \Delta_{t}^{\rm greedy} = r_{\theta}(I^*) - r_{\theta}\left(\tilde{I}^*_t\right).
	\]
 	The expected regret of the greedy decision is always smaller:
 	\begin{align}\label{eq:greedy_is_better}
 		\E\left[\Delta_{t}^{\rm greedy}\right] =  \E\left[ \E\left[ r_{\theta}(I^*) - r_{\theta}\left(\tilde{I}^*_t\right) | \tilde{H}_t \right] \right] \leq   \E\left[ \E\left[ r_{\theta}(I^*) - r_{\theta}(I_t) | \tilde{H}_t \right] \right] =  \E\left[\Delta_{t}^{\rm explore}\right].
 	\end{align}
  Here to simplify this argument, assume $X_{1:T}$ is an arbitrary deterministic sequence, equal to some fixed $x_{1:T}$ almost surely. Since it is deterministic, we do not need to condition on it in expectations. 
	 	
	A careful reading of the proof\footnote{It is not proper style to cite a proof rather than a result. In this case, the modification is quite simple, though. Follow the exact same steps, but interpret $t$ as possibly falling in the range $\{T+1, \ldots, T+L \}$. } of \eqref{eq:within_exp_regret_bound} reveals that it applies to bound 
	\[
	\E\left[\Delta_{t}^{\rm explore}\right] \leq  \sqrt{ \frac{ 2\cdot \iota \cdot k \cdot \log(k) }{ {\rm Precision}\left(x_{1 : (t-L)}\right) } }, \label{eq:tmp_inequality}
	\]
	even for $t \in \{T+1, \ldots, T+L\}$.  
 Note that that expected within-experiment regret is $\E[\Delta_t] = \E\left[ \Delta_t^{\rm explore}\right]$ and expected post-experiment regret is $\E[\Delta_{\rm post}]= \E\left[\Delta_{T+L}^{\rm greedy}\right]$. 
 Then picking $t=T+L$ and combining this with \eqref{eq:greedy_is_better} yields  \eqref{eq:post_exp_regret_bound}. 
\end{proof}

\subsection{Proof of Proposition \ref{prop:regret_to_estimation}}\label{subsec:proof_of_regret_to_estimation}
The first key to establishing Theorem \ref{thm:main_result} is Proposition \ref{prop:regret_to_estimation}, restated below. This reduces the problem of controlling the utilitarian regret to the problem of controlling the expected posterior variance of the optimal arm.  Recall that $s_{t,i} = \sqrt{{\rm Var}(r_{\theta}(i) \mid H_t)}$. 
\regretFromVariance*
\begin{proof}
In this proof, we avoid writing conditional expectations by letting $X_{1:T} = x_{1:T} \in \mathcal{X}^T$ with probability 1 for some arbitrary sequence $x_{1:T}$. 

We focus on proving the bound on $\E[\Delta_t]$. Define $Z_i = r_{\theta}(i)$ to be the uncertain population performance of arm $i$,  $m_{t,i}=\E\left[ Z_i \mid H_t \right]$ to be its posterior mean,  and $s_{t,i}^2={\rm Var}\left(  Z_i \mid H_t\right)$ to be its posterior variance. The notation $Z_i$ and $m_{t,i}$ is used only in this proof. Note that $Z_{i}\mid H_t \sim N\left(m_{t,i},s_{t,i}^2 \right)$. Take $Z = (Z_1, \ldots, Z_k)$ to be the vector. Under DTS, $I_t$ is a sample from the posterior, i.e. $\Prob(I_t=i \mid H_t) = \Prob(I^*=i \mid H_t)$ but $I_t$ is independent of $(Z_1,\ldots, Z_k)$ conditioned on $H_t$. We let $\I_{H_t}(Y_1; Y_2)$ denote the mutual information between random variables $Y_1$ and $Y_2$ under the distribution $\Prob\left((Y_1, Y_2)\in \cdot \mid H_t \right)$. This is random, due to its dependence on the history. Taking expectations yields the usual definition of mutual information, with $\E\left[\I_{H_t}\left( Y_1; Y_2 \right) \right] = \I(Y_1; Y_2 \mid H_t)$. This notation is used in this proof alone. 
	
 	We have, 
	\begin{align*}
		\E\left[ \Delta_{t}\right] = \E \left[ Z_{I^*} - Z_{ I_t} \right] &= \E \left[Z_{I^*}  - \E\left[ Z_{ I_t } \mid H_t  \right]  \right] \\ 
		&=\E \left[Z_{I^*}  - \E\left[ m_{t, I_t} \mid H_t  \right]  \right]\\
		&\overset{(a)}{=}\E \left[Z_{I^*}  - \E\left[ m_{t, I^*} \mid H_t  \right]  \right]\\
		&= \E \left[ \E \left[Z_{I^*}  - m_{t, I^*}  \mid H_t   \right]\right] \\
		&\overset{(b)}{\leq} \E \left[ \sqrt{  \E\left[ s^2_{t,I^*}  \mid H_t\right] }  \sqrt{ 2 \I_{H_t}\left( I^* ; Z \right)  } \right] \\
		&\overset{(c)}{\leq} \sqrt{ \E \left[  \E\left[ s^2_{t,I^*}  \mid H_t\right] \right] }  \sqrt{2\E\left[ \I_{H_t}\left( I^* ; Z \right)  \right]} \\
		&\overset{(d)}{=} \sqrt{  \E\left[ s^2_{t,I^*} \right] }  \sqrt{2\I\left( I^* ; Z  \mid H_t\right)} \\
		&\leq \sqrt{  \E\left[ s^2_{t,I^*} \right] }  \sqrt{2\ent\left( I^* \mid H_t\right)}.
	\end{align*}
	Early steps of the proof use the tower property of conditional expectation. Step $(a)$ is crucial and uses that fact that $I_t$ and $I^*$ have the same distribution conditioned on $H_t$ and that the vector $(m_{t,1},\ldots, m_{t,k})$ is nonrandom conditioned on $H_t$ (formally is measurable with respect to the sigma-algebra $H_t$ generates). Step $(b)$ applies Proposition 8 of \cite{russo2019much}, which is stated below. Step $(c)$ applies the H\"{o}lder inequality, step $(d)$ applies the tower property of conditional expectation, and the final step uses that entropy bounds mutual information. The proposition uses the coarse upper bound $\ent\left( I^* \mid H_t\right)\leq \log(k)$ to simplify the presentation.  The bound on $\E[\Delta_t \mid H_t]$ follows from $(c)$ using that $\I_{H_t}\left( I^* ; Z \right) \leq \log(k)$. 
\end{proof}
\begin{lemma}[Proposition 8 of \cite{russo2019much}]
	Consider a random vector $Z \in \mathbb{R}^n$  and a random index $I \in [n]$. Suppose that for each $i\in[n]$, $Z_i$ has mean $\mu_i$ and the distribution of $Z_i - \mu_i$ is sub-Gaussian with variance proxy $\sigma_i^2$. Then 
	\[
	\left|\E\left[Z_I - \mu_I \right]\right| \leq \sqrt{ \E\left[ \sigma_I^2 \right]  } \sqrt{2 \I\left( Z; I \right) }.
	\]
\end{lemma}
In the setting of the above lemma, a standard sub-Gaussian maximal inequality would bound the largest deviation of $Z_i$ from its mean as $\E\left[\max_{i\in [k]}| Z_i - \mu_i| \right] \leq (\max_{i\in[k] } \sigma_i) \sqrt{2 \log(n)}$. For our purposes,  the lemma offers a critical improvement because it depends only on the variance at the likely realizations of $I$. A second improvement, which is the focus of the discussion in \cite{russo2019much}, is that the mutual information term $\I\left( Z; I \right)$ could be much smaller than $\log(n)$.

\subsection{Optionally sampled matrix-valued processes}
One part of our proof (namely, Lemma \ref{lem:matrix_skipping_corollary}) relies on a  new result on optionally sampled matrix-valued processes. 
Stated in the abstract form below, one can view the positive definite matrix $V_{\ell}$ as generalized `reward' or `value' and $Z_{\ell}$ as a (randomized) decision of whether to collect that value. 
The result bounds the impact of randomization on the reward accrued.

Let $\mathbb{S}^d$ denote the set of symmetric $d\times d$ square matrices and $\mathbb{S}^d_+\subset \mathbb{S}^d$ denote the set of symmetric positive semidefinite matrices. 
\begin{restatable}[Optionally sampled matrix-valued  process]{proposition}{MatrixSkipping}\label{prop:matrix_concentration} 
	Consider a deterministic sequence of positive semidefinite matrices $V_1, V_2, \ldots \in \mathbb{S}^d_{+}$ satisfying $\sup_{t\in\mathbb{N}} \lambda_{\max}(V_t) \leq 1$,
	and a  random process $(Z_{t})_{t \in \mathbb{N}}$ taking values in $\{0,1\}$ that is adapted to some filtration $(\mathcal{F}_t)_{t\in \mathbb{N}_0 }$. 
	For $n\in \mathbb{N}$, define 
	\[ 
	S_{n} = \sum_{t=1}^{n} Z_t V_t  \quad  \text{and} \quad \tilde{S}_{n} = \sum_{t=1}^{n}\Prob(Z_t =1 \mid \mathcal{F}_{t-1}) V_t. 
	\] 
	Then, for any $\delta>0$, with probability exceeding $1-\delta$, 
	\[
	S_{n}  \succeq \underbrace{\left(3-e  \right)}_{\approx 0.28}  \tilde{S}_{n} -\log\left(\frac{d}{\delta}\right) I, \quad \forall n\in \mathbb{N}.
	\]
\end{restatable}
\begin{proof}
	See Section \ref{sec:app_matrix_optional_sampling} for a complete proof.  
	The analysis builds on\footnote{Direct application of that paper establishes a scalar inequality of the form
		$\lambda_{\max}\left( \tilde{S}_{n} - S_{n}  \right)  \leq c \lambda_{\max}\left( \tilde{S}_{n} \right) + \log\left(\frac{d}{\delta}\right)$.
		For our purposes Proposition \ref{prop:matrix_concentration} offers a critical improvement. It is able to provide a meaningful bound on $u^\top S_n u$ even for  directions $u\in \mathbb{R}^d$ for which $u^\top \tilde{S}_n u$ is extremely small.} 
	a beautiful theory of the concentration of matrix-valued martingales by \cite{tropp2011freedman}.
\end{proof}

\subsection{Introducing a smoothed observation model}\label{subsec:smoothed_model}
Our goal is to establish a regret bound by bounding $\E\left[s^2_{t,I^*}\right]$. 
As a first step toward this, we introduce a `smoothed' observation model as a device in the analysis. 
In this model, arms can be played fractionally; When the algorithm picks an effort allocation $(p_{t,1}, \ldots, p_{t,k})$, it observes in response noisy reward signals $\left(\tilde{R}_{t,1}, \ldots, \tilde{R}_{t,k}\right)$ where the standard deviation of $\tilde{R}_{t,i}$ is $\frac{\sigma}{p_{t,i}}$. 
The notation $\tilde{\Sigma}_t$, $\tilde{\Sigma}_{t,i}$ and $\tilde{s}^2_{t,i}$ is used to denote posterior (co)variances under this smoothed model. 

\begin{definition}[Smoothed observation model]   Define $\tilde{R}_{t,i} = r_{\theta}(i, X_t) + \frac{{W}_{t,i}}{p_{t,i}}$ so that $\tilde{R}_{t,i} \mid (H_t, p_{t}, \theta, X_t) \sim N\left( r_{\theta}(i, X_t) \, , \, \frac{\sigma^2}{p_{t,i}^2} \right)$. 
Set
	\begin{align*}
		\tilde{\Sigma}_t &= {\rm Cov}\left[\theta \mid  \left( X_{\ell}, (p_{\ell,j} , \tilde{R}_{\ell,j})_{j \in [k]}    \right)_{\ell\in [t-L]} \right], \\
		\tilde{\Sigma}_{t,i} &= {\rm Cov}\left[\theta^{(i)} \mid  \left( X_{\ell}, (p_{\ell,j}, \tilde{R}_{\ell,j})_{j \in [k]}    \right)_{\ell\in [t-L]} \right], \\
		\tilde{s}^2_{t,i} &= {\rm Var}\left[ \langle \theta^{(i)}, x_{\rm pop}\rangle \mid \left( X_{\ell}, (p_{\ell,j} , \tilde{R}_{\ell,j})_{j \in [k]}    \right)_{\ell\in [t-L]}  \right].
	\end{align*} 
\end{definition}
(\emph{As a warning, the notation $\tilde{s}_{t,i}$ means something different in Subsection \ref{subsec:unaware_algo}, where it is used to define a heuristic algorithm.} )
Posterior variances under the smoothed model are known functions of the chosen arm propensities and the context sequence. The next fact illustrates this for $\tilde{\Sigma}_t$. Define $\phi(x,i) = (0, \ldots, 0, x_1,\ldots, x_d, 0, \ldots,0) \in \mathbb{R}^{dk}$ to be the concatenation of $k$ subvectors of size $d$, where the $i$th subvector is $x$. Other quantities of interest can be derived from $\tilde{\Sigma}_t$. For instance, $\tilde{s}^2_{t,i} = \phi(x_{\rm pop}, i)^\top \tilde{\Sigma}_t  \phi(x_{\rm pop}, i)$. 
\begin{fact}\label{fact:smoothed_covariance}
	 $\tilde{\Sigma}_t$ obeys the formula
	\[
	\tilde{\Sigma}_t = \left(  \Sigma_1^{-1 } + \sigma^{-2} \sum_{\ell=1}^{t-L} \sum_{i=1}^{k} p_{\ell,i} \phi(X_\ell,i) \phi(X_\ell,i)^\top \right)^{-1}
	\]
	whereas 
	\[
	\Sigma_t = \left(  \Sigma_1^{-1 } + \sigma^{-2} \sum_{\ell=1}^{t-L} \phi(X_\ell,I_\ell) \phi(X_\ell, I_{\ell})^\top \right)^{-1}  =  \left(  \Sigma_1^{-1 } + \sigma^{-2} \sum_{\ell=1}^{t-L} \sum_{i=1}^{k}     \ind(I_{\ell}=i) \phi(X_\ell,i) \phi(X_\ell,i)^\top \right)^{-1}.
	\]
\end{fact}

The next result allows us to rigorously use the evolution of the posterior variance in the smoothed model to study the evolution of posterior covariance in the true model. It follows by applying Proposition \ref{prop:matrix_concentration} to our problem. 
\begin{lemma}\label{lem:matrix_skipping_corollary}
	For any $\delta>0$,
	\[
	\Prob\left(   \Sigma_{t}^{-1}  \succeq    (3-e) \tilde{\Sigma}_t^{-1}    - \sigma^{-2}\log\left(\frac{dk}{\delta}\right)I \quad \forall t \geq L   \mid \theta, X_{1:T} \right) \geq 1-\delta.  
	\]
\end{lemma}
\begin{proof}[Proof of Lemma \ref{lem:matrix_skipping_corollary}]
	Most of the analysis uses precision matrices, rather than covariance matrices. 
	Write the posterior precision matrix $\Sigma_t^{-1}$ in the form. 
        \[
	\Sigma_t^{-1} =  \Sigma_1^{-1} + \sigma^{-2}\begin{pmatrix} 
		\sum_{\ell=1}^{t-L} \ind\{I_\ell =1\} X_\ell X_\ell^\top & \dots  & 0\\
		\vdots & \ddots & \vdots\\
		0 & \dots  & \sum_{\ell=1}^{t-L} \ind\{I_\ell =k\} X_\ell X_\ell^\top 
	\end{pmatrix}  \triangleq  \Sigma_1^{-1} + \sigma^{-2} S_t. 
	\]
	The  posterior precision matrix in the smoothed observation model  $\tilde{\Sigma}^{-1}_t \in  \mathbb{R}^{dk\times dk}$ is defined by 
	\[
	\tilde{\Sigma}_t^{-1} =  \Sigma_1^{-1} + \sigma^{-2}\begin{pmatrix} 
		\sum_{\ell=1}^{t-L}  p_{\ell,1} X_\ell X_\ell^\top & \dots  & 0\\
		\vdots & \ddots & \vdots\\
		0 & \dots  & \sum_{\ell=1}^{t-L}  p_{\ell,k} X_\ell X_\ell^\top 
	\end{pmatrix}    \triangleq  \Sigma_1^{-1} + \sigma^{-2} \tilde{S}_t.
	\]
	For a fixed $i\in[k]$, we apply Proposition \ref{prop:matrix_concentration} with $V_\ell = X_{\ell} X_{\ell}^\top$ and $Z_{\ell}=\ind(I_\ell=i)$, and $\mathcal{F}_{\ell-1}  =\sigma(H_\ell)$ taken to be the sigma algebra generated by the history. Proposition \ref{prop:matrix_concentration}  applies, since  $\lambda_{\max}(V_\ell) =  \| X_{\ell} \|_2^2  \leq 1$ where the first equality is Fact \ref{fact:rank-one matrix} and the norm bound is an assumption in the problem formulation.  Observe that $p_{\ell,i} = \Prob(Z_{\ell}=i \mid \mathcal{F}_{\ell-1})$. Hence, for any $\delta'>0$, with probability exceeding $1-\delta'$, 
	\[
	\sum_{\ell=1}^{t-L} \ind\{I_\ell =i\} X_\ell X_\ell^\top \succeq (3-e) \left(\sum_{\ell=1}^{t-L} p_{\ell,i} X_\ell X_\ell^\top\right)  - \log\left(\frac{d}{\delta'}\right) I, \quad\forall t\geq L.
	\]    
 Taking $\delta = \frac{\delta'}{k}$ and applying a union bound, we have that with probability exceeding $1-\delta$,
	\begin{equation*}
		S_t  \succeq (3-e) \tilde{S}_t    - \log\left(\frac{dk}{\delta}\right) I, \quad \forall t \geq L,
	\end{equation*}
 where $S_t$ and $\tilde{S}_t$ are defined earlier in this proof.
	Then on the event that $S_t  \succeq (3-e) \tilde{S}_t    - \log\left(\frac{dk}{\delta}\right) I $, we have  
	\begin{align*}
		\Sigma_{t}^{-1} =  \Sigma_{1}^{-1} + \sigma^{-2} S_t   &\succeq  \Sigma_{1}^{-1} + (3-e) \sigma^{-2}\tilde{S}_t    - \sigma^{-2}\log\left(\frac{dk}{\delta}\right) I  \\
  &= \Sigma_{1}^{-1} + (3-e) \left(\tilde{\Sigma}_t^{-1} - \Sigma_{1}^{-1}\right)   -  \sigma^{-2}\log\left(\frac{dk}{\delta}\right)I\\
  &= (e-2)\Sigma_{1}^{-1} + (3-e) \tilde{\Sigma}_t^{-1} - \sigma^{-2}\log\left(\frac{dk}{\delta}\right)I\\
		& \succeq    (3-e) \tilde{\Sigma}_t^{-1}    - \sigma^{-2}\log\left(\frac{dk}{\delta}\right)I.
	\end{align*}
This completes the proof.
\end{proof}

An adaptation of the high probability bound above a bound in expectation, here stated in terms of the scalar quantities $s^2_{t,i}$ and $ \tilde{s}^2_{t,i}$ that are needed in the analysis. The proof follows by a messy calculation. 
\begin{corollary}\label{corr:concentration_ineq_scalar} For any $t\geq L$, 
	\[
	\E\left[  s^2_{t,i} \mid \theta,  X_{1:T} \right] \leq \iota \cdot \E\left[ \tilde{s}^2_{t,i} \mid \theta,  X_{1:T} \right], \quad\forall i\in [k].
	\]
In particular,
	\[
	\E\left[  s^2_{t,I^*} \mid \theta, X_{1:T} \right] \leq \iota \cdot \E\left[ \tilde{s}^2_{t,I^*} \mid \theta, X_{1:T} \right]. 
	\]
\end{corollary}
\begin{proof}
		The first step is to prove an inequality of the form $\Sigma_{t} \preceq c_{\delta} \tilde{\Sigma}_t$, which holds with high probability. In particular we prove that  for any $\delta>0$, conditioned on $X_{1:T}$ and $\theta$, with probability exceeding $1-\delta$, the following inequality holds simultaneously for every $t \geq L$: 
		\begin{equation}\label{eq:multiplicative_ratio}
			\Sigma_{t} \preceq c_\delta  \tilde{\Sigma }_{t} \quad \text{where} \quad c_\delta = 8\cdot\max\left\{ \sigma^{-2} \cdot \lambda_{\max}(\Sigma_1)\cdot \log\left(\frac{dk}{\delta}\right)   \, , \, 1  \right\}.
		\end{equation}		
		We know that $\Sigma_{t}^{-1} \succeq \Sigma_{1}^{-1}$. Combining this with Lemma \ref{lem:matrix_skipping_corollary} implies that for any arbitrary unit vector $u$, 
		\[
		u^\top \Sigma_{t}^{-1} u  \geq  \max\left\{ \lambda_{\min}\left(\Sigma_1^{-1}\right),   \left(3-e  \right)  u^\top \tilde{\Sigma }_{t}^{-1} u   - \sigma^{-2}\log\left(\frac{dk}{\delta}\right) \right\}. 
		\]
		If $ \left(3-e  \right)  u^\top \tilde{ \Sigma }_{t}^{-1} u  \geq  2\sigma^{-2}\log\left(\frac{dk}{\delta}\right)$, we have $u^\top \Sigma_{t}^{-1} u   \geq \frac{3-e}{2} u^\top \tilde{ \Sigma}_{t}^{-1} u .$
		On the other hand, if $\left(3-e  \right)  u^\top \tilde{\Sigma }_{t}^{-1} u  \leq  2\sigma^{-2}\log\left(\frac{dk}{\delta}\right)$, we have   
		\[
		u^\top \Sigma_{t}^{-1} u  \geq   \lambda_{\min}\left(\Sigma_1^{-1}\right) = \frac{\lambda_{\min}\left(\Sigma_1^{-1}\right) }{ \sigma^{-2}\log\left(\frac{dk}{\delta}\right) } \cdot  \sigma^{-2}\log\left(\frac{dk}{\delta}\right) \geq   \frac{\lambda_{\min}\left(\Sigma_1^{-1}\right) }{ \sigma^{-2}\log\left(\frac{dk}{\delta}\right) } \cdot \frac{3-e}{2} \cdot  u^\top \tilde{ \Sigma }_{t}^{-1} u .
		\] 
		In either case we have that for an arbitrary unit vector $u$, 
		\[
		u^\top \Sigma_{t}^{-1} u  \geq c_1   u^\top \tilde{ \Sigma }_{t}^{-1} u  \quad \text{where} \quad  c_1= \min\left\{ \frac{\lambda_{\min}\left(\Sigma_1^{-1}\right) }{ \sigma^{-2}\log\left(\frac{dk}{\delta}\right) }  \, , \, 1 \right\} \cdot \frac{3-e}{2}.
		\]
		We can simplify the expression using that $\frac{2}{3-e} <8$. Viewing this as a relation of the form $\Sigma_{t}^{-1} \preceq \frac{1}{c_1}\tilde{\Sigma}_t^{-1} \preceq c_\delta \tilde{\Sigma}_t^{-1}$ yields the claim \eqref{eq:multiplicative_ratio}.
		
		Let $\chi_{\delta}$ be the event that \eqref{eq:multiplicative_ratio} holds for all $t\geq L$. We also have the almost sure bounds, $\Sigma_{t} \preceq \Sigma_{1}$ and $\tilde{\Sigma}_{t}^{-1} \preceq \Sigma_{1}^{-1} + \sigma^{-2}(t-L) I$, (which holds since $\lambda_{ \max }(X_\ell X_\ell^\top)= \|X_\ell\|_2^2 \leq 1$ by Fact \ref{fact:rank-one matrix}).  We have that for every $\delta>0$
		\begin{align*}
			\E\left[ \Sigma_{t} \mid \theta, X_{1:T} \right] &=   c_{\delta }\E\left[\tilde{ \Sigma}_{t} \chi_{\delta}  \mid  \theta, X_{1:T}\right] +\E\left[ \Sigma_{t} (1-\chi_{\delta}) \mid  \theta, X_{1:T}\right]  \\
			&\preceq  c_{\delta } \E\left[\tilde{ \Sigma}_{t} \chi_{\delta} \mid  \theta, X_{1:T} \right] + \E\left[ \Sigma_1 (1-\chi_{\delta}) \mid  \theta, X_{1:T} \right] \\
			&\preceq c_{\delta } \E\left[\tilde{ \Sigma}_{t} \mid  \theta, X_{1:T}\right] + \delta \Sigma_1 \\
			&\preceq \E\left[\tilde{ \Sigma}_{t} \mid  \theta, X_{1:T} \right] \left(  c_{\delta }+ \delta \frac{ \lambda_{\max}\left( \Sigma_1\right)}{  \lambda_{\min}\left( \E\left[ \tilde{ \Sigma}_{t} \mid  \theta, X_{1:T} \right] \right)  }  \right) \\
			&= \E\left[\tilde{ \Sigma}_{t} \mid  \theta, X_{1:T} \right] \left( c_{\delta }+ \delta \lambda_{\max}\left( \Sigma_1\right) \lambda_{ \max }\left( \left(\E\left[ \tilde{ \Sigma}_{t} \mid  \theta, X_{1:T} \right]\right)^{-1} \right)  \right) \\
			&\preceq \E\left[\tilde{ \Sigma}_{t} \mid  \theta, X_{1:T} \right] \left(  c_{\delta }+ \delta \lambda_{\max}\left( \Sigma_1\right) \lambda_{ \max }\left( \Sigma_{1}^{-1} + \sigma^{-2}(t-L) I   \right)  \right) \\
			&=   \E\left[\tilde{ \Sigma}_{t} \mid  \theta, X_{1:T} \right] \left(  c_{\delta }+ \delta \lambda_{\max}\left( \Sigma_1\right) \left[\lambda_{ \max }\left( \Sigma_{1}^{-1} \right) + \sigma^{-2}(t-L)\right] \right).
		\end{align*}
		Hence, 
		\[
		\E\left[ \Sigma_{t} \mid  \theta, X_{1:T} \right] \preceq (c_{\delta^*} +1) \E\left[\tilde{ \Sigma}_{t} \mid  \theta, X_{1:T} \right] \quad \text{where}\quad \delta^*= \left(\lambda_{\max}\left( \Sigma_1\right) \left[\lambda_{ \max }\left( \Sigma_{1}^{-1} \right) + \sigma^{-2}(t-L)\right]\right)^{-1}.
		\]
		Now, 
		\begin{align*}
			c_{\delta^*}+1 &= 8\cdot \max\left\{\sigma^{-2}\cdot \lambda_{\max}(\Sigma_1)\cdot \log\left(\frac{dk}{\delta^*}\right)   \, , \, 1  \right\} +1 \\
			&=    \max\left\{8\sigma^{-2} \cdot \lambda_{\max}(\Sigma_1) \cdot \log\left( d k \lambda_{\max}\left( \Sigma_1\right) \left[\lambda_{ \max }\left( \Sigma_{1}^{-1} \right) + \sigma^{-2}(t-L)\right]\right)  +1 \, , \, 9 \right\} \\
			&\leq \max\left\{8\sigma^{-2} \cdot \lambda_{\max}(\Sigma_1) \cdot \log\left( d k \lambda_{\max}\left( \Sigma_1\right) \left[\lambda_{ \max }\left( \Sigma_{1}^{-1} \right) + \sigma^{-2}T\right]\right)  + 1 \, , \, 9  \right\} \\
			&\triangleq \iota.
		\end{align*}
		
		Fix $i\in[k]$ and use again the notation $\phi(x,i) = (0, \ldots, 0, x, 0, \ldots,0) \in \mathbb{R}^{dk}$ to be the concatenation of $k$ subvectors of size $d$, where the $i$-th subvector is $x$. Then $\tilde{s}^2_{t,i} = \phi(x_{\rm pop}, i)^\top \tilde{\Sigma}_t  \phi(x_{\rm pop}, i)$ and $s^2_{t,i} = \phi(x_{\rm pop}, i)^\top \Sigma_t  \phi(x_{\rm pop}, i)$. 
  We have 
		\begin{align*}
		 \E\left[s^2_{t,i} \mid  \theta, X_{1:T} \right] &= \E\left[\phi(x_{\rm pop},i)^\top \Sigma_{t}\phi(x_{\rm pop},i)   \mid  \theta, X_{1:T} \right] \\
		 &= \phi(x_{\rm pop},i)^\top \E\left[\Sigma_{t}  \mid  \theta, X_{1:T} \right] \phi(x_{\rm pop},i) \\
		 &\leq \iota \cdot \phi(x_{\rm pop},i)^\top \E\left[\tilde{\Sigma}_{t}  \mid  \theta, X_{1:T} \right] \phi(x_{\rm pop},i)\\ 
		 &= \iota \cdot \E\left[\phi(x_{\rm pop},i)^\top \tilde{\Sigma}_{t}\phi(x_{\rm pop},i)   \mid  \theta, X_{1:T} \right] \\
		 &= \iota \cdot \E\left[\tilde{s}^2_{t,i} \mid  \theta, X_{1:T} \right].
		\end{align*}	
 Since $I^*$ is non-random conditioned on $\theta$, we also have
 \[
\E\left[  s^2_{t,I^*} \mid \theta, X_{1:T} \right] \leq \iota \cdot \E\left[ \tilde{s}^2_{t,I^*} \mid \theta, X_{1:T} \right]. 
\]

\end{proof}

\subsection{Bounding the posterior precision by \emph{attainable} precision}\label{subsec:bounding_precision}

We prove the following result, which applies to DTS, since the condition belw holds under DTS, as shown in Proposition \ref{prop:prob_score_of_DTS}.
\begin{restatable}[]{proposition}{DTSVarianceBound}\label{prop:DTS_varaince_bound} If  $\E\left[  \frac{\ind(I^*=i)}{p_{t,i}}  \mid X_{1: T }  \right] \leq 1$ for each $t \in [T]$ and $i\in[k]$, then for any $t\geq L$,
	\[ 
	\E\left[ s_{t,I^*}^2 \mid X_{1:T} \right] \leq   \iota \cdot  \frac{k}{ {\rm Precision}\left(X_{1:(t-L)}\right)}.  
	\]
\end{restatable}

Taking expectations of the inequality for $I^*$ in Corollary \ref{corr:concentration_ineq_scalar} and using the tower property of conditional expectations yields,  $\E\left[s^2_{t,I^*} \mid  X_{1:T}\right] \leq \iota \cdot \E\left[ \tilde{s}^2_{t,I^*}\mid  X_{1:T} \right]$. 
Hence it suffices to prove
\begin{equation}\label{eq:smoothed_DTS_variance_bound}
	\E\left[\tilde{s}^2_{t,I^*} \mid X_{1:T}   \right] \leq  \frac{k}{ {\rm Precision}\left(X_{1:(t-L)}\right)}.
\end{equation}
Our main goal in this section is to prove \eqref{eq:smoothed_DTS_variance_bound} holds when arms are sampled according to DTS.

\subsubsection*{Preliminaries: matrix convex combinations.} Let $\mathbb{S}^n$ denote the set of symmetric $n\times n$ square matrices, and let $\mathbb{S}^n_+\subset \mathbb{S}^n$ denote the set of symmetric positive semidefinite matrices. A scalar function $f: \mathbb{R} \to \mathbb{R}$ can be extended to a function on symmetric matrices as follows. For any symmetric matrix $A \in \mathbb{S}_n$, one can write $A=\sum_{i=1}^{n} \lambda_i u_i u_i^\top$ where each $\lambda_i$ is a real eigenvalue and $u_i$ is the corresponding eigenvector. By defining $f(A) = \sum_{i=1}^{n} f(\lambda_i) u_i u_i^\top$, we have extended $f$ to a function mapping from $\mathbb{S}_n$ to $\mathbb{S}_n$. A function $f$ is said to be monotone increasing on the space of positive semidefinite matrices if for $A,B \in \mathbb{S}^n_{+}$, $A\preceq B$ implies $f(A) \preceq f(B)$ and monotone decreasing if this implies $f(A) \succeq f(B)$. A function $f$ is said to be operator convex on the space of positive definite matrices if for any $A,B \in \mathbb{S}^n_{+}$ and scalar $\lambda\in [0,1]$,  $f(\gamma A+(1-\gamma)B)  \preceq \gamma f(A) + (1-\gamma) f(B)$. For our purposes, a key fact is that the inverse function $f(A)= A^{-1}$ is convex and monotone decreasing.

To prove Proposition \ref{prop:DTS_varaince_bound}, we need to leverage  a generalization of Jensen's inequality that applies to matrix convex combinations. The following definitions can be found in \cite{tropp2015introduction}. 
\begin{definition}[Definition 8.5.1 in \cite{tropp2015introduction} -- Matrix Convex Combination]
	Let $B_1, B_2$ be Hermitian matrices (i.e. self-adjoint matrices). If $A_1^\top A_1 + A_2^\top A_2=I$, then the Hermitian matrix
	$A_1^\top B_1 A_1 + A_2^\top B_2 A_2$ is called a matrix convex combination of $B_1$ and $B_2$. 
\end{definition} 
The next result in Theorem 8.5.2 in \cite{tropp2015introduction} and a self-contained proof is given there. It provides a deep generalization of Jensen's inequality for operator convex functions, extending to a situation where the weights are matrices rather than scalars. 
\begin{lemma}[Theorem 8.5.2 in \cite{tropp2015introduction} -- Operator Jensen Inequality]\label{lem:operator_jensen}
	Let $f$ be an operator convex on the space of symmetric positive semidefinite matrices $\mathbb{S}_+^n$. Let $B_1, B_2 \in \mathbb{S}_+^n$. If $A_1^\top A_1 + A_2^\top A_2=I$ then, 
	\[
	f\left(A_1^\top B_1 A_1 + A_2^\top B_2 A_2    \right) \preceq A_1^\top f(B_1) A_1 + A_2^\top f(B_2) A_2.
	\] 
\end{lemma}
By induction, the lemma can be generalized to situations with more than two pairs of matrices. 

\subsubsection*{Using inverse propensity weights to analyze the evolution of posterior. }

The notation $V$ in Lemma \ref{lem:IPW-posterior-bound} is used only to simplify this lemma statement and is not used again in this paper. Observe that if the action selection is not randomized, and satisfies  $p_{\ell, i_\ell}=1$ for each $\ell\in[t-L]$, then  $\tilde{\Sigma}_{t} = {\rm Cov}(\theta \mid R_{1, i_1}, \ldots R_{t-L, i_{t-L}}, X_1, \ldots, X_{t-L}) $ and the bound in Lemma \ref{lem:IPW-posterior-bound} holds with equality. 
\begin{lemma}[Inverse-propensity weighted posterior evolution]\label{lem:IPW-posterior-bound}
Fix any sequence of arms $i_1, \ldots, i_{t-L}$. Then, with probability 1,
	\[
	\tilde{\Sigma}_{t}  \preceq V \left( \Sigma_{1}^{-1}+ \sigma^{-2}\sum_{\ell=1}^{t-L} \frac{ \phi(X_\ell, i_\ell)\phi(X_\ell, i_\ell)^\top }{ p_{\ell, i_\ell}}  \right) V,
	\]
	where 
	\begin{align*}
	V = {\rm Cov}\left(\theta \mid R_{1, i_1}, \ldots R_{t-L, i_{t-L}}, X_1, \ldots, X_{t-L}\right) =
	    \left( \Sigma_{1}^{-1}+ \sigma^{-2}\sum_{\ell=1}^{t-L} \phi(X_\ell, i_\ell)\phi(X_\ell, i_\ell)^\top   \right)^{-1}.
	\end{align*}
\end{lemma}
\begin{proof}
First observe that
\[
\tilde{\Sigma}_t = 
		\left(  \Sigma_1^{-1 } + \sigma^{-2} \sum_{\ell=1}^{t-L} \sum_{i=1}^{k} p_{\ell,i} \phi(X_\ell,i) \phi(X_\ell,i)^\top \right)^{-1} \\
		\preceq  \left(  \Sigma_1^{-1 } + \sigma^{-2} \sum_{\ell=1}^{t-L}   p_{\ell,i_\ell} \phi(X_\ell,i_\ell) \phi(X_\ell,i_\ell)^\top \right)^{-1}.
\]
Since $i_1, \ldots, i_{t-L}$ are fixed, drop them from notation and write $p_{\ell}= p_{\ell,i_\ell} = \Prob(I_\ell=i_\ell\mid H_\ell)$. Set $B_\ell = \sigma^{-2} \phi(X_\ell,i_\ell) \phi(X_\ell,i_\ell)^\top$. For notational convenience, set $B_{0}=\Sigma_{1}^{-1}$ and $p_{0}=1$. Then $V= \left(\sum_{\ell=0}^{t-L} B_{\ell}\right)^{-1}$, and the above inequality becomes
	\begin{align*}
		\tilde{\Sigma}_t
		& \preceq      \left(\sum_{\ell=0}^{t-L} p_{\ell} B_{\ell}\right)^{-1} \\
		&= V^{1/2}\left[ V^{1/2}     \left(\sum_{\ell=0}^{t-L} p_{\ell} B_{\ell}\right)  V^{1/2}  \right]^{-1}   V^{1/2}\\
		&=V^{1/2}\left[  V^{1/2}     \left(\sum_{\ell=0}^{t-L} B_{\ell}^{1/2} (p_{\ell} I) B_{\ell}^{1/2}\right)  V^{1/2}  \right]^{-1}   V^{1/2} \\
		&=V^{1/2}\left[   \sum_{\ell=0}^{t-L}  \left(V^{1/2}  B_{\ell}^{1/2}\right) (p_{\ell} I) \left( B_{\ell}^{1/2}  V^{1/2} \right) \right]^{-1}   V^{1/2} \\
		&\preceq V^{1/2}\left[   \sum_{\ell=0}^{t-L}  \left(V^{1/2}  B_{\ell}^{1/2}\right) (p_{\ell} I)^{-1} \left( B_{\ell}^{1/2}  V^{1/2} \right) \right]    V^{1/2} \\
		&= V \left(  \sum_{\ell=0}^{t-L}  \frac{B_{\ell}}{p_{\ell} }   \right)    V,
	\end{align*}
	where the last inequality applies the operator Jensen inequality in Lemma \ref{lem:operator_jensen}, using that 
	\[
	\sum_{\ell= 0}^{t-L}  \left(V^{1/2}  B_{\ell}^{1/2}\right)  \left(B_{\ell}^{1/2} V^{1/2}   \right) = V^{1/2} \left(\sum_{\ell=0}^{t-L} B_{\ell} \right) V^{1/2} = V^{1/2} V^{-1} V^{1/2} = I. 
	\]
\end{proof}

\subsubsection*{Completing the proof Proposition \ref{prop:DTS_varaince_bound}.}
Now we specialize this result to proof Proposition \ref{prop:DTS_varaince_bound}. 
\begin{proof}
	To start, we have 
	\begin{align*}
	\E\left[ s^2_{t,I^*}  \mid X_{1:T}  \right] = \E \left[ 	\E\left[ s^2_{t,I^*}  \mid \theta, X_{1:T}  \right]  \mid X_{1:T} \right] 
	&\leq  \E \left[ \iota	\cdot  \E\left[ \tilde{s}^2_{t,I^*}  \mid \theta, X_{1:T}  \right]  \mid X_{1:T} \right] \\ &=   \iota	\cdot  \E\left[ \tilde{s}^2_{t,I^*}  \mid  X_{1:T}  \right],
	\end{align*}
	where the inequality applies Corollary \ref{corr:concentration_ineq_scalar}, using that $I^*$ is non-random conditioned on $\theta$.  The remainder of the proof bounds $\tilde{s}^2_{t,I^*}$. 
	
 For $i\in[k]$, take 
	\[
	\phi_i =  \phi(x_{\rm pop},i)= (0,\ldots, 0, X_{{\rm pop},1},\ldots, X_{{\rm pop},d}, 0, \ldots,0) \in \mathbb{R}^{k \cdot d}
	\]
	to be a vector whose $i$-th subvector is $x_{\rm pop}$, and then $r_{\theta}(i) = \phi_i^\top \theta$. We can write 
	\begin{align*}
		\tilde{s}^2_{t,i} 
		&= {\rm Var}\left[ r_{\theta}(i) \mid  \left( X_{\ell}, (p_{\ell,j} , \tilde{R}_{\ell,j})_{j \in [k]}    \right)_{\ell\in [t-L]} \right]\\ 
		&= {\rm Var}\left[ \phi_i^\top \theta  \mid  \left( X_{\ell}, (p_{\ell,j} , \tilde{R}_{\ell,j})_{j \in [k]}    \right)_{\ell\in [t-L]} \right]\\
		& =  \phi_i^\top   {\rm Cov}\left[ \theta  \mid \left( X_{\ell}, (p_{\ell,j} , \tilde{R}_{\ell,j})_{j \in [k]}    \right)_{\ell\in [t-L]} \right]  \phi_i\\
		&= \phi_i^\top \tilde{\Sigma}_{t}  \phi_i.
	\end{align*}	
	Now set  
	\[
	V_i ={\rm Cov}(\theta \mid R_{1, i}, \ldots R_{t-L, i}, X_1, \ldots, X_{t-L}) =
	\left( \Sigma_{1}^{-1}+ \sigma^{-2}\sum_{\ell=1}^{t-L} \phi(X_\ell, i)\phi(X_\ell, i)^\top   \right)^{-1}.
	\]
	 Applying Lemma \ref{lem:IPW-posterior-bound} with $i_\ell = i$ for each $\ell\in[t-L]$ gives, 
	 \[ 
	 \tilde{s}^2_{t,i}  \leq \phi_i^\top V_i \left( \Sigma_{1}^{-1}+ \sigma^{-2}\sum_{\ell=1}^{t-L} \frac{ \phi(X_\ell, i)\phi(X_\ell, i)^\top }{ p_{\ell, i}}  \right) V_i \phi_i.
	 \]
	 Next,
	 \begin{align*}
	 	\tilde{s}^2_{t,I^*} = \sum_{i=1}^{k} \ind(I^*=i) \tilde{s}^2_{t,i}  &\leq  \sum_{i=1}^{k} \ind(I^*=i) \phi_i^\top V_i \left( \Sigma_{1}^{-1}+ \sigma^{-2}\sum_{\ell=1}^{t-L} \frac{ \phi(X_\ell, i)\phi(X_\ell, i)^\top }{ p_{\ell, i}}  \right) V_i \phi_i \\
	 	&\leq \sum_{i=1}^{k}  \phi_i^\top V_i \left( \Sigma_{1}^{-1}+ \sigma^{-2}\sum_{\ell=1}^{t-L} \frac{\ind(I^*=i)}{ p_{\ell, i}}   \phi(X, i)\phi(X, i)^\top \right) V_i \phi_i. 
	 \end{align*}
	Since $\E\left[\frac{\ind(I^*=i)}{ p_{\ell, i}} \mid X_{1:T}\right]  \leq 1$ for any $\ell$ and $i$, we have
	 \begin{align*}
	 	\E\left[ \tilde{s}^2_{t,I^*} \mid X_{1:T}\right]  
   \leq  \sum_{i=1}^{k}  \phi_i^\top V_i \left( \Sigma_{1}^{-1}+ \sigma^{-2}\sum_{\ell=1}^{t-L}  \phi(X_\ell, i)\phi(X_\ell, i)^\top \right) V_i \phi_i 
	 	&= \sum_{i=1}^{k}  \phi_i^\top V_i  V_i^{-1}  V_i \phi_i \\
	 	&= \sum_{i=1}^{k}  \phi_i^\top V_i \phi_i.
	 \end{align*}
 Recalling that $V_i = {\rm Cov}(\theta \mid R_{1, i}, \ldots R_{t-L, i}, X_1, \ldots, X_{t-L})$ and that $r_{\theta}(i) = \phi_i^\top \theta$, gives
 \begin{align*}
 		\sum_{i=1}^{k}  \phi_i^\top V_i \phi_i  	 	&= \sum_{i=1}^{k} \phi_i^\top {\rm Cov}(\theta \mid R_{1, i}, \ldots R_{t-L, i}, X_1, \ldots, X_{t-L}) \phi_i \\
 		&=  \sum_{i=1}^{k} {\rm Var}\left( r_\theta(i) \mid R_{1, i}, \ldots R_{t-L, i}, X_1, \ldots, X_{t-L}\right)\\
 		&\leq k \cdot \max_{i\in [k]} {\rm Var}\left( r_\theta(i) \mid R_{1, i}, \ldots R_{t-L, i}, X_1, \ldots, X_{t-L}\right) \\
 		& = \frac{k}{ {\rm Precision}\left(X_{1:(t-L)}\right) }, 
 \end{align*}
	completing the proof. 
\end{proof}

\section{Matrix-valued Martingales and the proof of Proposition \ref{prop:matrix_concentration}}
\label{sec:app_matrix_optional_sampling}
We begin by restating the result.
\MatrixSkipping*
 We let $\Prob_{t}(\cdot)=\Prob(\cdot \mid \mathcal{F}_{t} )$ and $\E_{t}[\cdot]=\E[\cdot \mid \mathcal{F}_{t} ]$. 
 Set $p_t = \E_{t-1}[Z_t] = \Prob_{t-1}(Z_t=1)$. 
 
 Throughout this proof, we use some specialized notation. 
Define $D_t = (p_t-Z_t) V_t$. We study 
\[ 
A_0 \triangleq 0\in\mathbb{R}^{d\times d} \quad \text{and} \quad  A_{n} \triangleq \tilde{S}_n- S_n  = \sum_{t=1}^{n} D_t,
\]
which is the sum of matrix martingale differences . We will follow \cite{tropp2011freedman} fairly closely. Define $\phi_t(\gamma)\triangleq\log\left( \E_{t-1}\left[ e^{\gamma D_t} \right]\right)$. Then set
\[ 
\Phi_0 \triangleq 0\in\mathbb{R}^{d\times d}  \quad \text{and} \quad \Phi_n(\gamma) \triangleq \sum_{t=1}^{n}\phi_t(\gamma).
\]
Recognize that both $A_0$ and $\Phi_0$ are matrices where all elements equal zero. Here $\Phi_n(\gamma)$ measures the total variability of the process. Our aim is to show that $A_n$ can only be large if $\Phi_n(\gamma)$ is large.

\subsubsection*{A Bernstein-type bound on the cumulants.}
We first recall a random matrix analogue of Bernstein's bound on the moment generating function of bounded random variables.
\begin{lemma}[Lemma 6.7 in \cite{tropp2012user}]
	\label{lem:tropp2012user}
	Suppose $D$ is a random self-adjoint matrix that satisfies
	\[
	\E [D] = 0 \quad \text{and} \quad \Prob\left(\lambda_{\max}(D)\leq 1\right)=1.
	\]
	Then 
	\[ 
	\E \left[e^{\gamma D}\right] \preceq \exp \left\{ (e^{\gamma} - \gamma -1)   \E\left[D^2\right] \right\}, \quad\forall \gamma > 0.
	\]
\end{lemma}
As a consequence of this, we can bound the sum of cumulants $\Phi_t(\gamma)$ by a simpler quantity that closely mimics $\tilde{S}_n$. 
\begin{lemma}\label{lem:smart_berstein}
	For any $\gamma >0$ and $n\in \mathbb{N} \cup \{0\}$,
	\[
	\Phi_n(\gamma) \preceq  \left( e^{\gamma} - \gamma -1 \right)\sum_{t=1}^{n} p_t V_t = \left( e^{\gamma} - \gamma -1 \right)\tilde{S}_n.
	\]
\end{lemma}
\begin{proof}[Proof of Lemma \ref{lem:smart_berstein}]
	We have  $\lambda_{\max}(D_t)\leq | p_t - Z_t | \lambda_{\max}(V_t) \leq 1$ 
 where the first inequality holds since $V_t$ is positive semidefinite and the last inequality follows from an assumption on the maximum eigenvalue of $V_t$. This allows us to apply the matrix Bernstein inequality above. 
	For notional convenience define $f(\gamma)=e^{\gamma} - \gamma -1$.   By Lemma \ref{lem:tropp2012user}, we have that
	\[
 \phi_t(\gamma) = \log \E_{t-1}\left[ e^{\gamma D_t} \right] \preceq f(\gamma) \E_{t-1}\left[ D_t^2\right].
 \]
 Using this gives, 
	\begin{align*}
		\Phi_{n}(\gamma) \preceq f(\gamma)\sum_{t=1}^{n} \E_{t-1}\left[D_t^2\right] &= f(\gamma)\sum_{t=1}^{n} \E_{t-1}\left[(p_t- Z_t)^2\right] V_t^2\\
		&\preceq f(\gamma)\sum_{t=1}^{n} p_t (1-p_t) V_t\\
		&\preceq f(\gamma)\sum_{t=1}^{n} p_t V_t,   
	\end{align*}
	as desired. To prove the first inequality, recall that all eigenvalues of $V_t \in \mathbb{S}_{+}^{d}$ are non-negative and smaller than 1.
	Write $V_t = \sum_{i=1}^{d} \lambda_i u_i u_i^\top$ in terms of its eigenvalues $\lambda_i \in [0,1]$ and eigenvectors $u_i$.  Then $V_t^2 u_i = V_t (\lambda_i u_i) = \lambda_i^2 u_i$. 
 The matrix $V_t^2$ also has $(u_1, \ldots, u_d)$ as eigenvectors, but with corresponding the smaller corresponding eigenvalues $(\lambda_1^2, \ldots, \lambda_d^2)$. 
\end{proof}

\subsubsection*{An exponential super-martingale.}
Again, our goal is to show that  $A_n$ can only be large if $\Phi_n(\gamma)$ is large. To this end, define 
\[
M_{n}(\gamma) = {\rm tr} \exp(  \gamma A_n  - \Phi_n(\gamma) ), \quad \forall n\in \{0,1,\ldots\},
\]
where ${\rm tr}$ denotes the trace operator. 
We now show that    $M_{n}(\gamma)$ is a  super-martingale, following the proof of Lemma 2.1 of \cite{tropp2011freedman}. We first state a powerful result of \cite{lieb1973convex} and then recall a simple corollary that is stated also in \cite{tropp2011freedman}. 
\begin{Theorem}[Theorem~6 in \cite{lieb1973convex}]  Fix a self-adjoint matrix $H$. The function $A \mapsto {\rm tr} \exp(H+ \log(A))$ is concave on the positive-definite cone. 
\end{Theorem} 
\begin{corollary}[Corollary 1.5 in \cite{tropp2011freedman}]\label{cor:Lieb} 
	Fix a self-adjoint matrix $H$. For a random self-adjoint matrix $X$, 
	\[
	\E \left[ {\rm tr} \exp(H+X) \right] \leq {\rm tr} \exp\left(  H + \log\left( \E e^{X}  \right)    \right).
	\]
\end{corollary}
We now conclude that $M_{n}(\gamma)$ is a super-martingale. 

\begin{corollary}\label{cor:supermartingale}
	For each $\gamma>0$, $\{ M_n(\gamma) : n=0, 1, \ldots\}$ is a super-martingale with initial value $M_0(\gamma)=d$. 
\end{corollary}
\begin{proof}[Proof of Corollary \ref{cor:supermartingale}]
	By definition, 
	\[ 
	M_0(\gamma)= {\rm tr}\exp(  \gamma A_0 - \Phi_{0}(\gamma)   ) = {\rm tr}\exp( 0  )= {\rm tr} I = d. 
	\]
	For $n>0$, taking conditional expectations gives
	\begin{align*} 
		\E_{n-1}\left[ M_{n}(\gamma)\right] &= \E_{n-1}\left[  {\rm tr} \exp(  \gamma A_{n-1}  - \Phi_{n}(\gamma) + \gamma D_n  ) \right]\\
		&\leq   {\rm tr}  \exp(  \gamma A_{n-1}  - \Phi_{n}(\gamma) +  \phi_{n}(\gamma)  ) \\
		&=  {\rm tr}  \exp(  \gamma A_{n-1}  - \Phi_{n-1}(\gamma)  )   = M_{n-1}(\gamma),
	\end{align*}
	where the inequality follows by Corollary \ref{cor:Lieb}, using that $\log\left(\E_{n-1}\left[e^{\gamma D_n}\right]\right) = \phi_n(\gamma)$.
 \end{proof}

\subsubsection*{Boundary crossing probabilities.}
Here is where we begin  to deviate from \cite{tropp2011freedman}. The next result gives a boundary that $A_n$ is unlikely to ever cross. The proof applies the same stopping time argument as the proof of one of Doob's martingale inequalities. 
\begin{lemma}\label{lem:line_crossing}
	For any fixed $\delta>0$ and $\gamma>0$, with probability exceeding $1-\delta$,
	\[
	A_{n}  \preceq  \frac{1}{\gamma}  \left[\Phi_{n}(\gamma) + \log\left(\frac{d}{\delta}\right) I\right], \quad \forall n\in \mathbb{N}.
	\]
\end{lemma}
\begin{proof}[Proof of Lemma \ref{lem:line_crossing}]
	Fix $\gamma>0$ throughout. Let $y\in\mathbb{R}$. We have
	\begin{align*}
		\Prob\left( \lambda_{\max}\left( \gamma A_{n}  - \Phi_{n}(\gamma)  \right)     \geq  y \right)  = \Prob\left( e^{\lambda_{\max}\left( \gamma A_{n}  - \Phi_{n}(\gamma) \right)}     \geq  e^y \right) &\leq \Prob\left( {\rm tr} \, e^{ \gamma A_{n}  - \Phi_{n}(\gamma) } \geq e^{y}\right) \\
		&\leq e^{-y} \E\left[ {\rm tr} \, e^{ \gamma A_{n}  - \Phi_{n}(\gamma) }  \right] \\ &= e^{-y}\E\left[ M_{n}(\gamma) \right]. 
	\end{align*}
	The same inequalities hold for any bounded stopping time $\tau$, yielding 
	\[ 
	\Prob\left( \lambda_{\max}\left( \gamma A_{\tau}  - \Phi_{\tau }(\gamma)  \right)     \geq  y \right) \leq  e^{-y} \E \left[ M_{\tau}(\gamma)\right].
	\]
	Take $\tau = \inf\{n \in \mathbb{N} : \lambda_{\max}\left( \gamma A_{n}  - \Phi_{n}(\gamma)  \right)     \geq  y \}$, with the convention that $\tau=\infty$ if $ \lambda_{\max}\left( \gamma A_{n}  - \Phi_{t }(\gamma)  \right)     <  y$ for every $n\in \mathbb{N}$. Then, 
	\begin{align*}
	\Prob\left( \exists n \leq N :  \lambda_{\max}\left( \gamma A_{n}  - \Phi_{n}(\gamma)  \right)     \geq  y \right)  
 &=  \Prob\left( \lambda_{\max}\left( \gamma A_{\tau\wedge N}  - \Phi_{\tau \wedge N }(\gamma)  \right) \geq y\right) \\ 
 &\leq   e^{-y}\E \left[ M_{\tau \wedge N}(\gamma)\right]\\
 &\leq e^{-y}d. 
	\end{align*}
	That $\E\left[ M_{\tau \wedge N}(\gamma) \right] \leq d$ uses Corollary \ref{cor:supermartingale} and Doob's optional sampling theorem. Taking $N\to \infty$ and applying the monotone convergence theorem gives, 
	\[
	\Prob\left( \exists n \in \mathbb{N} :  \lambda_{\max}\left( \gamma A_{n}  - \Phi_{n}(\gamma)  \right)     \geq  y \right)     \leq     e^{-y}d.
	\]
	For any $\delta > 0$, we choose $y=\log(d/\delta)$, and then with probability at least $1-\delta$,
	\[ 
	\lambda_{\max}\left( \gamma A_{n}  - \Phi_{n}(\gamma)  \right)   \leq \log\left(\frac{d}{\delta}\right),\quad\forall n\in \mathbb{N}. 
	\]
\end{proof}
Combining our results completes the proof of Proposition \ref{prop:matrix_concentration}.
\begin{proof}[Proof of Proposition \ref{prop:matrix_concentration}]
	Define $f(\gamma)=e^{\gamma} - \gamma -1$. Recall $A_{n} = \tilde{S}_n - S_n$.  By applying Lemmas \ref{lem:line_crossing} and \ref{lem:smart_berstein}, we get that with probability at least $1-\delta$, 
	\begin{align*}
		S_{n} - \tilde{S}_n = -A_n \succeq - \frac{1}{\gamma}  \left[\Phi_{n}(\gamma) + \log\left(\frac{d}{\delta}\right) I\right] \succeq  -\frac{1}{\gamma} \left[f(\gamma)\sum_{t=1}^{n} p_{t} V_t + \log\left(\frac{d}{\delta}\right) I \right] . 
	\end{align*}
	Picking $\gamma=1$, we have 
	\[
	S_{n} - \tilde{S}_n \succeq   (2-e) \sum_{t=1}^{n} p_t V_t  - \log\left(\frac{d}{\delta}\right) I= (2-e) \tilde{S}_{n}  - \log\left(\frac{d}{\delta}\right) I. 
	\]
	Adding $\tilde{S}_n$ to both sides yields the result. 
\end{proof}

\section{Bounds on attainable precision: proof of Lemma \ref{lem:bounds_on_precision}}
In this section, we use $\|\cdot\|$ to denote the spectral norm. First, we restate the claim.

\precisionBounds*

\begin{proof}[Proof of Lemma \ref{lem:bounds_on_precision}]

The definition of precision in \eqref{eq:precision_formula} gives
 \begin{align} 
		\frac{1}{{\rm Precision}(x_{1 : t})}
		&=  \max_{i\in [k]} \,  x_{\rm pop}^\top \left[ {\rm Cov}\left(\theta^{(i)}\right)^{-1} + \sigma^{-2} \sum_{\ell=1}^{t} x_\ell x_\ell^\top      \right]^{-1} x_{\rm pop} \nonumber\\
		&\leq \,   x_{\rm pop}^\top \left[ \lambda_{\min}\left(\Sigma_1^{-1}\right)I + \sigma^{-2} \sum_{\ell=1}^{t} x_\ell x_\ell^\top      \right]^{-1} x_{\rm pop} \label{eq:generic_upper_bound}\\ 
  &=  x_{\rm pop}^\top \left(\sigma^{-2}t \tilde{S}_x\right)^{-1} x_{\rm pop}, \nonumber
	\end{align}
 where the inequality uses Lemma \ref{lem:prior covariance}.
Taking the inverse on both sides yields the generic bound on precision. 

For vanilla bandit (with $d=1$), since $\theta^{(i)}\in\mathbb{R}$ for each $i\in[k]$ and $x_\ell=1=x_{\rm pop}$ for each $\ell\in [t]$, the definition of precision in \eqref{eq:precision_formula} becomes
	\[
	{\rm Precision}(x_{1 : t}) = \min_{i\in [k]}\, {\rm Var}\left(\theta^{(i)}\right)^{-1}+  \sigma^{-2} t = \min_{i\in[k]}\, \Sigma_{1,ii}^{-1} + \sigma^{-2}t,
	\]
and the above generic bound gives
\[
{\rm Precision}(x_{1 : t}) \geq \lambda_{\min}\left(\Sigma_1^{-1}\right) + \sigma^{-2}t.
\]
 
 Next we analyze the setting without empirical distribution shift. By \eqref{eq:generic_upper_bound} and Lemma \ref{lem:no empirical distribution shift},
	\begin{align*} 
		\frac{1}{{\rm Precision}(x_{1 : t})}
		&\leq \,   x_{\rm pop}^\top \left[\lambda_{\min}\left(\Sigma_1^{-1}\right)\cdot I + \sigma^{-2} \sum_{\ell=1}^{t} x_\ell x_\ell^\top      \right]^{-1} x_{\rm pop} \\
		&\leq \, x_{\rm pop}^\top \left[\lambda_{\min}\left(\Sigma_1^{-1}\right)\cdot I + \sigma^{-2} t \cdot x_{\rm pop} x_{\rm pop}^\top      \right]^{-1} x_{\rm pop}.
	\end{align*}	
 Fact \ref{fact:rank-one matrix} implies that $x_{\rm pop} x_{\rm pop}^\top$ only has one non-zero eigenvalue $\|x_{\rm pop}\|_2^2$ with a corresponding eigenvector $\frac{x_{\rm pop}}{\|x_{\rm pop}\|_2}$ (recall that $\|x_{\rm pop}\|_2\neq 0$),  so the eigendecomposition of $x_{\rm pop} x_{\rm pop}^\top$ can be written as
	\[
	x_{\rm pop} x_{\rm pop}^\top = Q\Lambda Q^\top
	\]
	where $\Lambda = {\rm diag}\left(\|x_{\rm pop}\|_2^2,0,\ldots,0\right)\in\mathbb{R}^{d\times d}$ is the diagonal matrix whose diagonal elements are the eigenvalues of $x_{\rm pop} x_{\rm pop}^\top$, and $Q\in\mathbb{R}^{d\times d}$ is a corresponding orthogonal matrix with the first column being $\frac{x_{\rm pop}}{\|x_{\rm pop}\|_2}$.
	Then we have 
	\begin{align*}
	&\lambda_{\min}\left(\Sigma_1^{-1}\right)\cdot I + \sigma^{-2} t\cdot x_{\rm pop} x_{\rm pop}^\top\\
 =& Q\cdot\left[\lambda_{\min}\left(\Sigma_1^{-1}\right)\cdot I\right]\cdot Q^\top + 
 Q\cdot\left[\sigma^{-2}t\cdot\Lambda\right]\cdot Q^\top \\
 =& Q\cdot
{\rm diag}\left(\lambda_{\min}\left(\Sigma_1^{-1}\right) + \sigma^{-2} t \|x_{\rm pop}\|_2^2, \lambda_{\min}\left(\Sigma_1^{-1}\right),\ldots, \lambda_{\min}\left(\Sigma_1^{-1}\right) \right)
 \cdot Q^\top.
	\end{align*}
	Hence,
	\begin{align*}
	&\left[\lambda_{\min}\left(\Sigma_1^{-1}\right)\cdot I + \sigma^{-2} t \cdot x_{\rm pop} x_{\rm pop}^\top\right]^{-1} \\
	=& Q\cdot{\rm diag}\left(\frac{1}{\lambda_{\min}\left(\Sigma_1^{-1}\right) + \sigma^{-2} t \|x_{\rm pop}\|_2^2 }, \frac{1}{\lambda_{\min}\left(\Sigma_1^{-1}\right)},\ldots, \frac{1}{\lambda_{\min}\left(\Sigma_1^{-1}\right)} \right)\cdot Q^\top,
	\end{align*}
	and thus
	\begin{align*}
		&\frac{1}{{\rm Precision}(x_{1 : t})} \\
		\leq& \, x_{\rm pop}^\top \left[\lambda_{\min}\left(\Sigma_1^{-1}\right)\cdot I + \sigma^{-2} t \cdot x_{\rm pop} x_{\rm pop}^\top\right]^{-1} x_{\rm pop} \\
		=& x_{\rm pop}^\top Q\cdot{\rm diag}\left(\frac{1}{\lambda_{\min}\left(\Sigma_1^{-1}\right) + \sigma^{-2} t \|x_{\rm pop}\|_2^2 }, \frac{1}{\lambda_{\min}\left(\Sigma_1^{-1}\right)},\ldots, \frac{1}{\lambda_{\min}\left(\Sigma_1^{-1}\right)} \right)\cdot Q^\top x_{\rm pop}\\
		=& \left(\|x_{\rm pop}\|_2,0,\ldots, 0\right) 
  \begin{pmatrix}
\frac{1}{\lambda_{\min}\left(\Sigma_1^{-1}\right) + \sigma^{-2} t \|x_{\rm pop}\|_2^2} & 0  &\ldots & 0\\
0 & \frac{1}{\lambda_{\min}\left(\Sigma_1^{-1}\right)} &\ldots & 0 \\
\vdots & \vdots &  \ddots & \vdots\\
0 & 0 & \ldots & \frac{1}{\lambda_{\min}\left(\Sigma_1^{-1}\right)}
\end{pmatrix}
  \begin{pmatrix}
	\|x_{\rm pop}\|_2 \\
	0\\
	\vdots\\
	0
    \end{pmatrix}\\
=& \frac{\|x_{\rm pop}\|_2^2}{\lambda_{\min}\left(\Sigma_1^{-1}\right)+\sigma^{-2} t \|x_{\rm pop}\|_2^2}.
	\end{align*}
 where the penultimate equality follows from that $Q\in\mathbb{R}^{d\times d}$ is the orthogonal matrix with the first column being $\frac{x_{\rm pop}}{\|x_{\rm pop}\|_2}$.
	Taking the inverse on both sides gives the lower bound on precision when there is no empirical distribution shift.
	
	Lastly we study the setting with i.i.d.~contexts. Let $\Omega = \E[X_1 X_1^\top]$ and $Z_\ell = X_\ell X_\ell^\top - \Omega$ for $\ell\in[t]$, and we have $\Omega = \mathbb{E}[ X_1 X_1^\top] \succeq c x_{\rm pop} x_{\rm pop}^{\top}$. Then by \eqref{eq:generic_upper_bound},
	\begin{align*} 
		\frac{1}{{\rm Precision}(X_{1 : t})}
		&\leq \,   x_{\rm pop}^\top \left[\lambda_{\min}\left(\Sigma_1^{-1}\right)\cdot I + \sigma^{-2} \sum_{\ell=1}^{t}X_\ell X_\ell^\top     \right]^{-1} x_{\rm pop} \\
		&= \,   x_{\rm pop}^\top \left[\lambda_{\min}\left(\Sigma_1^{-1}\right)\cdot I + \sigma^{-2} \sum_{\ell=1}^{t}Z_\ell + \sigma^{-2}t\cdot \Omega      \right]^{-1} x_{\rm pop} \\
		&\leq \,   x_{\rm pop}^\top \left[\lambda_{\min}\left(\Sigma_1^{-1}\right)\cdot I + \sigma^{-2} \sum_{\ell=1}^{t}Z_\ell + c\cdot\sigma^{-2}t \cdot x_{\rm pop} x_{\rm pop}^\top    \right]^{-1} x_{\rm pop}.
	\end{align*}
	Note that the spectral norm of $Z_\ell$ can be bounded as follows, by the triangle inequality,
	\[
	\| Z_\ell\| \leq \|X_\ell X_\ell^\top\| + \|\E[X_1 X_1^\top]\|\leq \|X_\ell X_\ell^\top\| + \E[\|X_1 X_1^\top\|] = \|X_\ell\|^2_2 + \E\left[\|X_1\|^2_2\right] \leq 2,
	\]
	where the first and second inequalities apply the triangle inequality and Jensen's inequality, respectively; the next equality uses Fact \ref{fact:rank-one matrix}; the last inequality follows from an assumption on the maximum $\ell_2$ norm of context vectors.   
 This implies
	$
	Z_\ell^2 \preceq 4I.
	$
	By the matrix Hoeffding inequality in Lemma \ref{lem:matrix hoeffding}, for $x\geq 0$, with probability at least $1 - d\exp(-x^2/(32t))$
	\[
	\lambda_{\min}\left(\sum_{\ell=1}^tZ_\ell\right) > -x,
	\quad\text{and thus}\quad
	\sum_{\ell=1}^tZ_\ell \succ -xI.
	\]
	Hence, for $x\geq 0$, with probability at least $1 - d\exp(-x^2/(32t))$,
	\begin{align*} 
		&\frac{1}{{\rm Precision}(X_{1 : t})}\\
		\leq& x_{\rm pop}^\top \left[\lambda_{\min}\left(\Sigma_1^{-1}\right) \cdot I + \sigma^{-2} \sum_{\ell=1}^{t}Z_\ell + c\cdot\sigma^{-2}t \cdot x_{\rm pop} x_{\rm pop}^\top    \right]^{-1} x_{\rm pop} \\
		\leq&  x_{\rm pop}^\top \left[\left(\lambda_{\min}\left(\Sigma_1^{-1}\right)-\sigma^{-2}x\right)I + c\cdot\sigma^{-2}t \cdot x_{\rm pop} x_{\rm pop}^\top   \right]^{-1} x_{\rm pop}\\
		=& x_{\rm pop}^\top Q\cdot
  \begin{pmatrix}
\frac{1}{\lambda_{\min}\left(\Sigma_1^{-1}\right) -\sigma^{-2}x + c\cdot\sigma^{-2} t \|x_{\rm pop}\|_2^2} & 0  &\ldots & 0\\
0 & \frac{1}{\lambda_{\min}\left(\Sigma_1^{-1}\right)-\sigma^{-2}x} &\ldots & 0 \\
\vdots & \vdots &  \ddots & \vdots\\
0 & 0 & \ldots & \frac{1}{\lambda_{\min}\left(\Sigma_1^{-1}\right)-\sigma^{-2}x}
\end{pmatrix}
\cdot Q^\top x_{\rm pop}\\
  =& \left(\|x_{\rm pop}\|_2,0,\ldots, 0\right)
  \begin{pmatrix}
\frac{1}{\lambda_{\min}\left(\Sigma_1^{-1}\right) -\sigma^{-2}x + c\cdot\sigma^{-2} t \|x_{\rm pop}\|_2^2} & 0  &\ldots & 0\\
0 & \frac{1}{\lambda_{\min}\left(\Sigma_1^{-1}\right)-\sigma^{-2}x} &\ldots & 0 \\
\vdots & \vdots &  \ddots & \vdots\\
0 & 0 & \ldots & \frac{1}{\lambda_{\min}\left(\Sigma_1^{-1}\right)-\sigma^{-2}x}
\end{pmatrix}
\begin{pmatrix}
			\|x_{\rm pop}\|_2 \\
			0\\
			\vdots\\
			0
		\end{pmatrix}\\
		=& \frac{\|x_{\rm pop}\|_2^2}{\lambda_{\min}
  \left(\Sigma_1^{-1}\right) -\sigma^{-2}x + c\cdot\sigma^{-2} t \|x_{\rm pop}\|_2^2},
	\end{align*}
	where the equalities above follow the same analysis for the setting with no empirical distribution shift.
	Equivalently, for $\delta > 0$, with probability at least $1-\delta$,
	\[
	\frac{1}{{\rm Precision}(X_{1 : t})} \leq \frac{\|x_{\rm pop}\|_2^2}{\lambda_{\min}\left(\Sigma_1^{-1}\right) - 4\sigma^{-2}  \sqrt{2t\log\frac{d}{\delta}} + c\cdot\sigma^{-2} t \|x_{\rm pop}\|_2^2}
	\]
 and taking the inverse on both sides gives
 \[
 {\rm Precision}(X_{1 : t}) \geq \lambda_{\min}\left( \Sigma_1^{-1} \right)\|x_{\rm pop}\|_2^{-2} + c\cdot \sigma^{-2}t - 4\sigma^{-2} \|x_{\rm pop}\|_2^{-2} \sqrt{2t\log\frac{d}{\delta}}.
 \]
\end{proof}

\paragraph{Supporting results.} We introduce several supporting results for the proof of Lemma 
\ref{lem:bounds_on_precision}.
\begin{lemma}
	\label{lem:prior covariance}
	For $i\in[k]$, $\lambda_{\min}\left({\rm Cov}\left(\theta^{(i)}\right)^{-1}\right) \geq \lambda_{\min}\left(\Sigma_1^{-1}\right)$.
\end{lemma}
\begin{proof}
	Fix $i\in[K]$. We prove an equivalent statement: $\lambda_{\max}\left({\rm Cov}\left(\theta^{(i)}\right)\right) \leq \lambda_{\max}\left(\Sigma_1\right)$.
 The $\ell_2$ induced norm (i.e. spectral norm) of a positive semidefinite (and symmetric) matrix equals its largest eigenvalue, so we have
	\begin{align*}
		\lambda_{\max}\left({\rm Cov}\left(\theta^{(i)}\right)\right) 
		&= \max_{x=(x_1,\ldots,x_d)\in\mathbb{R}^d:\|x\|_2=1}x^\top {\rm Cov}\left(\theta^{(i)}\right) x \\
		&=\max_{x\in\mathbb{R}^d:\|x\|_2=1}\phi(x,i)^\top \Sigma_1 \phi(x,i) \\
		&\leq \max_{z\in\mathbb{R}^{d\times k}} z^\top \Sigma_1 z \\
		&= \lambda_{\max}(\Sigma_1).
	\end{align*}
The second equality above holds because
$\phi(x,i) = (0, \ldots, 0,\underbrace{ x_1, \ldots, x_d}_{ i\text{-th subvector} }, 0,\ldots, 0)^\top \in \mathbb{R}^{kd}$  has non-zero entries only in the $i$-th subvector, and then we only need to consider the corresponding submatrix of $\Sigma_1$ when calculating the quadratic term.  
This completes the proof. 
\end{proof}

\begin{lemma}
	\label{lem:no empirical distribution shift}
	For any $t\in\mathbb{N}$,
	\[
	t\sum_{\ell=1}^{t} x_\ell x_\ell^\top \succeq \left(\sum_{\ell=1}^{t} x_\ell\right)\left(\sum_{\ell=1}^{t} x_\ell\right)^\top.
	\]
\end{lemma}
\begin{proof}
	Let $\bar{x} = \frac{1}{t} \sum_{\ell=1}^{t} x_\ell$.
	Then the statement follows from
	\begin{align*}
		\sum_{\ell=1}^{t} x_\ell x_\ell^\top = t \bar{x}\bar{x}^\top + \sum_{\ell=1}^t(x_\ell - \bar{x})(x_\ell - \bar{x})^\top \succeq t\bar{x}\bar{x}^\top.
	\end{align*}
\end{proof}

\begin{fact}
\label{fact:rank-one matrix}
Let $x\in\mathbb{R}^d$. The matrix $x x^\top\in\mathbb{R}^{d\times d}$ has only one potentially non-zero eigenvalue $\|x\|_2^2$ with a corresponding eigenvector $x$. The spectral norm of $x x^\top$, denoted by $\|x x^\top\|$, equals $\|x\|_2^2$.
\end{fact}

\begin{lemma}[Matrix Hoeffding -- Theorem 1.3 in \citep{tropp2012user}]
	\label{lem:matrix hoeffding}
	Consider a finite sequence $\{X_n\}$ of independent, random, self-adjoint matrices with dimension $d$ and a sequence $\{Y_n\}$ of fixed self-adjoint matrices. Assume that each random matrix satisfies
	\[
	\E[X_n] = 0
	\quad\text{and}\quad
	X_n^2\preceq Y_n^2
	\quad\text{almost surely}.
	\]
	Then, for all $x\geq 0$,
	\[
	\Prob\left(\lambda_{\max}\left(\sum_n X_n\right)\geq x\right)
	\leq d\cdot\exp\left(\frac{-x^2}{8\left\|\sum_n Y_n^2\right\|}\right)
        \]
        and\footnote{The inequality below follows from applying the inequality above to $\{-X_n\}$ and $\{Y_n\}$ and using $\Prob\left(\lambda_{\min}\left(\sum_n X_n\right)\leq -x\right) = \Prob\left(\lambda_{\max}\left(\sum_n -X_n\right)\geq x\right)$. See Remark 3.10 (Minimum Eigenvalue) in \citep{tropp2012user}.}
        \[
	\Prob\left(\lambda_{\min}\left(\sum_n X_n\right)\leq -x\right)
	\leq d\cdot\exp\left(\frac{-x^2}{8\left\|\sum_n Y_n^2\right\|}\right).
	\]
\end{lemma}

\end{document}